\documentclass{article}

\usepackage[letterpaper,margin=0.9in]{geometry}

\usepackage{microtype}
\usepackage{graphicx}
\usepackage{subcaption}
\usepackage{booktabs}
\usepackage{hyperref}

\usepackage{amsmath}
\usepackage{amssymb}
\usepackage{mathtools}
\usepackage{amsthm}
\usepackage{dsfont}
\usepackage{comment}
\usepackage{algorithm}
\usepackage{algorithmicx}
\usepackage{algpseudocode}
\usepackage{natbib}
\usepackage{bm}
\newcommand{\argmax}{\mathop{\rm arg~max}\limits}
\newcommand{\argmin}{\mathop{\rm arg~min}\limits}

\usepackage[capitalize,noabbrev]{cleveref}

\theoremstyle{plain}
\newtheorem{theorem}{Theorem}[section]
\newtheorem{proposition}[theorem]{Proposition}
\newtheorem{lemma}[theorem]{Lemma}
\newtheorem{corollary}[theorem]{Corollary}
\theoremstyle{definition}
\newtheorem{definition}[theorem]{Definition}
\newtheorem{assumption}[theorem]{Assumption}
\theoremstyle{remark}
\newtheorem{remark}[theorem]{Remark}
\theoremstyle{definition}
\newtheorem{problem}[theorem]{Problem}
\theoremstyle{remark}
\newtheorem{example}[theorem]{Example}

\title{An Efficient Algorithm for Thresholding Monte Carlo Tree Search}

\author{
    Shoma Nameki$^{1}$ \and Atsuyoshi Nakamura$^{2}$ \and Junpei Komiyama$^{3,4}$ \and Koji Tabata$^{5}$\\
    {\small $^{1}$Graduate School of Information Science and Technology, Hokkaido University}\\
    {\small $^{2}$Faculty of Information Science and Technology, Hokkaido University}\\
    {\small $^{3}$Mohamed bin Zayed University of Artificial Intelligence (MBZUAI)}\\
    {\small $^{4}$RIKEN AIP}\\
    {\small $^{5}$Research Institute for Electronic Science, Hokkaido University}\\
    {\small \texttt{\{nameki,atsu\}@ist.hokudai.ac.jp, junpei@komiyama.info, ktabata@es.hokudai.ac.jp}}
}

\date{}

\begin{document}

\twocolumn[
    \maketitle
    \begin{abstract}
        We introduce the Thresholding Monte Carlo Tree Search problem, in which, given a tree $\mathcal{T}$ and a threshold $\theta$, a player must answer whether the root node value of $\mathcal{T}$ is at least $\theta$ or not. In the given tree, `MAX' or `MIN' is labeled on each internal node, and the value of a `MAX'-labeled (`MIN'-labeled) internal node is the maximum (minimum) of its child values. The value of a leaf node is the mean reward of an unknown distribution, from which the player can sample rewards. For this problem, we develop a $\delta$-correct sequential sampling algorithm based on the Track-and-Stop strategy that has asymptotically optimal sample complexity. 
        We show that a ratio-based modification of the D-Tracking arm-pulling strategy leads to a substantial improvement in empirical sample complexity, as well as reducing the per-round computational cost from linear to logarithmic in the number of arms.
    \end{abstract}
    \vspace{1.5em}
]

\section{Introduction}

The last two decades saw the success of Monte Carlo Tree Search (MCTS) in many domains, including abstract games  \citep{DBLP:conf/cg/Coulom06,DBLP:journals/nature/SilverHMGSDSAPL16,DBLP:reference/ecgg/Yoshizoe019}, planning and exploration of Markov decision process (MDP) structure \citep{DBLP:journals/ior/ChangFHM05,pepels2014,Graeme2018,pmlr-v129-leurent20a,NEURIPS2020_0d85eb24}, solving NP-hard problems \cite{graphnn_mcts,DBLP:conf/iclr/YangAY21,DBLP:conf/aaai/KhalilVD22,ijcai2022p1}, as well as guiding through large language models \citep{DBLP:conf/nips/YaoYZS00N23,DBLP:conf/icml/WanFWM00024}.
In particular, the most widely used method in MCTS incorporates a bandit-based method \cite{DBLP:conf/ecml/KocsisS06}. Bandit-based methods framed MCTS as a solution to the exploration and exploitation tradeoff. However, in many problems, what we truly seek is the most effective action in the tree, and in this sense, MCTS can be viewed as a pure exploration problem \cite{tolpin2012}. As is well known, bandit-based methods are suboptimal for pure exploration problems \cite{DBLP:conf/alt/BubeckMS09}, as they lack the robustness in identifying the best action.

A key difficulty in understanding the theoretical property of purely exploring MCTS is that it intertwines two aspects: tree expansion (constructing a search tree adaptively) and exploration on the currently expanded tree (allocating samples to evaluate candidate actions/states).
Analyzing algorithms that dynamically expand the tree is theoretically appealing but challenging, as it typically requires nontrivial assumptions to relate statistics at an internal node to those of its descendants \citep{DBLP:conf/uai/CoquelinM07}.
As a result, much of the recent theoretical literature separates these two aspects and focuses on pure exploration of a given (expanded) tree, treating the tree structure as fixed and using this as a fundamental modeling assumption \citep{tolpin2012,teraoka2014,DBLP:journals/jair/FeldmanD14,DBLP:conf/nips/KaufmannK17}.
 
Most practical tree structures are astronomically large.
For example, the entire tree size of the Game of Go is estimated to be around $10^{172}$, making it impossible to evaluate all nodes exhaustively \cite{DBLP:reference/ecgg/Yoshizoe019}. Instead, we explicitly construct the search tree only up to a fixed depth or according to a predefined expansion criterion. From the resulting frontier nodes (leaf nodes of the constructed tree), the remaining part of the tree is evaluated via stochastic playouts, i.e., randomized rollouts that simulate the game until termination.
From this viewpoint, each leaf node corresponds not to a single deterministic value, but to an unknown reward distribution.

Seminal work by \citet{DBLP:conf/nips/KaufmannK17} considers an identification of $\varepsilon$-best action in this setting with a generalization of the $\delta$-correct best arm identification \citep{bechhofer1968sequential,DBLP:conf/wsc/GlynnJ04,pmlr-v49-garivier16a}. The work derived an algorithm based on the confidence bound strategy and proves its bound on the sample complexity (stopping time). 
However, the optimal sample complexity they derived does not admit a closed-form expression, nor does their algorithm have asymptotic optimality.

Similar to \citet{DBLP:conf/nips/KaufmannK17}, we consider the fixed-confidence setting, where the objective is to design a $\delta$-correct algorithm that minimizes the sample complexity.
We depart from their setting by introducing a threshold $\theta$ and aim to decide whether the root value exceeds $\theta$ with probability at least $1-\delta$, following the thresholding perspective studied in prior work \citep{DBLP:conf/icml/LocatelliGC16,DBLP:journals/ml/KanoHSMNS19,DBLP:conf/aistats/TabataKNK23,DBLP:conf/pakdd/TsaiTL25}.
Thanks to this threshold introduction, sample complexity is lower bounded by the value that is calculated by a simple recursive formula (Section~\ref{sec_optimalallocation}). As a result, the optimal sampling proportion of arms can be calculated in time linear to the tree depth, which enables to design an efficient algorithm that tracks the proportion.

Our algorithm design starts from the Track-and-Stop algorithm \cite{pmlr-v49-garivier16a,DBLP:conf/nips/DegenneK19} that uses the optimal sample proportion estimated by plug-in mean parameters.
However, the standard Track-and-Stop algorithm has two major shortcomings.
First, it is sensitive to the estimation error of the plug-in estimators that degrades the empirical sample complexity.
Second, it requires recomputing the optimal sampling proportions at every step, which requires $O(|\mathcal{L}(\mathcal{T})|)$ time in the number of leaves $|\mathcal{L}(\mathcal{T})|$ in the tree $\mathcal{T}$.
To address these, we introduce \textbf{the ratio-based draw}, a simple modification to the Track-and-Stop algorithm. This modification to D- Tracking reduces the empirical sample complexity by several factors (Section \ref{sec:experiment}). 
Moreover, the ratio-based draw, combined with careful construction of the data structure, enables the modified D-Tracking to run in $O(\log |\mathcal{L}(\mathcal{T})|)$ time per step if $\mathcal{T}$ is balanced, without resorting to approximate computations (Section~\ref{sec:time_complexity}).

We deal with the decision problem but the typical setting on MCTS is the identification problem, that is, the problem of the best next action identification. We also show that our slightly modified version of our algorithm is an optimal algorithm for the identification of a good next action.

\section{Problem Setting}
\label{sec_problem}

We study a binary decision problem with outcomes in $\{\text{`win'},\text{`lose'}\}$ defined on a rooted tree 
$\mathcal{T} = (S, B, s_0)$, where $S$ is the set of nodes, $B$ is the set of branches, and $s_0$ is the root node. 
Each node represents a state, and $s_0$ is the current state. A branch $(s_1,s_2)$ corresponds to an action at state $s_1$ that leads to state $s_2$.
A node is a leaf if it has no children. A node is internal if it is not a leaf.
Each internal node $s$ has label $L(s)\in \{ \text{`MAX'}, \text{`MIN'} \}$. 

Let $\mathcal{L}(\mathcal{T})$ be the set of leaf nodes in $\mathcal{T}$.
Each leaf corresponds to an arm, and we use the terms \emph{leaf} and \emph{arm} interchangeably.
Given a question tree $\mathcal{T}$ and its internal node label function $L$, at each round $t=1,2,\dots$,
the player selects one leaf node $I(t) \in \mathcal{L}(\mathcal{T})$, and observes a reward $r(t)\in X\subseteq \mathbb{R}$
drawn from an unknown distribution with mean $\mu_{I(t)}\in X$.
In this paper, we assume that the distributions belong to a one-parameter exponential family, and the rewards observed from the same leaf are i.i.d.
For each $\ell\in \mathcal{L}(\mathcal{T})$,
$\hat{\mu}_{\ell}(t)$ and $N_{\ell}(t) = |\{s \in [1,t] : I(s) = \ell\}|$ denote the sample mean of the rewards and the number of selections, respectively, for the leaf node $\ell$ by round $t$.

Throughout the paper, we denote $(\mu_\ell)_{\ell \in \mathcal{L}(\mathcal{T})}\in X^{|\mathcal{L}(\mathcal{T})|}$ by $\bm{\mu}$.
For any vector $\bm{\lambda}=(\lambda_{\ell})_{\ell\in \mathcal{L}(\mathcal{T})} \in X^{|\mathcal{L}(\mathcal{T})|}$, we denote $\bm{\lambda}_{\mathcal{L}'} = (\lambda_\ell)_{\ell \in \mathcal{L}'}$ for any subset of leaves $\mathcal{L}' \subseteq \mathcal{L}(\mathcal{T})$.
Let $\mathcal{D}(s) \subseteq \mathcal{L}(\mathcal{T})$ be the set of leaf nodes that are descendants of node $s$.
For each node $s \in S$, we define the \textit{value} $V_s(\bm{\mu})$ of node $s$ based on $\bm{\mu}_{\mathcal{D}(s)}$ as
\begin{align*}
    V_s(\bm{\mu}) =
    \begin{cases}
        \mu_s & (s\in \mathcal{L}(\mathcal{T})),\\
        \max_{c\in \mathcal{C}(s)}V_c(\bm{\mu}) & (L(s)=\text{`MAX'}),\\
        \min_{c\in \mathcal{C}(s)}V_c(\bm{\mu}) & (L(s)=\text{`MIN'}),\\
      \end{cases}
\end{align*}
where $\mathcal{C}(s)$ is the set of children of $s$.

Whether the answer for the question tree $\mathcal{T}$ is `win' or `lose' depends on a given threshold $\theta \in X$.
We call the nodes $s \in S$ with $V_s(\bm{\mu}) \geq \theta$ good (winning) nodes, and the other nodes bad (losing) nodes.
The player's task is to determine whether the root node $s_0$ is a winning or losing node.
Let $A_s(\bm{\mu}) = \{ c \in \mathcal{C}(s) \mid  V_c(\bm{\mu}) \ge \theta\}$ be the set of good child nodes for each node $s \in S$.
Here, we define $a_s(\bm{\mu})$ by
\begin{align*}
    a_s(\bm{\mu})=
    \begin{cases}
        \text{`win'} & (L(s)=\text{`MAX'},A_{s}(\bm{\mu})\neq \emptyset),\\
        \text{`lose'} & (L(s)=\text{`MAX'},A_{s}(\bm{\mu})= \emptyset),\\
        \text{`win'} & (L(s)=\text{`MIN'},A_{s}(\bm{\mu})= \mathcal{C}(s)),\\
        \text{`lose'} & (L(s)=\text{`MIN'},A_{s}(\bm{\mu})\neq \mathcal{C}(s)),
    \end{cases}
\end{align*}
and $a_s(\bm{\mu}) = \text{`win'}$ is equivalent to $V_s(\bm{\mu}) \geq \theta$.

The function values $V_s(\bm{\mu}), A_s(\bm{\mu})$ and $a_s(\bm{\mu})$ depend solely on $\bm{\mu}_{\mathcal{D}(s)}$.
Thus, we abuse the notations and use the function values even for $\bm{\lambda}=(\lambda_{\ell})_{\ell\in \mathcal{D}(s)}\in X^{|\mathcal{D}(s)|}$ as $V_s(\bm{\lambda}), A_s(\bm{\lambda})$, and $a_s(\bm{\lambda})$.

At the end of any round, the player can choose to stop and output the estimated answer `win' or `lose.' 
The player should follow an algorithm to decide which arm to select as well as when to stop.
Since the rewards are generated stochastically, it is impossible to output a correct answer with $100$\% accuracy.
In the fixed confidence setting, the objective is relaxed to $\delta$-correctness, defined as follows.
\begin{definition}[$\delta$-correctness]
Given a rooted tree $\mathcal{T}=(S,B,s_0)$, its internal node label function $L$, and $\delta>0$,
an algorithm is said to be $\delta$-correct if the algorithm outputs the correct answer $a_{s_0}(\bm{\mu})$ with probability at least $1-\delta$ by sampling rewards from the leaves for any unknown mean vector $\bm{\mu}$ of the leaf reward distributions.
\end{definition}

Let $\tau_{\delta}$ denote the number of reward samples (stopping time) for a given confidence parameter $\delta>0$.
The problem that we consider in this paper is formally stated as follows.
\begin{problem}\label{prob:def}
Design a $\delta$-correct algorithm with minimum $\mathbb{E}[\tau_{\delta}]$ for a given rooted tree $T=(S,B,s_0)$, its internal node label function $L$, and $\delta>0$.
\end{problem}

\section{Optimal Allocation}
\label{sec_optimalallocation}

In this section, we describe the optimal sampling proportions that achieve the asymptotically optimal stopping time with fixed confidence.

We define the alternative set, i.e., the set of mean vectors of leaf rewards that induces a different answer than $\bm{\mu}$ as
\begin{align*}
    \mathrm{Alt}(\bm{\mu}) = \{\bm{\lambda}=(\lambda_{\ell})_{\ell\in \mathcal{L}(\mathcal{T})} \in X^{|\mathcal{L}(\mathcal{T})|} \mid a_{s_0}(\bm{\lambda}) \neq a_{s_0}(\bm{\mu}) \}.
\end{align*}
Let $d(x,y)$ be the Kullback-Leibler divergence between two distributions (in the same exponential family) with mean $x$ and $y$.
Following Theorem~1 of \citet{pmlr-v49-garivier16a}, for any $\delta$-correct algorithm, we can construct the asymptotic lower bound of the expected stopping time $\mathbb{E}[\tau_\delta]$.
\begin{theorem}[\citet{pmlr-v49-garivier16a,DBLP:conf/nips/DegenneK19}]
The following inequality holds:
\begin{align}
    \liminf_{\delta \to 0} \frac{\mathbb{E}_{\bm{\mu}}[\tau_\delta]}{\log(1 / \delta)}
    \geq \left( \max_{\bm{w} \in \Delta} \inf_{\bm{\lambda} \in \mathrm{Alt}(\bm{\mu})} \sum_{\ell \in \mathcal{L}(\mathcal{T})} w_{\ell} d(\mu_{\ell},\lambda_{\ell}) \right)^{-1}, \label{ineq:lowerbound}
\end{align}
where $\Delta$ is the simplex over leaves
\begin{align*}
    \Delta = \left\{ (w_\ell)_{\ell \in \mathcal{L}(\mathcal{T})} \mid \sum_\ell w_\ell = 1, w_\ell \geq 0\ (\ell\in \mathcal{L}(\mathcal{T})) \right\}.
\end{align*}
\end{theorem}

Ineq. (\ref{ineq:lowerbound}) involves an optimization problem over $\bm{w}\in \Delta$, and the maximizer
\begin{align}
    \bm{w}(\bm{\mu}) &= (w_{\ell}(\bm{\mu}))_{\ell\in \mathcal{L}(\mathcal{T})} \nonumber\\
    &= \argmax_{\bm{w} \in \Delta}\displaystyle\inf_{\bm{\lambda} \in \mathrm{Alt}(\bm{\mu})} \sum_{\ell \in \mathcal{L}(\mathcal{T})} w_{\ell} d(\mu_{\ell},\lambda_{\ell}) \label{optimalweight}
\end{align}
is the asymptotically optimal sampling proportion over leaves.

To solve the optimization problem of (\ref{optimalweight}) recursively,
we consider a subproblem for the subtree rooted at node $s$.
Define the restrictions of $\mathrm{Alt}(\bm{\mu})$ and $\Delta$ to the descendant leaves $\mathcal{D}(s)$, denoted by $\mathrm{Alt}_{s}(\bm{\mu})$ and $\Delta_{s}$, respectively, as
\begin{align*}
  \mathrm{Alt}_{s}(\bm{\mu}) = \{\bm{\lambda}=(\lambda_{\ell})_{\ell\in\mathcal{D}(s)}\in X^{|\mathcal{D}(s)|}\mid a_s(\bm{\lambda})\neq a_{s_0}(\bm{\mu})\}
\end{align*}
and
\begin{align*}  
    \Delta_{s} = \left\{ (w_{\ell})_{\ell \in \mathcal{D}(s)} \mid \sum_{\ell\in \mathcal{D}(s)} w_{\ell}=1, w_{\ell} \ge 0\  (\ell \in \mathcal{D}(s))\right\}.
\end{align*}
We consider the following optimization problem for the subtree rooted at $s$.
\begin{align}
    \bm{w}^s(\bm{\mu}) =
    \argmax_{\bm{w} \in \Delta_{s}}\displaystyle\inf_{\bm{\lambda} \in \mathrm{Alt}_{s}(\bm{\mu})} \sum_{\ell \in \mathcal{D}(s)} w_{\ell} d(\mu_{\ell},\lambda_{\ell}). \label{optimalweight_sub}
\end{align}
Let $d_s(\bm{\mu})$ be the optimal value of the optimization problem of (\ref{optimalweight_sub}), that is,
\begin{align}
    d_s(\bm{\mu}) =
    \max_{\bm{w} \in \Delta_{s}}\displaystyle\inf_{\bm{\lambda} \in \mathrm{Alt}_{s}(\bm{\mu})} \sum_{\ell \in \mathcal{D}(s)} w_{\ell} d(\mu_{\ell},\lambda_{\ell}). \label{optimalvalue_sub}
\end{align}

\captionsetup[algorithm]{name=Recursive Formula}
\begin{algorithm}[tb]
\caption{Calculation of $\bm{w}^s(\bm{\mu})$ (1/2)}\label{rf:optimal_weight_win}
Case with $a_{s_0}(\bm{\mu})=$`win'
\begin{align*}
    &d_s(\bm{\mu}) \\
    =&\begin{cases}
        d(\mu_s,\theta) & (s\in \mathcal{L}(T), \mu_s\ge \theta)\\
      \displaystyle\max_{c\in \mathcal{C}(s)}d_c(\bm{\mu}) &
       (L(s)=\text{`MAX'})\\
       \displaystyle\frac{1}{\displaystyle\sum_{c\in \mathcal{C}(s)}1/d_c(\bm{\mu})} & \left( \begin{aligned}
            &L(s)=\text{`MIN'},\\& d_c(\bm{\mu})>0 \text{ for all } c\in \mathcal{C}(s)
        \end{aligned} \right)\\
         0 & (\text{otherwise})
    \end{cases}\\
    &w^s_{\ell}(\bm{\mu}) \\
    =&\begin{cases}
         1 & (s\in \mathcal{L}(T))\\
          w^{c^*(s)}_\ell(\bm{\mu}) &
          \left( \begin{aligned}
            &L(s)=\text{`MAX'},\\& d_s(\bm{\mu}) > 0,\ \ell \in \mathcal{D}(c^*(s))
        \end{aligned}\right)\\
        0 & \left( \begin{aligned}
            &L(s)=\text{`MAX'},\\ &d_s(\bm{\mu}) > 0,\ 
            \ell \in \mathcal{D}(c),\\& c \in \mathcal{C}(s) \setminus \{ c^*(s) \}
        \end{aligned}\right)\\
\displaystyle\frac{w^c_\ell(\bm{\mu}) / d_c(\bm{\mu})}{\displaystyle\sum_{c' \in \mathcal{C}(s)} 1 / d_{c'}(\bm{\mu})}  &
        \left( \begin{aligned}
            &L(s)=\text{`MIN'},\\ &d_s(\bm{\mu}) > 0, \
            \ell \in \mathcal{D}(c),\\ & c \in \mathcal{C}(s)
        \end{aligned}\right)\\
        \begin{aligned}
        \text{any }& w_{\ell}^s \text{ with}\\&\bm{w}^s\in \Delta_{s}
        \end{aligned}&
        (\text{otherwise}) 
    \end{cases}
\end{align*}
\end{algorithm}

Let $\displaystyle c^*(s) \in \argmax_{c \in \mathcal{C}(s)} d_{c}(\bm{\mu})$.
When $a_{s_0}(\bm{\mu})=$`win', the optimization problem of (\ref{optimalweight_sub}) for the subtree rooted at $s$ can be solved recursively using Recursive Formula~\ref{rf:optimal_weight_win}.
When $a_{s_0}(\bm{\mu})=$`lose', it can be solved by a symmetric formula obtained by swapping the `MAX' and `MIN' in Recursive Formula~\ref{rf:optimal_weight_win}, which is given in Recursive Formula~\ref{rf:optimal_weight_lose} in Appendix~\ref{append:symmetric_omitted_rf}.

\begin{theorem}\label{thm:optimal_weight}
  The vector $\bm{w}^{s}(\bm{\mu})$ and the scalar $d_{s}(\bm{\mu})$ calculated by Recursive Formula~\ref{rf:optimal_weight_win} and \ref{rf:optimal_weight_lose}  are the solutions of the optimization problem defined by Eqs.~(\ref{optimalweight_sub}) and (\ref{optimalvalue_sub}).
\end{theorem}
The proof of this Theorem is provided in Appendix~\ref{append:optimal_weight}.

Therefore, Ineq.~(\ref{ineq:lowerbound}) is rewritten as
\begin{align*}
    \liminf_{\delta \to 0} \frac{ \mathbb{E}[\tau_\delta] }{\log(1 / \delta)} \geq
    \frac{1}{d_{s_0}(\bm{\mu})}.
\end{align*}

\section{Track-and-Stop Algorithms}
\label{sec_algorithm}

This section introduces RD-Tracking-TMCTS, a Track-and-Stop algorithm for the thresholding MCTS problem. 
A Track-and-Stop algorithm \cite{pmlr-v49-garivier16a,DBLP:conf/nips/DegenneK19} draws arms proportional to the optimal weights (Eq.~\eqref{optimalweight}) where the true mean vector is replaced by its empirical analogue. 
The design of Track-and-Stop algorithms is done by determining a sampling rule, a stopping rule, and a recommendation rule. The recommendation rule we adopt here is very simple. We recommend $a_{s_0}(\hat{\bm{\mu}}(t))$ as the answer for the given $\mathcal{T}$.
We next describe our design of stopping and sampling rules. 

\subsection{Stopping Rule}

\begin{algorithm}[tb]
\caption{GLR calculation (1/2)}\label{rf:GLR_win}
Case with $a_{s_0}(\hat{\bm{\mu}}(t))=$`win'
\begin{align*}
    Z_s(t)
    =&\begin{cases}
        N_s(t)d(\hat{\mu}_s(t),\theta) & (s\in \mathcal{L}(T), \hat{\mu}_s(t)\ge \theta)\\
      \displaystyle\sum_{c\in \mathcal{C}(s)}Z_c(t) &
       (L(s)=\text{`MAX'})\\
       \displaystyle\min_{c\in \mathcal{C}(s)}Z_c(t)  & (L(s)=\text{`MIN'})\\
         0 & (\text{otherwise})
    \end{cases}
\end{align*}
\end{algorithm}

\captionsetup[algorithm]{name=Algorithm}
\addtocounter{algorithm}{-2}

To guarantee $\delta$-correctness,
we adopt the GLR (Generalized Likelihood Ratio) stopping rule \citep{kaufmann2021}.

Let $R(t) = \{ r(u) \}_{u = 1, \ldots, t}$ be the set of observed rewards until time $t$, and $R_s(t) = \{ r(u) \in R(t) \mid I(u) \in \mathcal{D}(s) \}$
be the set of rewards observed in the descendants of node $s$.

To evaluate the reliability of $a_s(\hat{\bm{\mu}}_{\mathcal{D}(s)}(t))$, we introduce the Generalized Likelihood Ratio (GLR) statistic $Z_s(t)$:
\begin{align}  
  Z_s(t)  
  =&
  \log
  \frac{\sup_{\bm{\lambda} \in X^{|\mathcal{D}(s)|}} \Lambda(\bm{\lambda} \mid R_s(t))}
       {\sup_{\bm{\lambda} \in \mathrm{Alt}_{s}(\hat{\bm{\mu}}(t))} \Lambda(\bm{\lambda} \mid R_s(t))}\nonumber\\
  =& \inf_{\bm{\lambda} \in \mathrm{Alt}_{s}(\hat{\bm{\mu}}(t))} \sum_{\ell\in \mathcal{D}(s)}N_{\ell}(t)d(\hat{\mu}_{\ell}(t),\lambda_{\ell})\label{eq:GLR}
\end{align}
where $\Lambda(\bm{\lambda} \mid R_s(t))$ is the likelihood of $\bm{\lambda}$ given the set of observed rewards $R_s(t)$.

When $a_{s_0}(\hat{\bm{\mu}}(t))=$`win', GLR statistic $Z_s(t)$ for any node $s$ can be computed recursively using Recursive Formula~\ref{rf:GLR_win}.
When $a_{s_0}(\hat{\bm{\mu}}(t))=$`lose', it can be computed by the symmetric counterpart obtained by exchanging `MAX' and `MIN' labels in Recursive Formula~\ref{rf:GLR_win}, which is given in Recursive Formula~\ref{rf:GLR_lose} in Appendix~\ref{append:symmetric_omitted_rf}.
The following Theorem~\ref{thm:GLR} holds.
\begin{theorem}\label{thm:GLR}
    For any node $s \in S$, the GLR statistic $Z_s(t)$ defined by Eq.~(\ref{eq:GLR})
    can be computed using Recursive Formula~\ref{rf:GLR_win} and \ref{rf:GLR_lose}. 
\end{theorem}
The proof of this theorem is provided in Appendix~\ref{append:GLR}.

The GLR stopping rule stops the algorithm at the first time $t$ such that $Z_{s_0}(t) \geq \beta(t,\delta)$, where $(\beta(t,\delta))_{t\in \mathbb{N}}$ is some sequence of thresholds.
The stopping time $\tau_\delta$ induced by this stopping rule can be represented as
\begin{equation}
  \tau_\delta = \inf \left\{ t \in \mathbb{N} : Z_{s_0}(t) \ge \beta(t,\delta) \right\}.\label{eq:stopping_rule}
\end{equation}
As a corollary of Proposition~15 in \citet{kaufmann2021}, $\delta$-correctness of any algorithm using this stopping rule is guaranteed if $\beta(t, \delta)$ is set as
\begin{align}
  &\beta(t,\delta)\nonumber\\ 
  =& 3\sum_{\ell\in \mathcal{L}(\mathcal{T})}\log(1 + \log N_{\ell}(t)) + |\mathcal{L}(\mathcal{T})|\mathcal{C}_{\text{exp}}\left(\frac{\log (1 / \delta)}{|\mathcal{L}(\mathcal{T})|}\right)\label{eq:betadef}
  \end{align}
where
\begin{align}
\mathcal{C}_{\text{exp}}(x)=&2\tilde{h}\left(\frac{h^{-1}(1+x)+\ln(\pi^2/3)}{2}\right),\label{Cexp}\\
\tilde{h}=&\begin{cases}
    e^{1/h^{-1}(x)}h^{-1}(x) & \text{if } x\ge h(1/\ln (3/2))\\
    \frac{3}{2}\left(x-\ln\ln\frac{3}{2}\right) & \text{otherwise}
\end{cases},\nonumber
\end{align}
and $h^{-1}(x)$ is the inverse function of $h(x)=x-\ln x$.
    
\begin{corollary}[Corollary of Proposition 15 in \cite{kaufmann2021}]\label{cor:kaufmann_prop15}
Any algorithm for Problem~\ref{prob:def} using the GLR stopping rule with the sequence of thresholds $(\beta(t,\delta))_{t\in \mathbb{N}}$ defined by Eq.~(\ref{eq:betadef})
ensures
\[
\mathbb{P}(\tau_{\delta}<\infty,a_{s_0}(\hat{\bm{\mu}}(\tau_{\delta}))\neq a_{s_0}(\bm{\mu}))\le \delta.
\]
\end{corollary}

\subsection{Sampling Rules}\label{sec:sampling_rule}

For the Best Arm Identification problem, \citet{pmlr-v49-garivier16a} proposed two Track-and-Stop sampling rules called C-Tracking and D-Tracking.
Our algorithm is based on the D-Tracking sampling rule (Eq.~(\ref{eq:sampling_rule_D})).
\begin{align} 
    &I(t) \nonumber\\ 
    \gets\mbox{}& \begin{cases}
        \argmin_{\ell\in \mathcal{L}(\mathcal{T})}N_{\ell}(t-1) & \\
        \multicolumn{2}{r}{\text{if } N_{\ell}(t-1) < \sqrt{t} - |\mathcal{L}(\mathcal{T})|/2 \text{ for some } \ell\in \mathcal{L}(\mathcal{T})}\\
        \argmax_{\ell \in \mathcal{L}(\mathcal{T})} t w^{s_0}_{\ell}(\hat{\bm{\mu}}(t-1)) - N_{\ell}(t-1) & \text{otherwise}
    \end{cases}
    \label{eq:sampling_rule_D}
\end{align}
This algorithm tracks the true optimal arm draw proportion through its plug-in estimator and draws the most under-sampled arm.
To ensure consistent tracking, it forces uniform exploration via the first case of Eq.~(\ref{eq:sampling_rule_D}), which makes the empirical means converge to their true values.

While it achieves an asymptotically optimal performance, its empirical performance is compromised by the instability of the selection rule, as we demonstrate later in Section~\ref{sec:advantage}. 
In our experiments, the empirical sample complexity of D-tracking when we adopt Eq.~\eqref{eq:sampling_rule_D} is an order of magnitude worse than the theoretical bound.
To improve the performance, we propose a new sampling rule based on D-Tracking, \textbf{RD-Tracking-TMCTS} (Eq.~(\ref{eq:sampling_rule_ratio})).
\begin{align} 
    &I(t) \nonumber \\
    \gets\mbox{}& 
    \begin{cases}
        \argmin_{\ell\in \mathcal{L}(\mathcal{T})}N_{\ell}(t-1) & \\
        \multicolumn{2}{r}{\text{if } N_{\ell}(t-1) < \sqrt{t} - |\mathcal{L}(\mathcal{T})|/2 \text{ for some } \ell\in \mathcal{L}(\mathcal{T})}\\
        \argmax_{\ell \in \mathcal{L}(\mathcal{T})} w^{s_0}_{\ell}(\hat{\bm{\mu}}(t-1)) / N_{\ell}(t-1) & \text{otherwise}
    \end{cases}
    \label{eq:sampling_rule_ratio}
\end{align}
This sampling rule has the same forced exploration rule as D-Tracking.
Instead of selecting an arm by maximizing $t w^{s_0}_{\ell}(\hat{\bm{\mu}}(t)) - N_{\ell}(t)$, this tracks the target proportion through a ratio-based criterion involving $w^{s_0}_{\ell}(\hat{\bm{\mu}}(t))$ and $N_{\ell}(t)$.
This simple update leads to a several-fold reduction in empirical sample complexity.

\subsection{Advantages of the Ratio-based Sampling}\label{sec:advantage}

For ease of discussion, suppose that there exists a time $T_\mathrm{conv}$ such that $\hat{\bm{\mu}}(t) = \bm{\mu}$ for all $t \geq T_\mathrm{conv}$.
Let $T^* = \frac{\log (1/\delta)}{d_{s_0}(\bm{\mu})}$ denote a lower bound on the expected stopping time.
If the algorithm samples according to the optimal allocation after time $T_\mathrm{conv}$, it would stop by time $T_\mathrm{conv} + T^*$.
In other words, once the empirical means converged, selecting an arm $\ell$ more than $T^* w^{s_0}_\ell(\bm{\mu})$ times is unnecessary.

However, the original D-Tracking can sample more than $T^* w^{s_0}_\ell(\bm{\mu})$ times for some arm $\ell$, even after the empirical means have sufficiently converged.

\begin{example}
    We consider a depth-2 tree with $\theta = 0.5$, $\delta = 10^{-10}$, $\bm{w}^{s_0} = (0.920, 0.056, 0.024, 0, 0, 0)$ (Figure~\ref{fig:example_prob}).
    In this case, the lower bound of the expected stopping time is $T^* = 4996.8$.
    \begin{figure}
        \centering
        \includegraphics[width=0.95\linewidth]{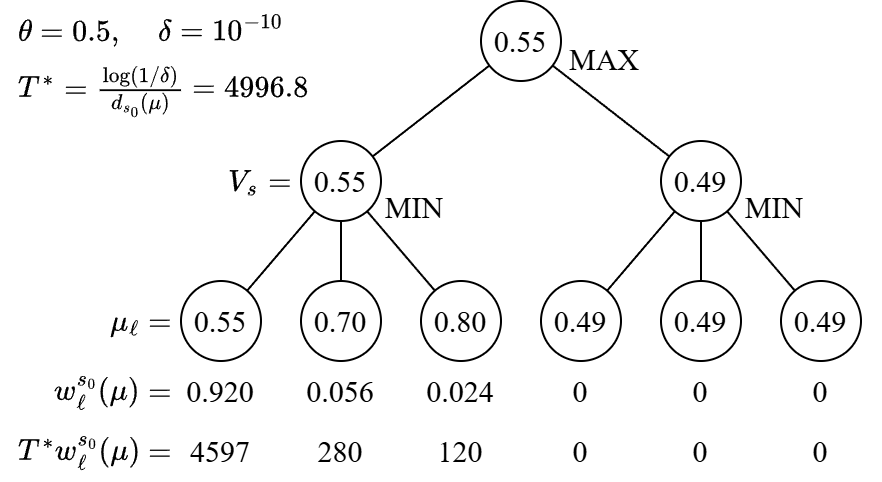}
        \caption{A six-armed example to demonstrate the inefficiency of the original D-tracking.}
        \label{fig:example_prob}
    \end{figure}
    
Suppose that $(N_\ell(T_\mathrm{conv}))_\ell = (100, 100, 100, 3300,\\ 3300, 3400)$ at $T_\mathrm{conv} = 10000$, and that $\hat{\bm{\mu}}(t) = \bm{\mu}$ for all $t > T_\mathrm{conv}$.
Here, the optimal weight is $\bm{w}^{s_0}(\bm{\mu}) = (0.92, 0.056, 0.024,0,0,0)$, and an algorithm stops after drawing each arm $\ell = 1,2,3$ approximately $T^* w_\ell^{s_0}(\bm{\mu})$ times. 
This corresponds to the scenario in which the sampling proportion was largely misspecified before round $T_\mathrm{conv} = 10000$, but the algorithm identified the optimal weight after round $T_\mathrm{conv}$.

In this case, there is a clear difference in the behavior of the original D-tracking and our ratio-based D-tracking. 
Even though $w_\ell^{s_0}(\bm{\mu})$ is correctly specified for $t > T_\mathrm{conv}$, the original D-tracking rule (Eq.~(\ref{eq:sampling_rule_D})) continues to draw arm~$1$ as long as it maximizes the second case of Eq.~(\ref{eq:sampling_rule_D}). 
As a result, $N_1(t)$ is driven close to $t w_1^{s_0}(\bm{\mu})$, leading to substantial oversampling.
In this case, $T_\mathrm{conv} \bm{w}^{s_0}(\bm{\mu}) = (9200, 560, 240, 0, 0, 0)$, and it continues to draw arm~$1$ until $N_1 \approx 9000$.
This leads to an oversampling of arm~$1$ well beyond the true optimal allocation $T^* w^{s_0}_1 = 4597$, even after the empirical mean has converged.
\end{example}

In contrast, RD-Tracking-TMCTS, our ratio-based variant of D-tracking that draws arms in accordance with Eq.~\eqref{eq:sampling_rule_ratio}, avoids such oversampling of the arm~$1$ once the empirical mean has converged.
\begin{lemma}
Assume that $\hat{\bm{\mu}}(t) = \bm{\mu}$ at round $t$. Then, under \textsc{RD-Tracking-TMCTS} (Eq.~\eqref{eq:sampling_rule_ratio}, second case), no leaf $\ell$ with
\[
N_\ell(t) > T^* w^{s_0}_\ell(\bm{\mu})
\]
is selected, unless this condition holds for all leaves.
\end{lemma}
This lemma can be proved using a similar argument to Lemma~\ref{lem:DrawProportionConvergence}.
As a result, if $\hat{\bm{\mu}}(t) = \bm{\mu}$ holds for all $t \ge T_\mathrm{conv}$, RD-Tracking-TMCTS selects leaves so that, by approximately time
$T_\mathrm{conv} + T^*$, $N_\ell > T^* w^{s_0}_\ell(\bm{\mu})$ holds for all arm $\ell$. 
Therefore, the algorithm stops by approximately time $T_\mathrm{conv} + T^*$.

Because the GLR involves a minimum over leaves, the algorithm cannot stop while there exists a leaf satisfying $N_\ell < T^* w^{s_0}_\ell(\bm{\mu})$ (as long as the empirical means do not deviate significantly).
Consequently, once the empirical means have converged, the time required to reach stopping is always shorter or equal for the proposed method; it is never longer.
Moreover, since both methods employ the same forced-exploration mechanism, the time until the empirical means converge is essentially the same, and thus the proposed method is expected to stop earlier.
This difference in expected stopping times is confirmed experimentally in Section~\ref{sec:experiment}.

Moreover, the ratio-based sampling of Eq.~\eqref{eq:sampling_rule_ratio} can run in logarithmic time of the tree per round, as we will describe in Section~\ref{sec:time_complexity}.

\subsection{Algorithms}

Algorithm~\ref{alg:proposed} is the pseudocode of RD-Tracking-TMCTS.

\begin{algorithm}[tb]
\caption{Ratio-based D-Tracking for Thresholding MCTS (RD-Tracking-TMCTS)} \label{alg:proposed}
\begin{algorithmic}[1]
  \Require $\mathcal{T}$: tree, $\theta$: threshold, $\delta$: confidence parameter
  \Ensure Return $a_{s_0}(\bm{\mu}) \in \{\text{`win'}, \text{`lose'}\}$ correctly with prob. at least $1 - \delta$
  \State Draw each leaf node $\ell\in\mathcal{L}(\mathcal{T})$ once
  \For{$t = 1, 2, \ldots$}
    \State $I(t) \gets \text{Eq.~(\ref{eq:sampling_rule_ratio})}$
    \State Observe reward $r(t)$ from leaf $I(t)$ and update $N_{I(t)}(t)$ and $\hat{\mu}_{I(t)}(t)$
    \If{$Z_{s_0}(t)\ge\beta(t,\delta)$} \label{line:stopping_condition}
      \State \textbf{Return} $a_{s_0}(\hat{\bm{\mu}}(t))$\label{line:output}
    \EndIf
  \EndFor
\end{algorithmic}
\end{algorithm}

\section{Analysis}

In Section~\ref{sec:assumption}, we describe the assumptions necessary for our analysis.
In Sections~\ref{sec:as_upper_bound} and \ref{sec:sample_complexity}, we derive an upper bound on the stopping time for Algorithm~\ref{alg:proposed}.
Furthermore, Section~\ref{sec:time_complexity} shows the computational cost at each time step $t$ for Algorithm~\ref{alg:proposed}.

\subsection{Assumptions on the True Parameters}\label{sec:assumption}

We assume that the true parameters $\bm{\mu}$ satisfy the following conditions.

\begin{assumption} \label{assumption:no_exact_theta}
  $V_{s_0}(\bm{\mu})\neq\theta$.  
\end{assumption}

\begin{assumption} \label{assumption:unique_weight}
    The optimal proportion $\bm{w}^{s_0}(\bm{\mu})$ is unique.
\end{assumption}

Assumption~\ref{assumption:no_exact_theta} is required to ensure that the algorithm terminates.
Assumption~\ref{assumption:unique_weight} is used to establish the convergence of the estimated optimal weights.
Note that if $c^*(s)$ is unique for every $s \in S$, Assumption~\ref{assumption:unique_weight} holds.

\subsection{Almost-Sure Upper Bound of Sample Complexity} \label{sec:as_upper_bound}

The following theorem is the almost sure upper bound of the stopping time for the class of algorithms whose sampling proportion of leaf nodes converges to the optimal proportion.
\begin{theorem}\label{thm:as_upper_bound}
    Using GLR stopping rule with the threshold $\beta(t,\delta)$ defined by Eq.~(\ref{eq:betadef}), 
    and any sampling rule ensuring that for every leaf $\ell \in \mathcal{L}(T)$, $N_\ell(t) / t$ converges almost surely to $w^{s_0}_\ell(\bm{\mu})$, any algorithm for Problem~\ref{prob:def} guarantees that for all $\delta \in (0,1)$,
    $\mathbb{P}(\tau_\delta < +\infty) = 1$ and
    \[ \mathbb{P}_{\bm{\mu}} \left( \limsup_{\delta \to 0} \frac{\tau_\delta}{\log(1 / \delta)} \leq \frac{1}{d_{s_0}(\bm{\mu})} \right) = 1. \]
\end{theorem}
The proof of this theorem is provided in Appendix~\ref{append:as_upper_bound}.

The following lemma guarantees that the sampling ratio of RD-Tracking-TMCTS converges to the optimal ratio.
(See Appendix~\ref{appendix:RatioConvergence} for the proof.)
\begin{lemma}\label{lem:sampling_proportion_convergence}
    RD-Tracking-TMCTS sampling rule satisfies
    \[ \mathbb{P}\left( \lim_{t \to \infty} \frac{N_\ell(t)}{t} = w^{s_0}_\ell(\bm{\mu}) \right) = 1.\]
\end{lemma}

We have the following corollary by Theorem~\ref{thm:as_upper_bound} and Lemma~\ref{lem:sampling_proportion_convergence}.
\begin{corollary}
 Using the threshold sequence $(\beta(t,\delta))_{t\in \mathbb{N}}$ defined by Eq.~(\ref{eq:betadef}), 
 RD-Tracking-TMCTS guarantees that for all $\delta \in (0,1)$,
    $\mathbb{P}(\tau_\delta < +\infty) = 1$ and
    \[ \mathbb{P}_{\bm{\mu}} \left( \limsup_{\delta \to 0} \frac{\tau_\delta}{\log(1 / \delta)} \leq \frac{1}{d_{s_0}(\bm{\mu})} \right) = 1. \]   
\end{corollary}

\subsection{Upper Bound of Expected Sample Complexity}\label{sec:sample_complexity}

The following theorem states that RD-Tracking-TMCTS achieves asymptotically optimal expected sample complexity.
\begin{theorem}\label{thm:sample_complexity}
    For RD-Tracking-TMCTS with the threshold sequence $(\beta(t,\delta))_{t\in \mathbb{N}}$ defined by Eq.~(\ref{eq:betadef}), expected stopping time $\mathbb{E}_{\bm{\mu}}[\tau_\delta]$ is asymptotically upper-bounded by 
    \[ \limsup_{\delta \to +0} \frac{\mathbb{E}_{\bm{\mu}}[\tau_\delta]}{\log(1 / \delta)} \leq \frac{1}{ d_{s_0}(\bm{\mu})}. \]
\end{theorem}

The proof of this theorem is provided in Appendix~\ref{append:sample_complexity}.

Note that, given the matching lower bound, Theorem \ref{thm:sample_complexity}, which bounds the expected stopping time, is slightly stronger than Theorem \ref{thm:as_upper_bound}, which bounds the stopping time with probability $1$. This is because a sequence of random variables that converges to a constant with probability $1$ can have a larger (or even diverging) expected value than the constant.

\subsection{Time Complexity}\label{sec:time_complexity}
Our tracking algorithm maintains the statistics of $|S|$ nodes, and its space complexity is $O(|S|)$. 
In this section, we consider the time complexity of calculating $I(t)$ in Eq.~\eqref{eq:sampling_rule_ratio}.
The first line of this equation, which corresponds to the forced exploration, can be implemented in $O(\log |\mathcal{L}(\mathcal{T})|)$ time by using a heap keyed by the number of observations at each leaf.
The more nontrivial part is the second line of this equation, which corresponds to the ratio-based sampling. To compute $I(t) = \argmax_{\ell \in \mathcal{L}(\mathcal{T})} w^{s_0}_{\ell}(\hat{\bm{\mu}}(t-1)) / N_{\ell}(t-1)$, we need to update the weight vector $\bm{w}^{s_0}(\hat{\bm{\mu}}(t-1))$ after each draw as well as choose its maximizer in a computationally efficient way.

A straightforward implementation of Track-and-Stop algorithms requires recomputation of the entire weight vector $\bm{w}^{s_0}(\bm{\mu})$ at each step, which takes $O(|S|)$ time using Recursive Formula~\ref{rf:optimal_weight_win} and \ref{rf:optimal_weight_lose}, since it visits all nodes at least once.
However, our proposed RD-Tracking-TMCTS admits a more efficient implementation. 

Instead of computing $I(t)$ directly, we consider the following subproblem for each subtree rooted at $s$:
\[I_s(t) = \argmax_{\ell \in \mathcal{D}(s)} w^s_\ell(\hat{\bm{\mu}}(t)) / N_\ell(t).\]
We can solve this subproblem recursively.
\begin{lemma} \label{lem:representative_leaf}
    $I_s(t) \in \{I_c(t)\}_{c \in \mathcal{C}(s)}$.
\end{lemma}
Each solution $I_s(t)$ depends on the information in the subtree rooted at $s$, and thus remains unchanged unless an observation occurs within this subtree.
\footnote{Strictly speaking, $\bm{w}^s(\hat{\bm{\mu}}(t))$ depends on $a_{s_0}(\hat{\bm{\mu}}(t))$. Therefore, we redefine $I_s(t)$ and introduce auxiliary variables that remove this dependence. (See Appendix~\ref{append:time_complexity}.)}
This enables updating the $I_{s_0}(t)$ via a backpropagation along a single root-to-leaf path (Figure~\ref{fig:updating_repr_leaf}).
\begin{figure}
    \centering
    \includegraphics[width=0.87\linewidth]{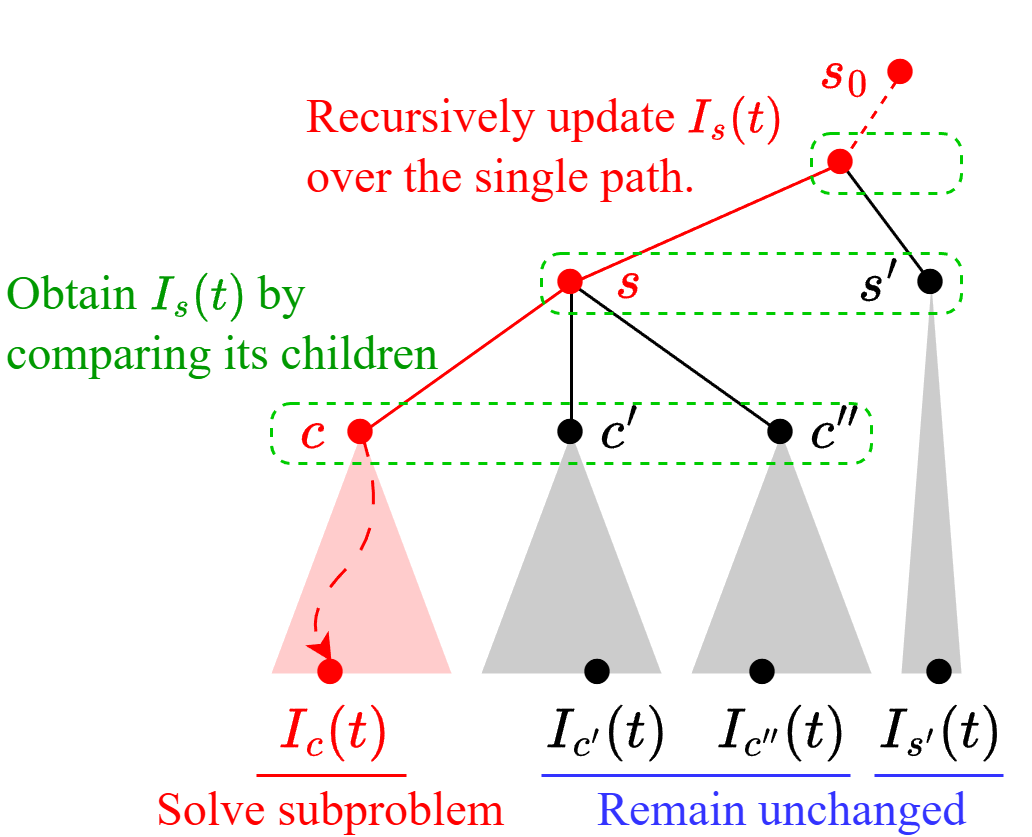}
    \caption{Recursive calculation of $I_s(t)$}
    \label{fig:updating_repr_leaf}
\end{figure}

We formally state the time complexity of RD-Tracking-TMCTS.
\begin{theorem} \label{thm:time_complexity}
    Let $D$ be the depth of the tree $\mathcal{T}$. Let $K$ be the maximum number of children $\max_{s \in S} |\mathcal{C}(s)|$.
    RD-Tracking-TMCTS (Algorithm~\ref{alg:proposed}) runs in $O(D \log K)$ time per round.
\end{theorem}
Since $|\mathcal{L}(\mathcal{T})| \leq K^D$, we have $\log |\mathcal{L}(\mathcal{T})| \leq D \log K$. 
If the tree is balanced, that is, if $|\mathcal{L}(\mathcal{T})| = \Theta(K^D)$ holds, then $D \log K = \Theta(\log |\mathcal{L}(\mathcal{T})|)$ (and also $\Theta(\log |S|)$) holds.

Theorem~\ref{thm:time_complexity} is achieved through the careful construction of data structures and algorithmic implementations, which we elaborate on in Appendix~\ref{append:time_complexity}.

\section{Experimental Results}\label{sec:experiment}

We conduct experiments on instances where the label of the root node $s_0$ satisfies $L(s_0) = \text{`MAX'}$.

To validate the effectiveness of our method, RD-Tracking-TMCTS, we compare it with several baseline algorithms.
Specifically, we use D-Tracking~\citep{pmlr-v49-garivier16a}, C-Tracking~\citep{pmlr-v49-garivier16a}, UGapE-MCTS~\citep{DBLP:conf/nips/KaufmannK17}, LUCB-micro~\citep{DBLP:conf/alt/HuangAS017}, and round-robin (= uniform sampling) as baselines.

D-Tracking and C-Tracking are instances of the Track-and-Stop algorithm.
To apply these algorithms to our Thresholding MCTS problem, we made them select a leaf following the plug-in estimator of the optimal weight $\bm{w}(\hat{\bm{\mu}}(t))$, which leads to convergence of the sampling proportion to the optimal weight $\bm{w}(\bm{\mu})$.
For fairness, we employ the same stopping rule as in the proposed method and use the same confidence function $\beta(t, \delta)$.
Since these methods ensure that the selection proportions converge to $\bm{w}(\bm{\mu}(t))$, they satisfy the conditions of Theorem~\ref{thm:as_upper_bound}.

UGapE-MCTS and LUCB-micro are confidence-interval-based algorithms designed for the Best Arm Identification problem.
Since these algorithms identify $s^* = \argmax_{c \in \mathcal{C}(s_0)} V_c(\bm{\mu})$, we modify them to instead return whether the value of $s^*$ is at least $\theta$.
For UGapE-MCTS, we set the tolerance parameter $\epsilon$ to zero, thereby requiring the algorithm to identify the exact $s^*$.
For fairness, we define the confidence interval at the leaf node $\ell$ as $\left[ \hat{\mu}_\ell(t) - \sqrt{\frac{\beta(N_\ell(t), \delta)}{2 N_\ell(t)}}, \hat{\mu}_\ell(t) + \sqrt{\frac{\beta(N_\ell(t), \delta)}{2 N_\ell(t)}} \right]$, and we use the same function $\beta(t,\delta)$ as that employed in the proposed method.

The round-robin algorithm pulls each leaf in a fixed cyclic order. 
We employ the same stopping rule as in the proposed method and use the same confidence function $\beta(t, \delta)$.
Once the condition $Z_{s_0}(t) \geq \beta(t,\delta)$ is met, the algorithm terminates and outputs $a_{s_0}(\hat{\bm{\mu}})$.

We conducted experiments on randomly generated complete 3-ary trees of depths $2, 3, 4,$ and $5$. ($|\mathcal{L}(\mathcal{T})| = 9, 27, 81, 243$).
We set nodes at even depths as MAX nodes and those at odd depths as MIN nodes.
To ensure a fair comparison with Best Arm Identification methods, we set the value of the second-best depth-1 node to $\theta$.
This setting makes the gap between the best and second-best nodes equal to the gap between the best node and the threshold.

We performed experiments with $\delta = 10^{-5}, 10^{-10},\\ 10^{-20}, \ldots, 10^{-60}$. 
For each value of $\delta$, we ran the algorithm $1000$ times and recorded the average and standard deviation of the stopping times.
Figure~\ref{fig:result_stopping_time} shows the ratio of the average stopping time to its asymptotic lower bound $\log(1/\delta) / d_{s_0}(\bm{\mu})$.

\begin{figure*}[t]
  \centering
  \begin{subfigure}{0.24\linewidth}
    \centering
    \includegraphics[width=\linewidth]{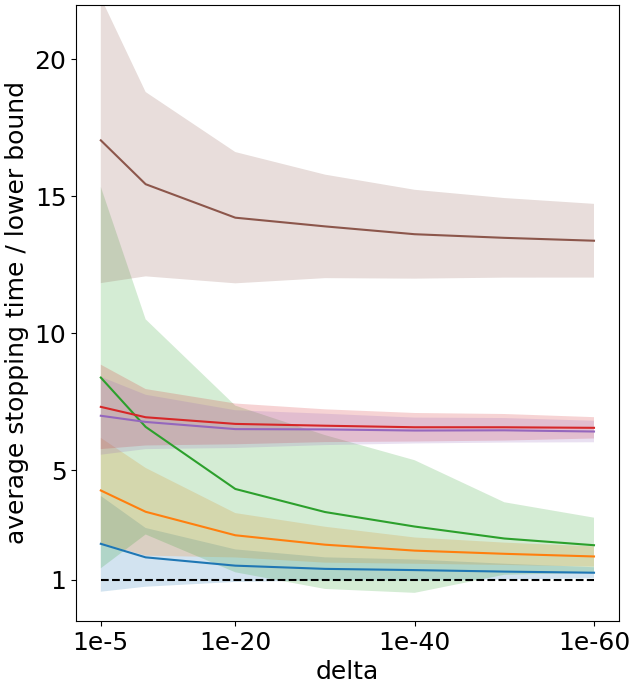}
    \caption{Tree depth $= 2$}
    \label{fig:result_depth2}
  \end{subfigure}
  \hfill
  \begin{subfigure}{0.24\linewidth}
    \centering
    \includegraphics[width=\linewidth]{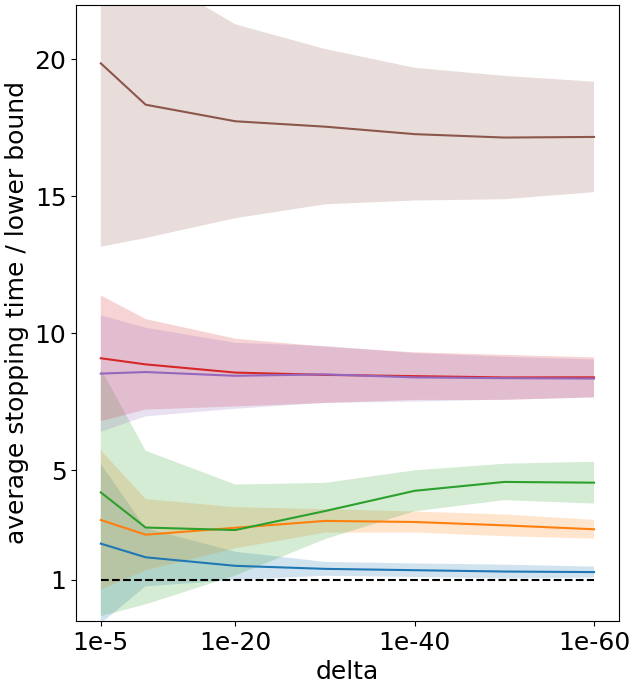}
    \caption{Tree depth $= 3$}
    \label{fig:result_depth3}
  \end{subfigure}
  \hfill
  \begin{subfigure}{0.24\linewidth}
    \centering
    \includegraphics[width=\linewidth]{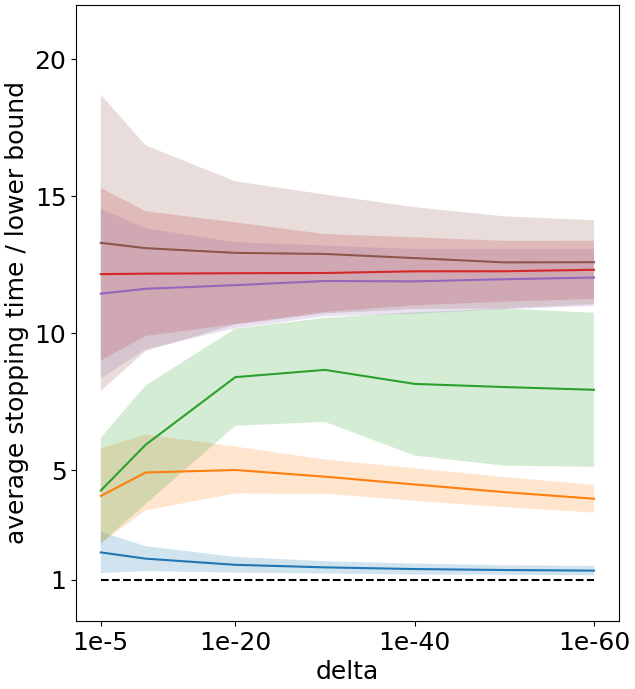}
    \caption{Tree depth $= 4$}
    \label{fig:result_depth4}
  \end{subfigure}
  \hfill
  \begin{subfigure}{0.24\linewidth}
    \centering
    \includegraphics[width=\linewidth]{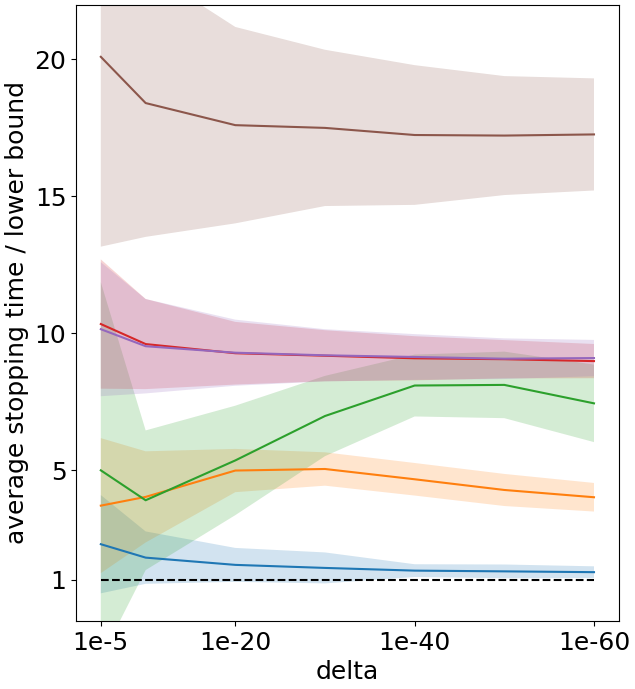}
    \caption{Tree depth $= 5$}
    \label{fig:result_depth5}
  \end{subfigure}
  \includegraphics[width=\linewidth]{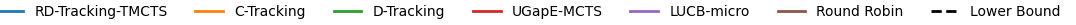}
  \caption{Ratio of the average stopping time to the lower bound for different values of $\delta$. Shaded regions indicate the standard deviation.}
  \label{fig:result_stopping_time}
\end{figure*}

We did not observe any incorrect outputs in $1000$ trials for any setting, which is consistent with\linebreak $\delta$-correctness.

The ratios of the average stopping time to the lower bound for UGapE-MCTS, LUCB-micro, and round-robin do not seem to approach $1$ as $\delta \to 0$, indicating that these methods are not asymptotically optimal.
Although RD-Tracking-TMCTS, as well as C-Tracking and D-Tracking, are all theoretically guaranteed to be asymptotically optimal, the rates at which their stopping times approach the lower bound differ substantially.
In particular, RD-Tracking-TMCTS converges to the lower bound significantly faster than the other methods, with a clear separation in performance.

\section{Good Action Identification\linebreak Problem}\label{sec:good_arm_identification}

Up to now, we have focused on the thresholding bandit problem, where the goal is to determine whether the value of the root node is at least a given threshold $\theta$. 
When $L(s_0) = \text{`MAX'}$, identifying $a_{s_0}(\bm{\mu})$ is equivalent to determining whether there exists some $c \in \mathcal{C}(s_0)$ such that $V_c(\bm{\mu}) > \theta$.
A natural question is to ask whether our framework can be applied to find a good action at the root node $s_0$ with $L(s_0) = \text{`MAX'}$, i.e., the child node whose value is at least $\theta$.
We call this problem the good action identification (GAI) problem and answer this question affirmatively.
The correct answers of GAI problem $a^{\mathrm{GAI}}(\bm{\mu})$ are defined as
\begin{align*}
a^{\mathrm{GAI}}(\bm{\mu})=
\begin{cases}
\text{any } c\in A_{s_0}(\bm{\mu}) & (A_{s_0}(\bm{\mu})\neq \emptyset)\\
\text{``no good action"} & (A_{s_0}(\bm{\mu})=\emptyset).
\end{cases}
\end{align*}
For GAI problem, $\liminf_{\delta\rightarrow 0}\mathbb{E}_{\bm{\mu}}[\tau_{\delta}]/\log(1/\delta)$ is known to be lower bounded \cite{DBLP:conf/nips/DegenneK19} by
\begin{align}
\left(\max_{c\in A_{s_0}(\bm{\mu})}\max_{\bm{w}\in\Delta}\inf_{\bm{\lambda}\in {\mathrm{Alt}^{\mathrm{GAI}}(c)}}\sum_{\ell\in \mathcal{L}(\mathcal{T})}w_{\ell}d(\mu_{\ell},\lambda_{\ell})\right)_{,}^{-1}\label{GAI:LB}
\end{align}
where
\[
\mathrm{Alt}^{\mathrm{GAI}}(c)=\{\bm{\lambda}\in X^{|\mathcal{L}(\mathcal{T})|}\mid c\not\in A_{s_0}(\bm{\lambda})\}.
\]
We can show that (\ref{GAI:LB}) is lower bounded by $d_{s_0}(\bm{\mu})^{-1}$ as the decision version.
An algorithm for GAI problem is constructed by modifying the GLR statistic $Z_{s_0}(t)$ in the stopping rule (line~\ref{line:stopping_condition} of Algorithm~\ref{alg:proposed}) as
\begin{equation*}\label{eq:GLR_GAI}
    Z^\mathrm{GAI}_{s_0}(t) =
    \begin{cases}
        \max_{c \in \mathcal{C}(s_0)} Z_c(t) & (V_{s_0}(\hat{\bm{\mu}}(t)) \ge \theta)\\
        Z_{s_0}(t) & (V_{s_0}(\hat{\bm{\mu}}(t)) < \theta).
    \end{cases}
\end{equation*}
and the output $a_{s_0}(\hat{\bm{\mu}}(t))$ (line~\ref{line:output} of Algorithm~\ref{alg:proposed}) as
\begin{align*}
    \hat{a}^\mathrm{GAI}(t) = \begin{cases}
        \argmax_{c \in \mathcal{C}(s_0)} Z^\mathrm{GAI}_c(t) & (A_{s_0}(\hat{\bm{\mu}}(t)) \neq \emptyset)\\
        \text{``no good action"} & (\text{otherwise}).
    \end{cases}
\end{align*}
Then, the modified algorithm satisfies 
\[
\mathbb{P}[\hat{a}^\mathrm{GAI}(\tau_\delta) \text{ is a correct answer}] \geq 1 - \delta.
\]

Since $Z^\mathrm{GAI}_{s_0}(t) \leq Z_{s_0}(t)$, the empirical stopping time becomes larger than that of the thresholding problem.
However, we can nevertheless derive the same upper bound on the asymptotic sample complexity.
\begin{theorem}\label{thm:sample_complexity_GAI}
    The algorithm obtained by replacing the stopping condition of Algorithm~\ref{alg:proposed} (line~\ref{line:stopping_condition}) with $Z^\mathrm{GAI}_{s_0}(t) \geq \beta(t,\delta)$ satisfies 
    \[\limsup_{\delta \to +0} \frac{ \mathbb{E}[\tau_\delta] }{\log(1 / \delta)} \leq
    \frac{1}{d_{s_0}(\bm{\mu})}.\]
\end{theorem}
This result follows from the fact that, under sufficient convergence of the empirical means, the GLR statistics $Z_{s_0}(t)$ and $Z^\mathrm{GAI}_{s_0}(t)$ grow at the same rate as functions of $t$.
Since their growth behaviors are identical and the same threshold $\beta(t, \delta)$ is used, the resulting stopping times share the same asymptotic behavior.
We provide the proof of these statements in Appendix~\ref{append:GAI}.

\bibliography{paper_arxiv}

@InProceedings{pmlr-v49-garivier16a,
  title = 	 {Optimal Best Arm Identification with Fixed Confidence},
  author = 	 {Garivier, Aurélien and Kaufmann, Emilie},
  booktitle = 	 {29th Annual Conference on Learning Theory},
  pages = 	 {998--1027},
  year = 	 {2016},
  editor = 	 {Feldman, Vitaly and Rakhlin, Alexander and Shamir, Ohad},
  volume = 	 {49},
  series = 	 {Proceedings of Machine Learning Research},
  address = 	 {Columbia University, New York, New York, USA},
  month = 	 {23--26 Jun},
  publisher =    {PMLR},
  url = 	 {https://proceedings.mlr.press/v49/garivier16a.html}
}

@article{kaufmann2021,
  author  = {Emilie Kaufmann and Wouter M. Koolen},
  title   = {Mixture Martingales Revisited with Applications to Sequential Tests and Confidence Intervals},
  journal = {Journal of Machine Learning Research},
  year    = {2021},
  volume  = {22},
  number  = {246},
  pages   = {1--44},
  url     = {http://jmlr.org/papers/v22/18-798.html}
}

@inproceedings{DBLP:conf/nips/KaufmannK17,
  author       = {Emilie Kaufmann and
                  Wouter M. Koolen},
  editor       = {Isabelle Guyon and
                  Ulrike von Luxburg and
                  Samy Bengio and
                  Hanna M. Wallach and
                  Rob Fergus and
                  S. V. N. Vishwanathan and
                  Roman Garnett},
  title        = {Monte-Carlo Tree Search by Best Arm Identification},
  booktitle    = {Advances in Neural Information Processing Systems 30: Annual Conference
                  on Neural Information Processing Systems 2017, December 4-9, 2017,
                  Long Beach, CA, {USA}},
  pages        = {4897--4906},
  year         = {2017},
  url          = {https://proceedings.neurips.cc/paper/2017/hash/a6d259bfbfa2062843ef543e21d7ec8e-Abstract.html},
  bibsource    = {dblp computer science bibliography, https://dblp.org}
}

@inproceedings{DBLP:conf/nips/DegenneK19,
  author       = {R{\'{e}}my Degenne and
                  Wouter M. Koolen},
  editor       = {Hanna M. Wallach and
                  Hugo Larochelle and
                  Alina Beygelzimer and
                  Florence d'Alch{\'{e}}{-}Buc and
                  Emily B. Fox and
                  Roman Garnett},
  title        = {Pure Exploration with Multiple Correct Answers},
  booktitle    = {Advances in Neural Information Processing Systems 32: Annual Conference
                  on Neural Information Processing Systems 2019, NeurIPS 2019, December
                  8-14, 2019, Vancouver, BC, Canada},
  pages        = {14564--14573},
  year         = {2019},
  url          = {https://proceedings.neurips.cc/paper/2019/hash/60cb558c40e4f18479664069d9642d5a-Abstract.html},
  bibsource    = {dblp computer science bibliography, https://dblp.org}
}

@inproceedings{DBLP:conf/ecml/KocsisS06,
  author       = {Levente Kocsis and
                  Csaba Szepesv{\'{a}}ri},
  editor       = {Johannes F{\"{u}}rnkranz and
                  Tobias Scheffer and
                  Myra Spiliopoulou},
  title        = {Bandit Based Monte-Carlo Planning},
  booktitle    = {Machine Learning: {ECML} 2006, 17th European Conference on Machine
                  Learning, Berlin, Germany, September 18-22, 2006, Proceedings},
  series       = {Lecture Notes in Computer Science},
  volume       = {4212},
  pages        = {282--293},
  publisher    = {Springer},
  year         = {2006},
  url          = {https://doi.org/10.1007/11871842\_29},
  doi          = {10.1007/11871842\_29},
  bibsource    = {dblp computer science bibliography, https://dblp.org}
}

@inproceedings{DBLP:conf/icml/LocatelliGC16,
  author       = {Andrea Locatelli and
                  Maurilio Gutzeit and
                  Alexandra Carpentier},
  editor       = {Maria{-}Florina Balcan and
                  Kilian Q. Weinberger},
  title        = {An optimal algorithm for the Thresholding Bandit Problem},
  booktitle    = {Proceedings of the 33nd International Conference on Machine Learning,
                  {ICML} 2016, New York City, NY, USA, June 19-24, 2016},
  series       = {{JMLR} Workshop and Conference Proceedings},
  volume       = {48},
  pages        = {1690--1698},
  publisher    = {JMLR.org},
  year         = {2016},
  url          = {http://proceedings.mlr.press/v48/locatelli16.html},
  bibsource    = {dblp computer science bibliography, https://dblp.org}
}

@article{DBLP:journals/ml/KanoHSMNS19,
  author       = {Hideaki Kano and
                  Junya Honda and
                  Kentaro Sakamaki and
                  Kentaro Matsuura and
                  Atsuyoshi Nakamura and
                  Masashi Sugiyama},
  title        = {Good arm identification via bandit feedback},
  journal      = {Mach. Learn.},
  volume       = {108},
  number       = {5},
  pages        = {721--745},
  year         = {2019},
  url          = {https://doi.org/10.1007/s10994-019-05784-4},
  doi          = {10.1007/S10994-019-05784-4},
  bibsource    = {dblp computer science bibliography, https://dblp.org}
}

@article{DBLP:journals/jair/FeldmanD14,
  author       = {Zohar Feldman and
                  Carmel Domshlak},
  title        = {Simple Regret Optimization in Online Planning for Markov Decision
                  Processes},
  journal      = {J. Artif. Intell. Res.},
  volume       = {51},
  pages        = {165--205},
  year         = {2014},
  url          = {https://doi.org/10.1613/jair.4432},
  doi          = {10.1613/JAIR.4432},
  bibsource    = {dblp computer science bibliography, https://dblp.org}
}

@inproceedings{DBLP:conf/uai/CoquelinM07,
  author       = {Pierre{-}Arnaud Coquelin and
                  R{\'{e}}mi Munos},
  editor       = {Ronald Parr and
                  Linda C. van der Gaag},
  title        = {Bandit Algorithms for Tree Search},
  booktitle    = {{UAI} 2007, Proceedings of the Twenty-Third Conference on Uncertainty
                  in Artificial Intelligence, Vancouver, BC, Canada, July 19-22, 2007},
  pages        = {67--74},
  publisher    = {{AUAI} Press},
  year         = {2007},
  url          = {https://dl.acm.org/doi/10.5555/3020488.3020497},
  doi          = {10.5555/3020488.3020497},
  bibsource    = {dblp computer science bibliography, https://dblp.org}
}

@inproceedings{DBLP:conf/cg/Coulom06,
  author       = {R{\'{e}}mi Coulom},
  editor       = {H. Jaap van den Herik and
                  Paolo Ciancarini and
                  H. H. L. M. Donkers},
  title        = {Efficient Selectivity and Backup Operators in Monte-Carlo Tree Search},
  booktitle    = {Computers and Games, 5th International Conference, {CG} 2006, Turin,
                  Italy, May 29-31, 2006. Revised Papers},
  series       = {Lecture Notes in Computer Science},
  volume       = {4630},
  pages        = {72--83},
  publisher    = {Springer},
  year         = {2006},
  url          = {https://doi.org/10.1007/978-3-540-75538-8\_7},
  doi          = {10.1007/978-3-540-75538-8\_7},
  bibsource    = {dblp computer science bibliography, https://dblp.org}
}

@article{DBLP:journals/nature/SilverHMGSDSAPL16,
  author       = {David Silver and
                  Aja Huang and
                  Chris J. Maddison and
                  Arthur Guez and
                  Laurent Sifre and
                  George van den Driessche and
                  Julian Schrittwieser and
                  Ioannis Antonoglou and
                  Vedavyas Panneershelvam and
                  Marc Lanctot and
                  Sander Dieleman and
                  Dominik Grewe and
                  John Nham and
                  Nal Kalchbrenner and
                  Ilya Sutskever and
                  Timothy P. Lillicrap and
                  Madeleine Leach and
                  Koray Kavukcuoglu and
                  Thore Graepel and
                  Demis Hassabis},
  title        = {Mastering the game of Go with deep neural networks and tree search},
  journal      = {Nat.},
  volume       = {529},
  number       = {7587},
  pages        = {484--489},
  year         = {2016},
  url          = {https://doi.org/10.1038/nature16961},
  doi          = {10.1038/NATURE16961},
  bibsource    = {dblp computer science bibliography, https://dblp.org}
}

@inproceedings{DBLP:conf/icml/WanFWM00024,
  author       = {Ziyu Wan and
                  Xidong Feng and
                  Muning Wen and
                  Stephen Marcus McAleer and
                  Ying Wen and
                  Weinan Zhang and
                  Jun Wang},
  title        = {AlphaZero-Like Tree-Search can Guide Large Language Model Decoding
                  and Training},
  booktitle    = {Forty-first International Conference on Machine Learning, {ICML} 2024,
                  Vienna, Austria, July 21-27, 2024},
  publisher    = {OpenReview.net},
  year         = {2024},
  url          = {https://openreview.net/forum?id=C4OpREezgj},
  bibsource    = {dblp computer science bibliography, https://dblp.org}
}

@article{DBLP:journals/ior/ChangFHM05,
  author       = {Hyeong Soo Chang and
                  Michael C. Fu and
                  Jiaqiao Hu and
                  Steven I. Marcus},
  title        = {An Adaptive Sampling Algorithm for Solving Markov Decision Processes},
  journal      = {Oper. Res.},
  volume       = {53},
  number       = {1},
  pages        = {126--139},
  year         = {2005},
  url          = {https://doi.org/10.1287/opre.1040.0145},
  doi          = {10.1287/OPRE.1040.0145},
  bibsource    = {dblp computer science bibliography, https://dblp.org}
}

@inproceedings{ijcai2022p1,
  title     = {Anytime Capacity Expansion in Medical Residency Match by Monte Carlo Tree Search},
  author    = {Abe, Kenshi and Komiyama, Junpei and Iwasaki, Atsushi},
  booktitle = {Proceedings of the Thirty-First International Joint Conference on
               Artificial Intelligence, {IJCAI-22}},
  publisher = {International Joint Conferences on Artificial Intelligence Organization},
  editor    = {Lud De Raedt},
  pages     = {3--9},
  year      = {2022},
  month     = {7},
  note      = {Main Track},
  doi       = {10.24963/ijcai.2022/1},
  url       = {https://doi.org/10.24963/ijcai.2022/1},
}

@inproceedings{DBLP:conf/nips/YaoYZS00N23,
  author       = {Shunyu Yao and
                  Dian Yu and
                  Jeffrey Zhao and
                  Izhak Shafran and
                  Tom Griffiths and
                  Yuan Cao and
                  Karthik Narasimhan},
  editor       = {Alice Oh and
                  Tristan Naumann and
                  Amir Globerson and
                  Kate Saenko and
                  Moritz Hardt and
                  Sergey Levine},
  title        = {Tree of Thoughts: Deliberate Problem Solving with Large Language Models},
  booktitle    = {Advances in Neural Information Processing Systems 36: Annual Conference
                  on Neural Information Processing Systems 2023, NeurIPS 2023, New Orleans,
                  LA, USA, December 10 - 16, 2023},
  year         = {2023},
  url          = {http://papers.nips.cc/paper\_files/paper/2023/hash/271db9922b8d1f4dd7aaef84ed5ac703-Abstract-Conference.html},
  bibsource    = {dblp computer science bibliography, https://dblp.org}
}

@inproceedings{DBLP:conf/alt/HuangAS017,
  author       = {Ruitong Huang and
                  Mohammad M. Ajallooeian and
                  Csaba Szepesv{\'{a}}ri and
                  Martin M{\"{u}}ller},
  editor       = {Steve Hanneke and
                  Lev Reyzin},
  title        = {Structured Best Arm Identification with Fixed Confidence},
  booktitle    = {International Conference on Algorithmic Learning Theory, {ALT} 2017,
                  15-17 October 2017, Kyoto University, Kyoto, Japan},
  series       = {Proceedings of Machine Learning Research},
  volume       = {76},
  pages        = {593--616},
  publisher    = {{PMLR}},
  year         = {2017},
  url          = {http://proceedings.mlr.press/v76/huang17a.html},
  bibsource    = {dblp computer science bibliography, https://dblp.org}
}

@ARTICLE{graphnn_mcts,
  author={Xing, Zhihao and Tu, Shikui},
  journal={IEEE Access}, 
  title={A Graph Neural Network Assisted Monte Carlo Tree Search Approach to Traveling Salesman Problem}, 
  year={2020},
  volume={8},
  number={},
  pages={108418-108428},
  doi={10.1109/ACCESS.2020.3000236}}

@inproceedings{DBLP:conf/aaai/KhalilVD22,
  author       = {Elias B. Khalil and
                  Pashootan Vaezipoor and
                  Bistra Dilkina},
  title        = {Finding Backdoors to Integer Programs: {A} Monte Carlo Tree Search
                  Framework},
  booktitle    = {Thirty-Sixth {AAAI} Conference on Artificial Intelligence, {AAAI}
                  2022, Thirty-Fourth Conference on Innovative Applications of Artificial
                  Intelligence, {IAAI} 2022, The Twelveth Symposium on Educational Advances
                  in Artificial Intelligence, {EAAI} 2022 Virtual Event, February 22
                  - March 1, 2022},
  pages        = {3786--3795},
  publisher    = {{AAAI} Press},
  year         = {2022},
  url          = {https://doi.org/10.1609/aaai.v36i4.20293},
  doi          = {10.1609/AAAI.V36I4.20293},
  bibsource    = {dblp computer science bibliography, https://dblp.org}
}

@inproceedings{DBLP:conf/alt/BubeckMS09,
  author       = {S{\'{e}}bastien Bubeck and
                  R{\'{e}}mi Munos and
                  Gilles Stoltz},
  editor       = {Ricard Gavald{\`{a}} and
                  G{\'{a}}bor Lugosi and
                  Thomas Zeugmann and
                  Sandra Zilles},
  title        = {Pure Exploration in Multi-armed Bandits Problems},
  booktitle    = {Algorithmic Learning Theory, 20th International Conference, {ALT}
                  2009, Porto, Portugal, October 3-5, 2009. Proceedings},
  series       = {Lecture Notes in Computer Science},
  volume       = {5809},
  pages        = {23--37},
  publisher    = {Springer},
  year         = {2009},
  url          = {https://doi.org/10.1007/978-3-642-04414-4\_7},
  doi          = {10.1007/978-3-642-04414-4\_7},
  bibsource    = {dblp computer science bibliography, https://dblp.org}
}

@inproceedings{DBLP:conf/iclr/YangAY21,
  author       = {Xiufeng Yang and
                  Tanuj Kr Aasawat and
                  Kazuki Yoshizoe},
  title        = {Practical Massively Parallel Monte-Carlo Tree Search Applied to Molecular
                  Design},
  booktitle    = {9th International Conference on Learning Representations, {ICLR} 2021,
                  Virtual Event, Austria, May 3-7, 2021},
  publisher    = {OpenReview.net},
  year         = {2021},
  url          = {https://openreview.net/forum?id=6k7VdojAIK},
  bibsource    = {dblp computer science bibliography, https://dblp.org}
}

@incollection{DBLP:reference/ecgg/Yoshizoe019,
  author       = {Kazuki Yoshizoe and
                  Martin M{\"{u}}ller},
  editor       = {Newton Lee},
  title        = {Computer Go},
  booktitle    = {Encyclopedia of Computer Graphics and Games},
  publisher    = {Springer},
  year         = {2019},
  url          = {https://doi.org/10.1007/978-3-319-08234-9\_20-1},
  doi          = {10.1007/978-3-319-08234-9\_20-1},
  bibsource    = {dblp computer science bibliography, https://dblp.org}
}

@inproceedings{DBLP:conf/wsc/GlynnJ04,
  author       = {Peter W. Glynn and
                  Sandeep Juneja},
  title        = {A Large Deviations Perspective on Ordinal Optimization},
  booktitle    = {Proceedings of the 36th conference on Winter simulation, Washington,
                  DC, USA, December 5-8, 2004},
  pages        = {577--585},
  publisher    = {{IEEE} Computer Society},
  year         = {2004},
  url          = {http://www.informs-sim.org/wsc04papers/070.pdf},
  bibsource    = {dblp computer science bibliography, https://dblp.org}
}

@book{bechhofer1968sequential,
  title={Sequential Identification and Ranking Procedures: With Special Reference to Koopman-Darmois Populations},
  author={Bechhofer, R.E. and Kiefer, J. and Sobel, M.},
  series={Statistical Research Monographs},
  url={https://books.google.ae/books?id=bE04AAAAIAAJ},
  year={1968},
  publisher={University of Chicago Press}
}

@inproceedings{DBLP:conf/aistats/TabataKNK23,
  author       = {Koji Tabata and
                  Junpei Komiyama and
                  Atsuyoshi Nakamura and
                  Tamiki Komatsuzaki},
  editor       = {Francisco J. R. Ruiz and
                  Jennifer G. Dy and
                  Jan{-}Willem van de Meent},
  title        = {Posterior Tracking Algorithm for Classification Bandits},
  booktitle    = {International Conference on Artificial Intelligence and Statistics,
                  25-27 April 2023, Palau de Congressos, Valencia, Spain},
  series       = {Proceedings of Machine Learning Research},
  volume       = {206},
  pages        = {10994--11022},
  publisher    = {{PMLR}},
  year         = {2023},
  url          = {https://proceedings.mlr.press/v206/tabata23a.html},
  bibsource    = {dblp computer science bibliography, https://dblp.org}
}

@inproceedings{DBLP:conf/pakdd/TsaiTL25,
  author       = {Yun{-}Da Tsai and
                  Tzu{-}Hsien Tsai and
                  Shou{-}De Lin},
  editor       = {Xintao Wu and
                  Myra Spiliopoulou and
                  Can Wang and
                  Vipin Kumar and
                  Longbing Cao and
                  Yanqiu Wu and
                  Yu Yao and
                  Zhangkai Wu},
  title        = {Differentiable Good Arm Identification},
  booktitle    = {Advances in Knowledge Discovery and Data Mining - 29th Pacific-Asia
                  Conference on Knowledge Discovery and Data Mining, {PAKDD} 2025, Sydney,
                  NSW, Australia, June 10-13, 2025, Proceedings, Part {I}},
  series       = {Lecture Notes in Computer Science},
  volume       = {15870},
  pages        = {253--264},
  publisher    = {Springer},
  year         = {2025},
  url          = {https://doi.org/10.1007/978-981-96-8170-9\_20},
  doi          = {10.1007/978-981-96-8170-9\_20},
  bibsource    = {dblp computer science bibliography, https://dblp.org}
}

@ARTICLE{teraoka2014,
author={Kazuki Teraoka and Kohei Hatano and Eiji Takimoto},
journal={IEICE TRANSACTIONS on Information},
title={Efficient Sampling Method for Monte Carlo Tree Search Problem},
year={2014},
volume={E97-D},
number={3},
pages={392-398},
doi={10.1587/transinf.E97.D.392},
ISSN={1745-1361},
month={March},}

@InProceedings{pmlr-v129-leurent20a,
  title = 	 {Monte-Carlo Graph Search: the Value of Merging Similar States},
  author =       {Leurent, Edouard and Maillard, Odalric-Ambrym},
  booktitle = 	 {Proceedings of The 12th Asian Conference on Machine Learning},
  pages = 	 {577--592},
  year = 	 {2020},
  editor = 	 {Pan, Sinno Jialin and Sugiyama, Masashi},
  volume = 	 {129},
  series = 	 {Proceedings of Machine Learning Research},
  month = 	 {18--20 Nov},
  publisher =    {PMLR},
  url = 	 {https://proceedings.mlr.press/v129/leurent20a.html}
}

@inproceedings{tolpin2012,
author = {Tolpin, David and Shimony, Solomon Eyal},
title = {MCTS based on simple regret},
year = {2012},
publisher = {AAAI Press},
booktitle = {Proceedings of the Twenty-Sixth AAAI Conference on Artificial Intelligence},
pages = {570–576},
numpages = {7},
location = {Toronto, Ontario, Canada},
series = {AAAI'12}
}

@article{Graeme2018,
author = {Best, Graeme and Cliff, Oliver and Patten, Timothy and Mettu, Ramgopal and Fitch, Robert},
year = {2018},
month = {03},
pages = {027836491875592},
title = {Dec-MCTS: Decentralized planning for multi-robot active perception},
volume = {38},
journal = {The International Journal of Robotics Research},
doi = {10.1177/0278364918755924}
}

@inproceedings{NEURIPS2020_0d85eb24,
 author = {Jonsson, Anders and Kaufmann, Emilie and Menard, Pierre and Darwiche Domingues, Omar and Leurent, Edouard and Valko, Michal},
 booktitle = {Advances in Neural Information Processing Systems},
 editor = {H. Larochelle and M. Ranzato and R. Hadsell and M.F. Balcan and H. Lin},
 pages = {1253--1263},
 publisher = {Curran Associates, Inc.},
 title = {Planning in Markov Decision Processes with Gap-Dependent Sample Complexity},
 url = {https://proceedings.neurips.cc/paper_files/paper/2020/file/0d85eb24e2add96ff1a7021f83c1abc9-Paper.pdf},
 volume = {33},
 year = {2020}
}

@ARTICLE{pepels2014,
  author={Pepels, Tom and Winands, Mark H. M. and Lanctot, Marc},
  journal={IEEE Transactions on Computational Intelligence and AI in Games}, 
  title={Real-Time Monte Carlo Tree Search in Ms Pac-Man}, 
  year={2014},
  volume={6},
  number={3},
  pages={245-257},
  doi={10.1109/TCIAIG.2013.2291577}}

@article{knuth1975alphabeta,
  author = {Knuth, Donald E. and Moore, Ronald W.},
  journal = {Artificial Intelligence},
  pages = {293--326},
  title = {An Analysis of Alpha-Beta Pruning},
  volume = 6,
  year = 1975
}
\bibliographystyle{plainnat}

\newpage
\appendix
\onecolumn

\section{Symmetric Forms of Recursive Formulas}\label{append:symmetric_omitted_rf}

Recursive Formula~\ref{rf:optimal_weight_lose} is a symmetric form of Recursive Formula~\ref{rf:optimal_weight_win}, which solves the optimization problem of (\ref{optimalweight_sub}) for a subtree rooted at $s$ in the case with $a_{s_0}(\bm{\mu})=$`lose'.
\addtocounter{algorithm}{1}
\captionsetup[algorithm]{name=Recursive Formula}
\begin{algorithm}[tb]
\caption{Calculation of $\bm{w}^s(\bm{\mu})$ (2/2)}\label{rf:optimal_weight_lose}
Case with $a_{s_0}(\bm{\mu})=$`lose'
\begin{align*}
   d_s(\bm{\mu})
   =&\begin{cases}
   d(\mu_s,\theta) & (s\in \mathcal{L}(T), \mu_s< \theta)\\
        \displaystyle\max_{c\in \mathcal{C}(s)}d_c(\bm{\mu}) & (L(s)=\text{`MIN'})\\
        \displaystyle\frac{1}{\displaystyle\sum_{c\in \mathcal{C}(s)}1/d_c(\bm{\mu})} & \left( \begin{aligned}
            &L(s)=\text{`MAX'},\\& d_c(\bm{\mu})>0 \text{ for all } c\in \mathcal{C}(s)
        \end{aligned} \right)\\
        0 &
     (\text{otherwise})
   \end{cases}\\
    w^s_{\ell}(\bm{\mu})
    =&\begin{cases}
    1 & (s\in \mathcal{L}(T))\\
    w^{c^*(s)}_\ell(\bm{\mu}) & \left( \begin{aligned}
            &L(s)=\text{`MIN'},\\& d_s(\bm{\mu}) > 0,\ 
            \ell \in \mathcal{D}(c^*(s))
        \end{aligned}\right)\\
        0 & \left( \begin{aligned}
            &L(s)=\text{`MIN'},\\& d_s(\bm{\mu}) > 0,\
            \ell \in \mathcal{D}(c),\\& c \in \mathcal{C}(s) \setminus \{ c^*(s) \}
        \end{aligned}\right)\\
        \displaystyle\frac{w^c_\ell(\bm{\mu}) / d_c(\bm{\mu})}{\displaystyle\sum_{c' \in \mathcal{C}(s)} 1 / d_{c'}(\bm{\mu})}  & \left( \begin{aligned}
            &L(s)=\text{`MAX'},\\& d_s(\bm{\mu}) > 0,\ 
            \ell \in \mathcal{D}(c),\\& c \in \mathcal{C}(s)
        \end{aligned}\right)\\
        \begin{aligned}
        \text{any }& w_{\ell}^s \text{ with } \\
        &\bm{w}^s\in \Delta_{s}
        \end{aligned}& (\text{otherwise}) 
    \end{cases}
\end{align*}
\end{algorithm}

Recursive Formula~\ref{rf:GLR_lose} is a symmetric form of Recursive Formula~\ref{rf:GLR_win}, which calculates the GLR statistic $Z_s(t)$ for any node $s$ in case with $a_{s_0}(\hat{\bm{\mu}}(t))=$`lose'.
\begin{algorithm}[tb]
    \caption{GLR calculation (2/2)}\label{rf:GLR_lose}
    Case with $a_{s_0}(\hat{\bm{\mu}}(t))=$`lose'
    \begin{align*}
       Z_s(t)
       =\begin{cases}
         N_s(t)d(\hat{\mu}_s(t),\theta) & (s\in \mathcal{L}(T), \hat{\mu}_s(t)< \theta)\\
            \displaystyle\sum_{c\in \mathcal{C}(s)}Z_c(t) & (L(s)=\text{`MIN'})\\
           \displaystyle\min_{c\in \mathcal{C}(s)}Z_c(t)  & (L(s)=\text{`MAX'})\\
            0 & (\text{otherwise})
       \end{cases}
    \end{align*}
\end{algorithm}

\section{Proof of Theorem~\ref{thm:optimal_weight} (Calculation of Optimal Weights)}\label{append:optimal_weight}

\begin{proof}
  We first prove the case with $a_{s_0}(\bm{\mu})=$`win'.

  Consider the case that node $s$ is a MAX node and $A_s(\bm{\mu})= \emptyset$, or the case that node $s$ is a MIN node and $A_s(\bm{\mu})\neq\mathcal{C}(s)$.
In these cases, $\bm{\mu}_{\mathcal{D}(s)}\in \mathrm{Alt}_{s}(\bm{\mu})$. 
Thus,
\begin{align*}
  \inf_{\bm{\lambda}\in \mathrm{Alt}_{s}(\bm{\mu})}\sum_{\ell\in\mathcal{D}(s)}w_{\ell}d(\mu_{\ell},\lambda_{\ell})=\sum_{\ell\in\mathcal{D}(s)}w_{\ell}\times 0 = 0
\end{align*} holds, which means any $\bm{w}\in \Delta_{\mathcal{D}(s)}$ can be $\bm{w}^s$ defined by (\ref{optimalweight_sub}). That solution also coincides with $(w^s_{\ell})_{\ell\in \mathcal{D}(s)}$ calculated by Recursive Formula~\ref{rf:optimal_weight_win}.
In this case,
\begin{align*}
  d_s=0=\max_{\bm{w}\in\Delta_{\mathcal{D}(s)}}\displaystyle\inf_{\bm{\lambda}\in \mathrm{Alt}_{s}(\bm{\mu})}\sum_{\ell\in\mathcal{D}(s)}w_{\ell}d(\mu_{\ell},\lambda_{\ell})
  \end{align*}
holds.
  
In the remaining cases with $a_{s_0}(\bm{\mu})=$`win', we prove by mathematical induction on the height of the tree node $s$.

  Consider a height-$0$ tree node $s$. 
  In this case, node $s$ is a leaf node itself, thus $w_s=1$. In this case, 
 \begin{align*}
    d_s(\bm{\mu})=\max_{\bm{w}\in\Delta_{\mathcal{D}(s)}}\inf_{\bm{\lambda}\in \mathrm{Alt}_{s}(\bm{\mu})}\sum_{\ell\in\mathcal{D}(s)}w_{\ell}d(\mu_{\ell},\lambda_{\ell})=
\begin{cases}
d(\mu_{s},\theta) & (\mu_i\ge 0)\\
d(\mu_{s},\mu_{s})=0& (\mu_i< 0),
\end{cases}
\end{align*}
which coincides with weights $w^s_{s}(\bm{\mu})$ and $d_s(\bm{\mu})$ calculated by Recursive Formula~\ref{rf:optimal_weight_win}.

Assume that $\bm{w}^s$ is calculated by Recursive Formula~\ref{rf:optimal_weight_win} and
\begin{align*}
  d_s=\max_{\bm{w}\in\Delta_{\mathcal{D}(s)}}\displaystyle\inf_{\bm{\lambda}\in \mathrm{Alt}_{s}(\bm{\mu})}\sum_{\ell\in\mathcal{D}(s)}w_{\ell}d(\mu_{\ell},\lambda_{\ell})
\end{align*}
  holds for the tree nodes $s$ with the height at most $k$.
  Consider a height-$(k+1)$ tree node $s$.
  Assume that node $s$ is a MAX node and $A_s(\bm{\mu})\neq\emptyset$.
  Then,
  \begin{align*}
    \inf_{\bm{\lambda}\in \mathrm{Alt}_{s}(\bm{\mu})}\sum_{\ell\in\mathcal{D}(s)}w_{\ell}d(\mu_{\ell},\lambda_{\ell})=\sum_{c\in \mathcal{C}(s),V_c(\bm{\mu})\ge \theta}W_c\inf_{\bm{\lambda}\in \mathrm{Alt}_{c}(\bm{\mu})}\sum_{\ell\in \mathcal{D}(c)}w'_{\ell}d(\mu_{\ell},\lambda_{\ell})
\end{align*}
  for any $\bm{w}\in \Delta_{\mathcal{D}(s)}$, where $W_c=\sum_{\ell\in \mathcal{D}(c)}w_{\ell}$ and $w'_{\ell}=w_{\ell}/W_c$. Let optimal $(W_c)_{c\in\mathcal{C}(s)}$ denote $(W^*_c)_{c\in\mathcal{C}(s)}$. Thus,
  \begin{align*}
    \max_{\bm{w}\in \Delta_{\mathcal{D}(s)}}\inf_{\bm{\lambda}\in \mathrm{Alt}_{s}(\bm{\mu})}\sum_{\ell\in\mathcal{D}(s)}w_{\ell}d(\mu_{\ell},\lambda_{\ell})=&\max_{\bm{W}\in \Delta_{\mathcal{C}(s)}}\sum_{c\in \mathcal{C}(s),V_c(\bm{\mu})\ge \theta}W_c\max_{\bm{w}'\in \Delta_{\mathcal{D}(c)}}\inf_{\bm{\lambda}\in \mathrm{Alt}_{c}(\bm{\mu})}\sum_{\ell\in \mathcal{D}(c)}w'_{\ell}d(\mu_{\ell},\lambda_{\ell})\\
    =&\max_{\bm{W}\in \Delta_{\mathcal{C}(s)}}\sum_{c\in \mathcal{C}(s),V_c(\bm{\mu})\ge \theta}W_cd_c\\
     =&\sum_{c\in \mathcal{C}(s),V_c(\bm{\mu})\ge \theta}W^*_cd_c
\end{align*}
where
\begin{align*}
  W^*_{c}=\begin{cases} 1 & \left(c=\displaystyle\argmax_{c'\in A_s(\bm{\mu})}d_{c'}\right)\\
  0 & (\text{otherwise}).
  \end{cases}
\end{align*}
Therefore,
\begin{align*}
  w^s_{\ell}=\begin{cases}
  w^c_{\ell} & (c=\argmax_{c'\in A_s(\bm{\mu})}d_{c'},\ell\in \mathcal{D}(c))\\
  0 &  (\text{otherwise}),
 \end{cases}
\end{align*}
which coincides with weights $w^s_{\ell}$ ($\ell\in \mathcal{D}(s)$) calculated by Recursive Formula~\ref{rf:optimal_weight_win}. In this case,
\begin{align*} d_s=\max_{c\in \mathcal{C}(s)}d_c=\sum_{c\in \mathcal{C}(s),V_c\ge \theta}W^*_cd_c=\max_{\bm{w}\in\Delta_{\mathcal{D}(s)}}\displaystyle\inf_{\bm{\lambda}\in \mathrm{Alt}_{s}(\bm{\mu})}\sum_{\ell\in\mathcal{D}(s)}w_{\ell}d(\mu_{\ell},\lambda_{\ell})
  \end{align*}
holds.

Assume that node $s$ is a MIN node and $A_s(\bm{\mu})=\mathcal{C}(s)$.
Then, 
\begin{align*}
  \inf_{\bm{\lambda}\in \mathrm{Alt}_{s}(\bm{\mu})}\sum_{\ell\in\mathcal{D}(s)}w_{\ell}d(\mu_{\ell},\lambda_{\ell})=\min_{\ell\in \mathcal{C}(s)}W_c\inf_{\bm{\lambda}\in \mathrm{Alt}_{c}(\bm{\mu})}\sum_{\ell\in \mathcal{D}(c)}w'_{\ell}d(\mu_{\ell},\lambda_{\ell})
\end{align*}
for any $\bm{w}\in \Delta_{\mathcal{D}(s)}$, where $W_c=\sum_{\ell\in \mathcal{D}(c)}w_{\ell}$ and $w'_{\ell}=w_{\ell}/W_c$. Thus,
\begin{align*}
  \max_{\bm{w}\in\Delta_{\mathcal{D}(s)}}\inf_{\bm{\lambda}\in \mathrm{Alt}_{s}(\bm{\mu})}\sum_{\ell\in\mathcal{D}(s)}w_{\ell}d(\mu_{\ell},\lambda_{\ell})=&\max_{\bm{W}\in \Delta_{\mathcal{C}(s)}}\min_{c\in \mathcal{C}(s)}W_c\max_{\bm{w}'\in \Delta_{\mathcal{D}(c)}}\inf_{\bm{\lambda}\in \mathrm{Alt}_{c}(\bm{\mu})}\sum_{\ell\in \mathcal{D}(c)}w'_{\ell}d(\mu_{\ell},\lambda_{\ell})\\
  =&\max_{\bm{W}\in \Delta_{\mathcal{C}(s)}}\min_{c\in \mathcal{C}(s)}W_cd_c
\end{align*}
holds. Let optimal $(W_c)_{c\in\mathcal{C}(s)}$ denote $(W^*_c)_{c\in\mathcal{C}(s)}$. Then,
\begin{align*}
  W^*_c=\frac{1/d_c}{\sum_{c'\in\mathcal{C}(s)}1/d_{c'}}
\end{align*}
holds, which means
\begin{align*}
w^s_{\ell}=\frac{1/d_c}{\sum_{c'\in\mathcal{C}(s)}1/d_{c'}}w^c_{\ell}
\end{align*}
for $c \in \mathcal{C}(s)$ and $\ell\in \mathcal{D}(c)$.
These weights $w^s_{\ell}$ coincide with the weights $w^s_{\ell}$ ($\ell\in \mathcal{D}(s)$) calculated by Recursive Formula~\ref{rf:optimal_weight_win}.
In this case,
\begin{align*}
  d_s=\frac{1}{\sum_{c'\in \mathcal{C}(s)}1/d_{c'}}=W^*_cd_c=\max_{\bm{w}\in\Delta_{\mathcal{D}(s)}}\displaystyle\inf_{\bm{\lambda}\in \mathrm{Alt}_{s}(\bm{\mu})}\sum_{\ell\in\mathcal{D}(s)}w_{\ell}d(\mu_{\ell},\lambda_{\ell})
  \end{align*}
holds for any $c\in \mathcal{C}(s)$.

Next, we  prove the case with $a_{s_0}(\bm{\mu})=$`lose'.

  Consider the case that node $s$ is a MIN node and $A_s(\bm{\mu})= \mathcal{C}(s)$, or the case that
node $s$ is a MAX node and $A_s(\bm{\mu})\neq\emptyset$.
In these cases, $\bm{\mu}_{\mathcal{D}(s)}\in \mathrm{Alt}_{s}(\bm{\mu})$. Thus,
\begin{align*}
  \inf_{\bm{\lambda}\in \mathrm{Alt}_{s}(\bm{\mu})}\sum_{\ell\in\mathcal{D}(s)}w_{\ell}d(\mu_{\ell},\lambda_{\ell})=\sum_{\ell\in\mathcal{D}(s)}w_{\ell}\times 0 = 0
\end{align*} holds, which means any $\bm{w}\in \Delta_{\mathcal{D}(s)}$ can be $\bm{w}^s$ defined by (\ref{optimalweight_sub}). That solution also coincides with $(w^s_{\ell})_{\ell\in \mathcal{D}(s)}$ calculated by Recursive Formula~\ref{rf:optimal_weight_lose}.
In this case,
\begin{align*}
  d_s=0=\max_{\bm{w}\in\Delta_{\mathcal{D}(s)}}\displaystyle\inf_{\bm{\lambda}\in \mathrm{Alt}_{s}(\bm{\mu})}\sum_{\ell\in\mathcal{D}(s)}w_{\ell}d(\mu_{\ell},\lambda_{\ell})
  \end{align*}
holds.
  
In the remaining cases with $a_{s_0}(\bm{\mu})=$`lose', we prove by mathematical induction on the height of the tree node $s$.

 Consider a height-$0$ tree node $s$. 
  In this case, node $s$ is a leaf node itself, thus $w_s=1$. In this case, 
 \begin{align*}
    d_s(\bm{\mu})=\max_{\bm{w}\in\Delta_{\mathcal{D}(s)}}\inf_{\bm{\lambda}\in \mathrm{Alt}_{s}(\bm{\mu})}\sum_{\ell\in\mathcal{D}(s)}w_{\ell}d(\mu_{\ell},\lambda_{\ell})=
\begin{cases}
d(\mu_{s},\mu_s)=0 & (\mu_i\ge 0)\\
d(\mu_{s},\theta)& (\mu_i< 0),
\end{cases}
\end{align*}
which coincides with weights $w^s_{s}(\bm{\mu})$ and $d_s(\bm{\mu})$ calculated by Recursive Formula~\ref{rf:optimal_weight_lose}.

Assume that $\bm{w}^s$ is calculated by Recursive Formula~\ref{rf:optimal_weight_lose} and
\begin{align*}
  d_s=\max_{\bm{w}\in\Delta_{\mathcal{D}(s)}}\displaystyle\inf_{\bm{\lambda}\in \mathrm{Alt}_{s}(\bm{\mu})}\sum_{\ell\in\mathcal{D}(s)}w_{\ell}d(\mu_{\ell},\lambda_{\ell})
\end{align*}
  holds for the tree nodes $s$ with the height at most $k$.
  Consider a height-$(k+1)$ tree node $s$.
  Assume that node $s$ is a MIN node and $A_s(\bm{\mu})\neq\mathcal{C}(s)$.
  Then,
  \begin{align*}
    \inf_{\bm{\lambda}\in \mathrm{Alt}_{s}(\bm{\mu})}\sum_{\ell\in\mathcal{D}(s)}w_{\ell}d(\mu_{\ell},\lambda_{\ell})=\sum_{c\in \mathcal{C}(s),V_c(\bm{\mu})< \theta}W_c\inf_{\bm{\lambda}\in \mathrm{Alt}_{c}(\bm{\mu})}\sum_{\ell\in \mathcal{D}(c)}w'_{\ell}d(\mu_{\ell},\lambda_{\ell})
\end{align*}
  for any $\bm{w}\in \Delta_{\mathcal{D}(s)}$, where $W_c=\sum_{\ell\in \mathcal{D}(c)}w_{\ell}$ and $w'_{\ell}=w_{\ell}/W_c$. Let optimal $(W_c)_{c\in\mathcal{C}(s)}$ denote $(W^*_c)_{c\in\mathcal{C}(s)}$. Thus,
  \begin{align*}
    \max_{\bm{w}\in \Delta_{\mathcal{D}(s)}}\inf_{\bm{\lambda}\in \mathrm{Alt}_{s}(\bm{\mu})}\sum_{\ell\in\mathcal{D}(s)}w_{\ell}d(\mu_{\ell},\lambda_{\ell})=&\max_{\bm{W}\in \Delta_{\mathcal{C}(s)}}\sum_{c\in \mathcal{C}(s),V_c(\bm{\mu})< \theta}W_c\max_{\bm{w}'\in \Delta_{\mathcal{D}(c)}}\inf_{\bm{\lambda}\in \mathrm{Alt}_{c}(\bm{\mu})}\sum_{\ell\in \mathcal{D}(c)}w'_{\ell}d(\mu_{\ell},\lambda_{\ell})\\
    =&\max_{\bm{W}\in \Delta_{\mathcal{C}(s)}}\sum_{c\in \mathcal{C}(s),V_c(\bm{\mu})< \theta}W_cd_c\\
     =&\sum_{c\in \mathcal{C}(s),V_c(\bm{\mu})< \theta}W^*_cd_c
\end{align*}
where
\begin{align*}
  W^*_{c}=\begin{cases} 1 & \left(c=\displaystyle\argmax_{c'\in \mathcal{C}(s)\setminus A_s(\bm{\mu})}d_{c'}\right)\\
  0 & (\text{otherwise}).
  \end{cases}
\end{align*}
Therefore,
\begin{align*}
  w^s_{\ell}=\begin{cases}
  w^c_{\ell} & (c=\argmax_{c'\in \mathcal{C}(s)\setminus A_s(\bm{\mu})}d_{c'},\ell\in \mathcal{D}(c))\\
  0 &  (\text{otherwise}),
 \end{cases}
\end{align*}
which coincides with weights $w^s_{\ell}$ ($\ell\in \mathcal{D}(s)$) calculated by Recursive Formula~\ref{rf:optimal_weight_lose}. In this case,
\begin{align*} d_s=\max_{c\in \mathcal{C}(s)}d_c=\sum_{c\in \mathcal{C}(s),V_c< \theta}W^*_cd_c=\max_{\bm{w}\in\Delta_{\mathcal{D}(s)}}\displaystyle\inf_{\bm{\lambda}\in \mathrm{Alt}_{s}(\bm{\mu})}\sum_{\ell\in\mathcal{D}(s)}w_{\ell}d(\mu_{\ell},\lambda_{\ell})
  \end{align*}
holds.

Assume that node $s$ is a MAX node and $A_s(\bm{\mu})=\emptyset$.
Then, 
\begin{align*}
  \inf_{\bm{\lambda}\in \mathrm{Alt}_{s}(\bm{\mu})}\sum_{\ell\in\mathcal{D}(s)}w_{\ell}d(\mu_{\ell},\lambda_{\ell})=\min_{\ell\in \mathcal{C}(s)}W_c\inf_{\bm{\lambda}\in \mathrm{Alt}_{c}(\bm{\mu})}\sum_{\ell\in \mathcal{D}(c)}w'_{\ell}d(\mu_{\ell},\lambda_{\ell})
\end{align*}
for any $\bm{w}\in \Delta_{\mathcal{D}(s)}$, where $W_c=\sum_{\ell\in \mathcal{D}(c)}w_{\ell}$ and $w'_{\ell}=w_{\ell}/W_c$. Thus,
\begin{align*}
  \max_{\bm{w}\in\Delta_{\mathcal{D}(s)}}\inf_{\bm{\lambda}\in \mathrm{Alt}_{s}(\bm{\mu})}\sum_{\ell\in\mathcal{D}(s)}w_{\ell}d(\mu_{\ell},\lambda_{\ell})=&\max_{\bm{W}\in \Delta_{\mathcal{C}(s)}}\min_{c\in \mathcal{C}(s)}W_c\max_{\bm{w}'\in \Delta_{\mathcal{D}(c)}}\inf_{\bm{\lambda}\in \mathrm{Alt}_{c}(\bm{\mu})}\sum_{\ell\in \mathcal{D}(c)}w'_{\ell}d(\mu_{\ell},\lambda_{\ell})\\
  =&\max_{\bm{W}\in \Delta_{\mathcal{C}(s)}}\min_{c\in \mathcal{C}(s)}W_cd_c
\end{align*}
holds. Let optimal $(W_c)_{c\in\mathcal{C}(s)}$ denote $(W^*_c)_{c\in\mathcal{C}(s)}$. Then,
\begin{align*}
  W^*_c=\frac{1/d_c}{\sum_{c'\in\mathcal{C}(s)}1/d_{c'}}
\end{align*}
holds, which means
\begin{align*}
w^s_{\ell}=\frac{1/d_c}{\sum_{c'\in\mathcal{C}(s)}1/d_{c'}}w^c_{\ell}
\end{align*}
for $c \in \mathcal{C}(s)$ and $\ell\in \mathcal{D}(c)$.
These weights $w^s_{\ell}$ coincide with the weights $w^s_{\ell}$ ($\ell\in \mathcal{D}(s)$) calculated by Recursive Formula~\ref{rf:optimal_weight_lose}.
In this case,
\begin{align*}
  d_s=\frac{1}{\sum_{c'\in \mathcal{C}(s)}1/d_{c'}}=W^*_cd_c=\max_{\bm{w}\in\Delta_{\mathcal{D}(s)}}\displaystyle\inf_{\bm{\lambda}\in \mathrm{Alt}_{s}(\bm{\mu})}\sum_{\ell\in\mathcal{D}(s)}w_{\ell}d(\mu_{\ell},\lambda_{\ell})
  \end{align*}
holds for any $c\in \mathcal{C}(s)$.
\end{proof}

\section{Proof of Theorem~\ref{thm:GLR} (Calculation of GLR Statistics)}\label{append:GLR}

\begin{proof}
We first prove the case with $a_{s_0}(\bm{\mu})=$`win' by mathematical induction on the height of the tree node $s$.
Consider a height-$0$ tree node $s$. Then, GLR $Z_s(t)$ can be calculated as
\begin{align*}
Z_s(t)=&\inf_{\bm{\lambda} \in \mathrm{Alt}_{s}(\hat{\bm{\mu}}(t))} \sum_{\ell\in \mathcal{D}(s)}N_{\ell}(t)d(\hat{\mu}_{\ell}(t),\lambda_{\ell})\\
=& \begin{cases}
N_s(t)d(\hat{\mu}_s(t),\theta) & (\hat{\mu}_s(t)\ge \theta)\\
N_s(t)d(\hat{\mu}_s(t),\hat{\mu}_s(t))=0 & (\hat{\mu}_s(t)< \theta),
\end{cases}
\end{align*}
which coincides with $Z_s(t)$ calculated by Recursive Formula~\ref{rf:GLR_win}.

Assume that the GLR $Z_s(t)$ defined by Eq.~(\ref{eq:GLR}) can be calculated through Recursive Formula~\ref{rf:GLR_win} for any height-$k$ tree node $s$.
Consider a height-$(k+1)$ tree node $s$.
Then,
\begin{align*}
Z_s(t)=&\inf_{\bm{\lambda} \in \mathrm{Alt}_{s}(\hat{\bm{\mu}}(t))} \sum_{\ell\in \mathcal{D}(s)}N_{\ell}(t)d(\hat{\mu}_{\ell}(t),\lambda_{\ell})\\
=& 
\begin{cases}
\displaystyle\sum_{c\in \mathcal{C}(s)}\inf_{\bm{\lambda}\in \mathrm{Alt}_{c}(\hat{\bm{\mu}}(t))}\sum_{\ell\in \mathcal{D}(c)}N_{\ell}d(\mu_{\ell},\lambda_{\ell})=\sum_{c\in \mathcal{C}(s)}Z_c(t) & (L(s)=\text{`MAX'})\\
\displaystyle\min_{c\in \mathcal{C}(s)}\inf_{\bm{\lambda}\in \mathrm{Alt}_{c}(\hat{\bm{\mu}}(t))}\sum_{\ell\in \mathcal{D}(c)}N_{\ell}d(\mu_{\ell},\lambda_{\ell})=\min_{c\in \mathcal{C}(s)}Z_c(t) & (L(s)=\text{`MIN'})
\end{cases}
\end{align*}
holds, where $Z_c(t)=\inf_{\bm{\lambda}\in \mathrm{Alt}_{c}(\hat{\bm{\mu}}(t))}\sum_{\ell\in \mathcal{D}(c)}N_{\ell}d(\mu_{\ell},\lambda_{\ell})$ holds by the induction assumption. Thus, the GLR $Z_s(t)$ can be calculated by Recursive Formula~\ref{rf:GLR_win}.

We next prove the case with $a_{s_0}(\bm{\mu})=$`lose' by mathematical induction similarly.
Consider a height-$0$ tree node $s$. Then, GLR $Z_s(t)$ can be calculated as
\begin{align*}
Z_s(t)=&\inf_{\bm{\lambda} \in \mathrm{Alt}_{s}(\hat{\bm{\mu}}(t))} \sum_{\ell\in \mathcal{D}(s)}N_{\ell}(t)d(\hat{\mu}_{\ell}(t),\lambda_{\ell})\\
=& \begin{cases}
N_s(t)d(\hat{\mu}_s(t),\hat{\mu}_s(t))=0 & (\hat{\mu}_s(t)\ge \theta)\\
N_s(t)d(\hat{\mu}_s(t),\theta)& (\hat{\mu}_s(t)< \theta),
\end{cases}
\end{align*}
which coincides with $Z_s(t)$ calculated by Recursive Formula~\ref{rf:GLR_lose}.

Assume that the GLR $Z_s(t)$ defined by Eq.~(\ref{eq:GLR}) can be calculated through Recursive Formula~\ref{rf:GLR_lose} for any height-$k$ tree node $s$.
Consider a height-$(k+1)$ tree node $s$.
Then,
\begin{align*}
Z_s(t)=&\inf_{\bm{\lambda} \in \mathrm{Alt}_{s}(\hat{\bm{\mu}}(t))} \sum_{\ell\in \mathcal{D}(s)}N_{\ell}(t)d(\hat{\mu}_{\ell}(t),\lambda_{\ell})\\
=& 
\begin{cases}
\displaystyle\min_{c\in \mathcal{C}(s)}\inf_{\bm{\lambda}\in \mathrm{Alt}_{c}(\hat{\bm{\mu}}(t))}\sum_{\ell\in \mathcal{D}(c)}N_{\ell}d(\mu_{\ell},\lambda_{\ell})=\min_{c\in \mathcal{C}(s)}Z_c(t) & (L(s)=\text{`MAX'})\\
\displaystyle\sum_{c\in \mathcal{C}(s)}\inf_{\bm{\lambda}\in \mathrm{Alt}_{c}(\hat{\bm{\mu}}(t))}\sum_{\ell\in \mathcal{D}(c)}N_{\ell}d(\mu_{\ell},\lambda_{\ell})=\sum_{c\in \mathcal{C}(s)}Z_c(t) & (L(s)=\text{`MIN'})
\end{cases}
\end{align*}
holds, where $Z_c(t)=\inf_{\bm{\lambda}\in \mathrm{Alt}_{c}(\hat{\bm{\mu}}(t))}\sum_{\ell\in \mathcal{D}(c)}N_{\ell}d(\mu_{\ell},\lambda_{\ell})$ holds by the induction assumption. Thus, the GLR $Z_s(t)$ can be calculated by Recursive Formula~\ref{rf:GLR_lose}.
\end{proof}

\section{Proofs for Almost-Sure Upper Bound of Stopping Time}
\subsection{Proof of Theorem~\ref{thm:as_upper_bound} (Almost-Sure Upper Bound of Stopping Time)}\label{append:as_upper_bound}

Lemma~\ref{lem:GLR_converge} and Propositions~\ref{prop:ineq-hold-upper-bound} and \ref{prop:Cexp-order} are used to prove Lemma~\ref{lem:stopping-time-upper-bound}, which is needed to prove Theorem~\ref{thm:as_upper_bound}.
First, we prove Propositions~\ref{prop:non-zero-weight} and \ref{prop:weight_val_form} to prove Lemma~\ref{lem:GLR_converge}.

\begin{proposition}
     For $\ell\in \mathcal{L}(\mathcal{T})$ with $w_{\ell}^{s_0}(\bm{\mu})>0$,
     \begin{align*}
         \mu_{\ell}\begin{cases}
             > \theta & (a_{s_0}(\bm{\mu})=\text{`win'})\\
            < \theta & (a_{s_0}(\bm{\mu})=\text{`lose'})
         \end{cases}
     \end{align*}
     holds.\label{prop:non-zero-weight}
\end{proposition}
\begin{proof}
By Theorem~\ref{thm:optimal_weight}, we prove this proposition for $\bm{w}^s(\bm{\mu})$ calculated using Recursive Formulae~\ref{rf:optimal_weight_win} and \ref{rf:optimal_weight_lose}.

    Let $s_0,s_1,\dots,s_k=\ell$ be the node sequence on the path from node $s_0$ to node $\ell$. Consider the case with $a_{s_0}(\bm{\mu})=\text{`win'}$. Then, $V_{s_0}(\bm{\mu})>\theta$ by Assumption~\ref{assumption:no_exact_theta}.
In this case, Condition $w_{\ell}^{s_0}(\bm{\mu})>0$
implies that, for $i=0,1,\dots,k-1$, $V_{s_i}(\bm{\mu})=V_{s_{i+1}}(\bm{\mu})$ for node $s_i$ with $L(s_i)=\text{`MAX'}$ and $V_{s_i}(\bm{\mu})\le V_{s_{i+1}}(\bm{\mu})$ for node $s_i$ with $L(s_i)=\text{`MIN'}$. Thus,
$\theta<V_{s_0}(\bm{\mu})\le V_{s_1}(\bm{\mu})\le\cdots\le V_{s_k}(\bm{\mu})=V_{\ell}(\bm{\mu})=\mu_{\ell}$ holds.
In the case with $a_{s_0}(\bm{\mu})=\text{`lose'}$, $V_{s_0}(\bm{\mu})<\theta$ holds. Condition $w_{\ell}^{s_0}(\bm{\mu})>0$
implies that, for $i=0,1,\dots,k-1$, $V_{s_i}(\bm{\mu})=V_{s_{i+1}}(\bm{\mu})$ for node $s_i$ with $L(s_i)=\text{`MIN'}$ and $V_{s_i}(\bm{\mu})\ge V_{s_{i+1}}(\bm{\mu})$ for node $s_i$ with $L(s_i)=\text{`MAX'}$. Thus,
$\theta>V_{s_0}(\bm{\mu})\ge V_{s_1}(\bm{\mu})\ge\cdots\ge V_{s_k}(\bm{\mu})=V_{\ell}(\bm{\mu})=\mu_{\ell}$ holds.
\end{proof}

\begin{proposition}
    For $\ell\in \mathcal{L}(\mathcal{T})$ with $w_{\ell}^{s_0}(\bm{\mu})>0$,
    \begin{align}
   w_{\ell}^s(\bm{\mu})=\frac{d_{s}(\bm{\mu})}{d(\mu_{\ell},\theta)} \label{weight_val_form}
    \end{align}
 holds for $s=s_0,s_1,\dots,s_k(=\ell)$, where $s_0,s_1,\dots,s_k$ is the node sequence on the path from node $s_0$ to node $\ell$.\label{prop:weight_val_form}
\end{proposition}
\begin{proof}
By Theorem~\ref{thm:optimal_weight}, Eq.~(\ref{weight_val_form}) can be proved for $\bm{w}^s(\bm{\mu})$ and $d_s(\bm{\mu})$ calculated using Recursive Formula~\ref{rf:optimal_weight_win} and \ref{rf:optimal_weight_lose}.

We prove Eq.~(\ref{weight_val_form}) by mathematical induction. 

Consider the case with $s=s_k$.
By proposition~\ref{prop:non-zero-weight} and Recursive Formula~\ref{rf:optimal_weight_win} and \ref{rf:optimal_weight_lose},
\begin{align*}
d_s(\bm{\mu})=d(\mu_{\ell},\theta).
\end{align*}
Thus, in this case, 
\begin{align*}
  w^{s}_{\ell}(\bm{\mu})=1=\frac{d_s(\bm{\mu})} {d(\mu_{\ell},\theta)}
\end{align*}
holds.

Assume that Eq.~(\ref{weight_val_form}) holds for $s=s_i$ ($1\le i\le k$).
Consider the case with $a_{s_0}(\bm{\mu})=\text{`win'}$.
If $L(s_{i-1})=\text{`MAX'}$, $w_{\ell}^{s_0}(\bm{\mu})>0$ implies 
$w^{s_{i-1}}_{\ell}(\bm{\mu})=w^{s_i}_{\ell}(\bm{\mu})$ and $d_{s_{i-1}}(\bm{\mu})=d_{s_i}(\bm{\mu})$, thus
\begin{align*}
w^{s_{i-1}}_{\ell}(\bm{\mu})=w^{s_i}_{\ell}(\bm{\mu})=
\frac{d_{s_i}(\bm{\mu})}{d(\mu_{\ell},\theta)}=\frac{d_{s_{i-1}}(\bm{\mu})}{d(\mu_{\ell},\theta)}
\end{align*}
holds. If $L(s_{i-1})=\text{`MIN'}$, $w_{\ell}^{s_0}(\bm{\mu})>0$ implies 
\begin{align*}
    w^{s_{i-1}}_{\ell}(\bm{\mu})=\frac{w^{s_i}_{\ell}(\bm{\mu})/d_{s_i}(\bm{\mu})}{\sum_{c\in\mathcal{C}(s_i)}1/d_c(\bm{\mu})} \text{ and } d_{s_{i-1}}(\bm{\mu})=\frac{1}{\sum_{c\in \mathcal{C}(s_i)}1/d_{c}(\bm{\mu})},
\end{align*}
thus
\begin{align*}
w^{s_{i-1}}_{\ell}(\bm{\mu})=\frac{w^{s_i}_{\ell}(\bm{\mu})/d_{s_i}(\bm{\mu})}{\sum_{c\in\mathcal{C}(s_i)}1/d_c(\bm{\mu})}=\frac{w^{s_i}_{\ell}(\bm{\mu})d_{s_{i-1}}(\bm{\mu})}{d_{s_i}(\bm{\mu})}=\frac{d_{s_i}(\bm{\mu})}{d(\mu_{\ell},\theta)}\times\frac{d_{s_{i-1}}(\bm{\mu})}{d_{s_i}(\bm{\mu})}=\frac{d_{s_{i-1}}(\bm{\mu})}{d(\mu_{\ell},\theta)}
\end{align*}
holds. We can prove Eq.~(\ref{weight_val_form}) similarly
for the case with  $a_{s_0}(\bm{\mu})=\text{`lose'}$.
\end{proof}

\begin{lemma} \label{lem:GLR_converge}
    On the event $\mathcal{E}=
  \left\{\frac{N_\ell(t)}{t} \xrightarrow{t \to \infty} w^{s_0}_\ell(\bm{\mu}) \text{ for all } \ell \in \mathcal{L}(\mathcal{T}) \right\}$,
\begin{align}
\frac{Z_s(t)}{t}\xrightarrow{t \to \infty}\begin{cases}
    d_{s_0}(\bm{\mu}) & (w^{s_0}_{\ell}(\bm{\mu})>0 \text{ for some } \ell\in \mathcal{D}(s))\\
    0 & (w^{s_0}_{\ell}(\bm{\mu})=0 \text{ for all } \ell\in \mathcal{D}(s))
\end{cases}\label{GLRconverge}
\end{align}
holds for any node $s\in S$.
\end{lemma}
\begin{proof}
We prove the lemma by mathematical induction.

First, we prove (\ref{GLRconverge}) in the case with $s\in \mathcal{L}(\mathcal{T})$. For a leaf node $s$ with $w^{s_0}_{s}(\bm{\mu})>0$, $w^{s_0}_s(\bm{\mu}) = d_{s_0}(\bm{\mu}) / d(\mu_s, \theta)$ by Proposition~\ref{prop:weight_val_form}. 
By the assumption that $\frac{N_\ell(t)}{t} \xrightarrow{t \to \infty} w^{s_0}_s(\bm{\mu})>0$, the leaf node $s$ is selected infinitely often which leads to $\hat{\mu}_s(t)\xrightarrow{t \to \infty}\mu_s$ by law of large numbers. Thus,
\begin{align*}
\frac{Z_s(t)}{t} = \frac{N_s(t)d(\hat{\mu}_s(t), \theta)}{t} \xrightarrow{t \to \infty} w^{s_0}_s(\bm{\mu})d(\mu_s, \theta)=\frac{d_{s_0}(\bm{\mu})}{d(\mu_s,\theta)}d(\mu_s, \theta)=d_{s_0}(\bm{\mu})
\end{align*}
 converges to $w^{s_0}_s d(\mu_s, \theta) = d_{s_0}(\bm{\mu})$.
As for a leaf node $s$ with $w^{s_0}_{s}(\bm{\mu})=0$,
$N_s(t)/t\xrightarrow{t \to \infty}w^{s_0}_{s}(\bm{\mu})= 0$
by the assumption. Thus, 
\begin{align*}
\frac{Z_s(t)}{t} = \frac{N_s(t)d(\hat{\mu}_s(t), \theta)}{t} \text{ or } \frac{0}{t}\xrightarrow{t \to \infty} 0\times d(\mu_s, \theta) \text{ or } 0=0.
\end{align*}

Next, we prove (\ref{GLRconverge}) for node $s\in S$, assuming that (\ref{GLRconverge}) holds for all the child nodes $c\in \mathcal{C}(s)$.

Consider the case with ($L(s)=\text{`MAX'}$ and $a_{s_0}(\bm{\mu}) = \text{`win'}$) or ($L(s)=\text{`MIN'}$ and $a_{s_0}(\bm{\mu}) = \text{`lose'}$).

For node $s$ with $w^{s_0}_{\ell}(\bm{\mu})>0 \text{ for some } \ell\in \mathcal{D}(s)$
  \begin{align*}
      \frac{Z_s(t)}{t}=\frac{\sum_{c\in\mathcal{C}(s)}Z_c(t)}{t} =\frac{Z_{c^*(s)}(t)}{t}+\sum_{c\in \mathcal{C}(s)\setminus \{c^*(s)\}}\frac{Z_c(t)}{t}\xrightarrow{t \to \infty} d_{s_0}(\bm{\mu})+\sum_{c\in \mathcal{C}(s)\setminus \{c^*(s)\}}0=d_{s_0}(\bm{\mu})
  \end{align*}  
  by inductive assumption because $c\in \mathcal{C}(s)$ satisfies $w^{s_0}_{\ell}(\bm{\mu})>0 \text{ for some } \ell\in \mathcal{D}(c)$ if and only if $c=c^*(s)$.

For node $s$ with $w^{s_0}_{\ell}(\bm{\mu})=0 \text{ for all } \ell\in \mathcal{D}(s)$,
\begin{align*}
      \frac{Z_s(t)}{t}=\sum_{c\in\mathcal{C}(s)}\frac{Z_c(t)}{t} \xrightarrow{t \to \infty} \sum_{c\in \mathcal{C}(s)}0=0
  \end{align*}  
by inductive assumption.

Consider the case with ($L(s)=\text{`MIN'}$ and $a_{s_0}(\bm{\mu}) = \text{`win'}$) or ($L(s)=\text{`MAX'}$ and $a_{s_0}(\bm{\mu}) = \text{`lose'}$). 

For node $s$ with $w^{s_0}_{\ell}(\bm{\mu})>0 \text{ for some } \ell\in \mathcal{D}(s)$,
\begin{align*}
    \frac{Z_s(t)}{t}=\min_{c\in \mathcal{C}(s)}\frac{Z_c(t)}{t} \xrightarrow{t \to \infty} \min_{c\in \mathcal{C}(s)}d_{s_0}(\bm{\mu})=d_{s_0}(\bm{\mu})
\end{align*}
by inductive assumption.

For node $s$ with $w^{s_0}_{\ell}(\bm{\mu})=0 \text{ for all } \ell\in \mathcal{D}(s)$,
\begin{align*}
    \frac{Z_s(t)}{t}=\min_{c\in \mathcal{C}(s)}\frac{Z_c(t)}{t} \xrightarrow{t \to \infty} \min_{c\in \mathcal{C}(s)}0=0
\end{align*}
by inductive assumption.

Therefore, (\ref{GLRconverge}) holds for all $s\in S$.
\end{proof}

\begin{proposition}
For any three constants $c_1,c_2,c_3>0$ with $2c_2c_3\ge c_1$,
\[
x=\frac{1}{c_1}\left[\ln c_2+\ln\ln \frac{c_3}{c_1}+\ln\ln\ln\left(\frac{2c_2c_3}{c_1}\right)^2\right]
\]
satisfies $c_1x\ge \ln (c_2\ln c_3x)$.\label{prop:ineq-hold-upper-bound}
\end{proposition}
\begin{proof}
By substituting for $x$, we obtain
\begin{align*}
    c_1x=&\ln c_2+\ln\ln \frac{c_3}{c_1}+\ln\ln\ln\left(\frac{2c_2c_3}{c_1}\right)^2
\end{align*}
and 
\begin{align*}
    \ln (c_2\ln c_3x)=&\ln c_2+\ln\ln\left(\frac{c_3}{c_1}\left[\ln c_2+\ln\ln \frac{c_3}{c_1}+\ln\ln\ln\left(\frac{2c_2c_3}{c_1}\right)^2\right]\right)\\
    =&\ln c_2+\ln\ln \frac{c_3}{c_1}+\ln\ln\ln\left(c_2\ln\left(\frac{c_3}{c_1}\ln\left(\frac{2c_2c_3}{c_1}\right)^2\right)\right).
\end{align*}
Since
\begin{align*}
   c_2\ln\left(\frac{c_3}{c_1}\ln\left(\frac{2c_2c_3}{c_1}\right)^2\right)\le& c_2\ln\left(\frac{2c_3}{c_1}\left(\frac{2c_2c_3}{c_1}-1\right)\right)\\
   \le &c_2\left(\frac{2c_3}{c_1}\left(\frac{2c_2c_3}{c_1}-1\right)-1\right)\\
   <&\left(\frac{2c_2c_3}{c_1}\right)^2,
\end{align*}
$c_1x\ge \ln (c_2\ln c_3x)$ holds.
\end{proof}

\begin{proposition}
For the function $\mathcal{C}_{\text{exp}}(x)$ defined by Eq.~(\ref{Cexp}),
\begin{align*}
 \mathcal{C}_{\text{exp}}(x)=x+o(x)   
\end{align*}
holds.\label{prop:Cexp-order}
\end{proposition}
\begin{proof}
Let $X=\frac{h^{-1}(x+1)+\ln(\pi^2/3)}{2}$.
Since inequality 
\begin{align}
h^{-1}(x)\le x+\ln\left(x+\sqrt{2(x-1)}\right)\label{prop8KK}
\end{align}
holds for $x\ge 1$ by Proposition 8 in \citep{kaufmann2021},
\begin{align*}
X\le& \frac{x+1+\ln\left(1+x+\sqrt{2x}\right)+\ln(\pi^2/3)}{2}
\end{align*}
holds. Thus,
\begin{align}
    h^{-1}(X)\le& X+\ln\left(X+\sqrt{2(X-1)}\right)\nonumber\\
    \le&\frac{x+1+\ln\left(1+x+\sqrt{2x}\right)+\ln(\pi^2/3)}{2}\nonumber\\
&+\ln\left(\frac{x+1+\ln\left(1+x+\sqrt{2x}\right)+\ln(\pi^2/3)}{2}+\sqrt{x-1+\ln\left(1+x+\sqrt{2x}\right)+\ln(\pi^2/3)}\right)\nonumber\\
    =& \frac{x}{2}+o(x) \label{h-1funcbound}
\end{align}
holds.
For $X\ge h(1/\ln(3/2))$,
\begin{align}
\mathcal{C}_{\text{exp}}(x)=&2\tilde{h}(X)\nonumber\\
=& 2e^{1/h^{-1}(X)}h^{-1}(X)\nonumber\\
\le & 2\left(1+e\frac{1}{h^{-1}(X)}\right)h^{-1}(X)\nonumber\\
=& 2h^{-1}(X)+2e\label{Cexpbound}
\end{align}
holds for $h^{-1}(X)\ge 1$.
Therefore, $\mathcal{C}_{\text{exp}}(x)=x+o(x)$ holds by Ineqs.~(\ref{h-1funcbound}) and (\ref{Cexpbound}).
\end{proof}

\begin{lemma}\label{lem:stopping-time-upper-bound}
Function $\beta(t,\delta)$ defined by Eq.~(\ref{eq:betadef}) satisfies
\begin{align*}
    \inf \left\{ t \in \mathbb{N} : at \ge \beta(t,\delta) \right\}\le\frac{1}{a}\ln \frac{1}{\delta}+o\left(\ln \frac{1}{\delta}\right)
\end{align*}
for $a>0$.
\end{lemma}
\begin{proof}
We have
       \begin{align*}
        &\inf \left\{ t \in \mathbb{N} : at \ge \beta(t,\delta) \right\}\\
        =& \inf \left\{ t \in \mathbb{N} : at \ge 3\sum_{\ell\in \mathcal{L}(\mathcal{T})}\ln(1 + \ln N_{\ell}(t)) + |\mathcal{L}(\mathcal{T})|\mathcal{C}_{\text{exp}}\left(\frac{\ln (1 / \delta)}{|\mathcal{L}(\mathcal{T})|}\right)\right\}\tag{by definition of $\beta(t,\delta)$}\\
        \le & \inf \left\{ t \in \mathbb{N} : at \ge 3 |\mathcal{L}(\mathcal{T})|\ln\left(1 + \ln \frac{t}{|\mathcal{L}(\mathcal{T})|}\right) + |\mathcal{L}(\mathcal{T})|\mathcal{C}_{\text{exp}}\left(\frac{\ln (1 / \delta)}{|\mathcal{L}(\mathcal{T})|}\right)\right\}\tag{by Jensen's inequality for $f(x)=\ln(1+\ln x)$ ($x\ge 1$)}\\
         = & \inf \left\{ t \in \mathbb{N} : \frac{at}{3|\mathcal{L}(\mathcal{T})|} \ge \ln \left(e^{\frac{1}{3}\mathcal{C}_{\text{exp}}\left(\frac{\ln (1 / \delta)}{|\mathcal{L}(\mathcal{T})|}\right)}\ln \frac{et}{|\mathcal{L}(\mathcal{T})|}\right) \right\}\\
         \le& \frac{3|\mathcal{L}(\mathcal{T})|}{a}\left(\frac{1}{3}\mathcal{C}_{\text{exp}}\left(\frac{\ln (1 / \delta)}{|\mathcal{L}(\mathcal{T})|}\right)+\ln\ln\frac{3e}{a}+\ln\ln\ln\left(e^{\frac{1}{3}\mathcal{C}_{\text{exp}}\left(\frac{\ln (1 / \delta)}{|\mathcal{L}(\mathcal{T})|}\right)}\frac{6e}{a}\right)^2\right)\tag{by Proposition~\ref{prop:ineq-hold-upper-bound}}\\
         =& \frac{3|\mathcal{L}(\mathcal{T})|}{a}\left(\frac{1}{3}\mathcal{C}_{\text{exp}}\left(\frac{\ln (1 / \delta)}{|\mathcal{L}(\mathcal{T})|}\right)+\ln\ln\frac{3e}{a}+\ln\ln\left(\frac{2}{3}\mathcal{C}_{\text{exp}}\left(\frac{\ln (1 / \delta)}{|\mathcal{L}(\mathcal{T})|}\right)+2\ln\frac{6e}{a}\right)\right)\\
         =& \frac{1}{a}\log\frac{1}{\delta}+o\left(\frac{\ln (1 / \delta)}{|\mathcal{L}(\mathcal{T})|}\right)+\frac{3|\mathcal{L}(\mathcal{T})|}{a}\left(\ln\ln\frac{3e}{a}+\ln\ln\left(\frac{2\ln (1 / \delta)}{3|\mathcal{L}(\mathcal{T})|}+o\left(\frac{\ln (1 / \delta)}{|\mathcal{L}(\mathcal{T})|}\right)+2\ln\frac{6e}{a}\right)\right)\tag{by Proposition~\ref{prop:Cexp-order}}\\
         =& \frac{1}{a}\ln\frac{1}{\delta}+o\left(\ln \frac{1}{\delta}\right).
    \end{align*}
\end{proof}

Finally, we provide the proof of Theorem~\ref{thm:as_upper_bound}.
\begin{proof}[Proof of Theorem~\ref{thm:as_upper_bound}]
    For all $\epsilon > 0$, there exists $t_1$ such that $(1 + \epsilon)Z_{s_0}(t)/t \geq d_{s_0}(\bm{\mu})$  for all $t \geq t_1$ by Lemma~\ref{lem:GLR_converge}.
Thus,
    \begin{align*}
        \tau_\delta 
        = &\inf \left\{ t \in \mathbb{N} : Z_{s_0}(t) \ge \beta(t,\delta) \right\}\\
        \leq &t_1 \vee \inf \left\{ t \in \mathbb{N} : t (1 + \epsilon)^{-1} d_{s_0}(\bm{\mu}) \ge \beta(t,\delta) \right\}\\
         =& t_1\vee \frac{(1+\epsilon)\ln (1 / \delta)}{d_{s_0}(\bm{\mu})}+o\left(\ln (1 / \delta)\right),\tag{by Lemma~\ref{lem:stopping-time-upper-bound}}         
    \end{align*}
where $a\vee b=\max\{a,b\}$ for any $a,b\in \mathbb{N}$.
    Therefore, $\mathbb{P}(\tau_\delta < +\infty) = 1$.
    Since $t_1$ does not depend on $\delta$, we obtain the asymptotically upper bound 
    \begin{align*}
        \limsup_{\delta \to +0} \frac{\tau_\delta}{\log(1 / \delta)} \leq \frac{1 + \epsilon}{ d_{s_0}(\bm{\mu})}.
    \end{align*}
    Setting $\epsilon \to +0$, the proof is complete.
\end{proof}

\subsection{Continuity of the Optimal Proportion}

Under Assumption~\ref{assumption:no_exact_theta} and \ref{assumption:unique_weight}, we show its continuity at $\bm{\mu}$ for the Thresholding MCTS problem.
\begin{lemma}[Continuity of the optimal proportion $\bm{w}^{s_0}(\cdot)$ at $\bm{\mu}$] \label{lem:weight_continuity}
    For any $\epsilon > 0$, there exists $0 < \xi$ such that
    \[ 
    \max_{\ell \in \mathcal{L}(\mathcal{T})} |w^{s_0}_\ell(\bm{\mu}') - w^{s_0}_\ell(\bm{\mu})| < \epsilon
    \text{\ \ for all \ \ } \bm{\mu}' \in \prod_{\ell\in \mathcal{L}(\mathcal{T})}[\mu_{\ell}-\xi, \mu_{\ell}+\xi].\]
\end{lemma}

The function $w^{s_0}_{\ell}(\cdot)$ is continuous not in the whole mean-vector domain but in a part of the domain that includes $\bm{\mu}$. The following propositions, a corollary, and a lemma are used to make clear such a part of the domain.

The first proposition states the continuity of the function $V_s(\cdot)$ for $s\in S$. 
\begin{proposition}
    $|\mu'_{\ell}-\mu_{\ell}|< \epsilon$ for $\ell\in\mathcal{L}(\mathcal{T})\Rightarrow |V_{s}(\bm{\mu}')-V_{s}(\bm{\mu})|<\epsilon$  for $s\in S$.    \label{prop:tree-value-change}
\end{proposition}
\begin{proof}
    This can be proved by mathematical induction.
    For $s\in \mathcal{L}(\mathcal{T})$, 
    $|V_{s}(\bm{\mu}')-V_{s}(\bm{\mu})|<\epsilon$ holds by the assumption.
    
    Assume that $|V_{c}(\bm{\mu}')-V_{c}(\bm{\mu})|<\epsilon$
    for all $c\in \mathcal{C}(s)$ for node $s\in S\setminus \mathcal{L}(\mathcal{T})$.
    Since $V_c(\bm{\mu})-\epsilon<V_c(\bm{\mu}')<V_c(\bm{\mu})+\epsilon$,
    \begin{align*}
        \max_{c\in\mathcal{C}(s)}V_c(\bm{\mu})-\epsilon<&\max_{c\in\mathcal{C}(s)}V_c(\bm{\mu}')<\max_{c\in\mathcal{C}(s)}V_c(\bm{\mu})+\epsilon \text{ and}\\
        \min_{c\in\mathcal{C}(s)}V_c(\bm{\mu})-\epsilon<&\min_{c\in\mathcal{C}(s)}V_c(\bm{\mu}')<\min_{c\in\mathcal{C}(s)}V_c(\bm{\mu})+\epsilon   
    \end{align*}
    hold. For node $s$, $V_s(\bm{\mu})=\max_{c\in\mathcal{C}(s)}V_c(\bm{\mu})$ if $L(s)=\text{`MAX'}$ and 
    $V_s(\bm{\mu})=\min_{c\in\mathcal{C}(s)}V_c(\bm{\mu})$ 
    if $L(s)=\text{`MIN'}$, thus in both cases,
    $V_s(\bm{\mu})-\epsilon<V_s(\bm{\mu}')<V_s(\bm{\mu})+\epsilon$
    holds. Therefore, $|V_s(\bm{\mu}')-V_s(\bm{\mu})|<\epsilon$ holds.
\end{proof}
We have the following corollary of this proposition.
\begin{corollary}\label{cor:same-answer-range}
    $|\mu'_{\ell}-\mu_{\ell}|<|V_{s_0}(\bm{\mu})-\theta| \text{ for } \ell\in \mathcal{L}(\mathcal{T})\Rightarrow a_{s_0}(\bm{\mu}')=a_{s_0}(\bm{\mu})$   
\end{corollary}
\begin{proof}
    By proposition~\ref{prop:tree-value-change}, $|\mu'_{\ell}-\mu_{\ell}|<|V_{s_0}(\bm{\mu})-\theta| \text{ for } \ell\in \mathcal{L}(\mathcal{T})$ implies $|V_{s_0}(\bm{\mu}')-V_{s_0}(\bm{\mu})|<|V_{s_0}(\bm{\mu})-\theta|$.
    Thus,
    \begin{align*}
        V_{s_0}(\bm{\mu})-|V_{s_0}(\bm{\mu})-\theta|<V_{s_0}(\bm{\mu}')<V_{s_0}(\bm{\mu})+|V_{s_0}(\bm{\mu})-\theta|
    \end{align*}
    Thus,
    \begin{align*}
        V_{s_0}(\bm{\mu})\ge \theta \Rightarrow & V_{s_0}(\bm{\mu}')>V_{s_0}(\bm{\mu})-(V_{s_0}(\bm{\mu})-\theta)=\theta \text{ and}\\
        V_{s_0}(\bm{\mu})< \theta \Rightarrow & V_{s_0}(\bm{\mu}')<V_{s_0}(\bm{\mu})+(\theta-V_{s_0}(\bm{\mu}))=\theta
    \end{align*}
    hold, which means $a_{s_0}(\bm{\mu}')=a_{s_0}(\bm{\mu})$.
\end{proof}

The next proposition is used to prove Proposition~\ref{prop:d-value-change}.
\begin{proposition}\label{prop:ineq_d}
For any positive real numbers $a_1,a_2,\cdots,a_n$ with $n\ge 1$, the following inequalities hold.
\begin{itemize}
    \item[(1)] For $0<\epsilon<\min_{i=1,2,\dots,n} a_i$,
    \[
    \frac{1}{\sum_{i=1}^n\frac{1}{a_i}}-\epsilon<\frac{1}{\sum_{i=1}^n\frac{1}{a_i-\epsilon}}.
    \]
    \item[(2)] For $\epsilon>0$,
    \[
        \frac{1}{\sum_{i=1}^n\frac{1}{a_i+\epsilon}}<\frac{1}{\sum_{i=1}^n\frac{1}{a_i}}+\epsilon.
    \]
\end{itemize}
\end{proposition}
\begin{proof}
    Define $f(x)$ as
    \[
    f(x)=\frac{1}{\sum_{i=1}^n\frac{1}{a_i}}+x-\frac{1}{\sum_{i=1}^n\frac{1}{a_i+x}}.
    \]
    Then, $f(x)$ is continuous on $(-\min_{i=1,2,\dots,n} a_i,\infty)$.  
    Since
    \[
    \frac{d}{dx}f(x)=1-\frac{\sum_{i=1}^n\frac{1}{(a_i+x)^2}}{\left(\sum_{i=1}^n\frac{1}{a_i+x}\right)^2}>0
    \]
    on $(-\min_{i=1,2,\dots,n} a_i,\infty)$ and $f(0)=0$,
     (1) $f(x)<0$ on $(\min_{i=1,2,\dots,n} a_i,0)$ and (2) $f(x)>0$ on $(0,\infty)$ hold.
\end{proof}

The following proposition is on the continuity of function $d_s(\cdot)$ for $s\in S$.
\begin{proposition}
    $|d_{\ell}(\bm{\mu}') - d_{\ell}(\bm{\mu})|<\epsilon\text{ for } \ell \in \mathcal{L}(\mathcal{T})
    \Rightarrow
    |d_{s}(\bm{\mu}') - d_{s}(\bm{\mu})|<\epsilon \text{ for } s\in S$.
    \label{prop:d-value-change}
\end{proposition}
\begin{proof}
    This can be proved by mathematical induction.
    For $s \in \mathcal{L}(\mathcal{T})$, 
    $|d_{s}(\bm{\mu}')-d_{s}(\bm{\mu})|<\epsilon$ holds by the assumption.
    
    Assume that $|d_{c}(\bm{\mu}') - d_{c}(\bm{\mu})| < \epsilon$ for all $c \in \mathcal{C}(s)$ for node $s \in S \setminus \mathcal{L}(\mathcal{T})$,
    and we obtain $d_c(\bm{\mu}) - \epsilon < d_c(\bm{\mu}') < d_c(\bm{\mu}) + \epsilon$.
    
    First, we consider $a_{s_0}(\bm{\mu}') = \text{`win'}$.
    
    In the case with  $L(s) = \text{`MAX'}$,
    \begin{align*}
        \max_{c\in\mathcal{C}(s)}d_c(\bm{\mu})-\epsilon <
        \max_{c\in\mathcal{C}(s)}d_c(\bm{\mu}') < \max_{c\in\mathcal{C}(s)}d_c(\bm{\mu})+\epsilon 
    \end{align*}
    holds. Since $d_s(\bm{\mu}) = \max_{c \in \mathcal{C}(s)} d_c(\bm{\mu})$,
    \begin{align*}
        d_s(\bm{\mu})-\epsilon <
        d_s(\bm{\mu}') < d_s(\bm{\mu})+\epsilon 
    \end{align*}
    holds.
    
    Consider the case with $L(s) = \text{`MIN'}$.
    
    If $d_c(\bm{\mu})=d_{c'}(\bm{\mu}')=0$ for some $c,c' \in \mathcal{C}(s)$, then $d_s(\bm{\mu})=d_s(\bm{\mu}')=0$.
    Thus, $|d_{s}(\bm{\mu}') - d_{s}(\bm{\mu})|<\epsilon$ holds.
    
    If $d_c(\bm{\mu})=0$ for some $c\in \mathcal{C}(s)$ and $d_c(\bm{\mu}')>0$ for all $c\in \mathcal{C}(s)$, then $d_s(\bm{\mu})=0$ and 
    \[
    d_s(\bm{\mu}') = \frac{1}{\sum_{c \in \mathcal{C}(s)} \frac{1}{d_c(\bm{\mu}')}}<\frac{1}{\frac{1}{\epsilon}}=\epsilon.
    \]
    Thus, $|d_{s}(\bm{\mu}') - d_{s}(\bm{\mu})|<\epsilon$ holds.

  If $d_c(\bm{\mu})>0$ for all $c\in \mathcal{C}(s)$ and $d_c(\bm{\mu}')=0$ for some $c\in \mathcal{C}(s)$, then $d_c(\bm{\mu})<\epsilon$ by the induction assumption.
    In this case, $d_s(\bm{\mu}')=0$ and 
  \[
    d_s(\bm{\mu}) = \frac{1}{\sum_{c \in \mathcal{C}(s)} \frac{1}{d_c(\bm{\mu})}}<\frac{1}{\frac{1}{\epsilon}}=\epsilon.
    \]
    Thus, $|d_{s}(\bm{\mu}') - d_{s}(\bm{\mu})|<\epsilon$ holds.

   If $d_c(\bm{\mu}),d_c(\bm{\mu}')>0$ for all $c\in \mathcal{C}(s)$, then
    \begin{align*}
       d_s(\bm{\mu}') =
        \frac{1}{\sum_{c\in\mathcal{C}(s)} \frac{1}{d_c(\bm{\mu}')}}
        < \frac{1}{\sum_{c\in\mathcal{C}(s)} \frac{1}{d_c(\bm{\mu}) + \epsilon}}
        \leq \frac{1}{\sum_{c\in\mathcal{C}(s)} \frac{1}{d_c(\bm{\mu})}} + \epsilon
        = d_s(\bm{\mu})+\epsilon
    \end{align*}
    holds, where the last inequality holds by Proposition~\ref{prop:ineq_d}. When $d_c(\bm{\mu})>\epsilon$ for all $c\in \mathcal{C}(s)$,
   \begin{align*}
    d_s(\bm{\mu})-\epsilon=
        \frac{1}{\sum_{c\in\mathcal{C}(s)} \frac{1}{d_c(\bm{\mu})}} - \epsilon
        \leq \frac{1}{\sum_{c\in\mathcal{C}(s)} \frac{1}{d_c(\bm{\mu}) - \epsilon}}
        < \frac{1}{\sum_{c\in\mathcal{C}(s)} \frac{1}{d_c(\bm{\mu}')}}
        =d_s(\bm{\mu}')
    \end{align*}    
    holds, where the first inequality holds by Proposition~\ref{prop:ineq_d}. when $d_c(\bm{\mu})\le \epsilon$ for some $c\in \mathcal{C}(s)$,
  \begin{align*}
    d_s(\bm{\mu})-\epsilon=
        \frac{1}{\sum_{c\in\mathcal{C}(s)} \frac{1}{d_c(\bm{\mu})}} - \epsilon
        \leq \frac{1}{\frac{1}{\epsilon}}-\epsilon=0
        <d_s(\bm{\mu}')
    \end{align*}  
    holds.
    Therefore, $|d_{s}(\bm{\mu}') - d_{s}(\bm{\mu})|<\epsilon$ holds.
    
    Similarly, we can show   $|d_{s}(\bm{\mu}') - d_{s}(\bm{\mu})|<\epsilon$ even if $a_{s_0}(\bm{\mu}') = \text{`lose'}$.
\end{proof}

The following proposition states a relation between the functions $d_s(\cdot)$ and $V_s(\cdot)$, and is used in the proof of Proposition~\ref{prop:WeightPositiveCond}.
\begin{proposition}\label{prop:RelationBetweenVandd}
   For $s\in S$,
  \[
  d_s(\bm{\mu})>0 \Leftrightarrow V_s(\bm{\mu})\begin{cases}
      >\theta & (a_{s_0}(\bm{\mu})=\text{`win'})\\
      <\theta & (a_{s_0}(\bm{\mu})=\text{`lose'}).
  \end{cases}
  \] 
\end{proposition}
\begin{proof}
For $s\in \mathcal{L}(\mathcal{T})$,
\[
\begin{cases}
    d_s(\bm{\mu})=d(\mu_s,\theta)>0 \Leftrightarrow V_s(\bm{\mu})=\mu_s>\theta& (a_{s_0}(\bm{\mu})=\text{`win'})\\
    d_s(\bm{\mu})=d(\mu_s,\theta)>0 \Leftrightarrow V_s(\bm{\mu})=\mu_s<\theta& (a_{s_0}(\bm{\mu})=\text{`lose'})
\end{cases}
\]
hold. Thus, proposition statement holds in this case.

Assume that the proposition statement holds for $s\in \mathcal{C}(s')$.

In the case with $L(s')=\text{`MAX'}$, 
\[
\begin{cases}
 \displaystyle d_{s'}(\bm{\mu})=\max_{s\in\mathcal{C}(s')} d_s(\bm{\mu})>0 \Leftrightarrow V_{s'}(\bm{\mu})=\max_{s\in \mathcal{C}(s')}V_s(\bm{\mu})>\theta &  (a_{s_0}(\bm{\mu})=\text{`win'})\\
 \displaystyle d_{s'}(\bm{\mu})>0 \Leftrightarrow d_s(\bm{\mu})>0 \text{ for } s\in \mathcal{C}(s') \Leftrightarrow V_{s'}(\bm{\mu})=\max_{s\in \mathcal{C}(s')}V_s(\bm{\mu})<\theta & (a_{s_0}(\bm{\mu})=\text{`lose'}).
\end{cases}
\]
In the case with $L(s')=\text{`MIN'}$, 
\[
\begin{cases}
\displaystyle d_{s'}(\bm{\mu})>0 \Leftrightarrow d_s(\bm{\mu})>0 \text{ for } s\in \mathcal{C}(s') \Leftrightarrow V_{s'}(\bm{\mu})=\min_{s\in \mathcal{C}(s')}V_s(\bm{\mu})>\theta & (a_{s_0}(\bm{\mu})=\text{`win'})\\
 \displaystyle d_{s'}(\bm{\mu})=\max_{s\in\mathcal{C}(s')} d_s(\bm{\mu})>0 \Leftrightarrow V_{s'}(\bm{\mu})=\min_{s\in \mathcal{C}(s')}V_s(\bm{\mu})<\theta &  (a_{s_0}(\bm{\mu})=\text{`lose'}).
\end{cases}
\]
Thus, the proposition statement holds for $s'$.
Therefore, the proposition statement holds for $s\in S$ by induction.
\end{proof}

The next proposition makes clear the necessary and sufficient condition of $w^{s_0}_{\ell}(\bm{\mu})>0$ for leaf nodes $\ell\in \mathcal{L}(\mathcal{T})$, and is used in the proof of Lemma~\ref{lem:same-nonzero-range}.
\begin{proposition}\label{prop:WeightPositiveCond}
Let $s_0,s_1,\cdots,s_k=\ell$ denote the path from node $s_0$ to $\ell$ in $\mathcal{T}$. Then, 
\[
w^{s_0}_\ell(\bm{\mu}) > 0\Leftrightarrow \begin{aligned}[t]
&d_{s_i}(\bm{\mu})>0 \text{ for } i=0,1,\dots,k \text{ and }\\ &s_{i+1}=c^*(s_i) \text{ for } i=0,1,\dots,k-1 \text{ with } L(s_i)=\begin{cases}
  \text{`MAX'} & (a_{s_0}(\bm{\mu})=\text{`win'})\\
  \text{`MIN'} &(a_{s_0}(\bm{\mu})=\text{`lose'}).
\end{cases}
\end{aligned}
\]
\end{proposition}
\begin{proof}
We show the case with $a_{s_0}(\bm{\mu})=\text{`win'}$. Similar proof works for the case with $a_{s_0}(\bm{\mu})=\text{`lose'}$.\\
($\Rightarrow$) By Assumption~\ref{assumption:no_exact_theta}, $V_{s_0}(\bm{\mu})>\theta$ in this case, which implies $d_{s_0}(\bm{\mu})>0$ by Proposition~\ref{prop:RelationBetweenVandd}.
Assume $w^{s_0}_\ell(\bm{\mu}) > 0$. We prove $d_{s_i}(\bm{\mu})>0$ for all $i=0,1,\dots,k$ by contradiction. Assume $d_{s_i}(\bm{\mu})=0$ for some $i\in \{0,1,\dots,k\}$. Let $i_0$ be the smallest such index, that is, $d_{s_i}(\bm{\mu})>0$ for $i=0,1,\dots,i_0-1$ and $d_{s_{i_0}}(\bm{\mu})=0$. Since $d_{s_0}(\bm{\mu})>0$, $i_0>0$. $L(s_{i_0-1})\neq\text{`MIN'}$ because $d_{s_{i_0-1}}(\bm{\mu})>0$ and $d_{s_{i_0}}(\bm{\mu})=0$.
Thus, $L(s_{i_0-1})=\text{`MAX'}$ and $s_{i_0}\neq c^*(s_{i_0-1})$. Therefore, $w^{s_{i_0-1}}_{\ell}(\bm{\mu})=0$.
The fact that $d_{s_i}(\bm{\mu})>0$ for $i=0,1,\dots,i_0-1$ implies $L(s_i)=\text{`MAX'}$ or ($L(s_i)=\text{`MIN'}$ and $d_c(\bm{\mu})>0$ for $c\in \mathcal{C}(s_i)$) for $i=0,1,\dots,i_0-1$. Thus, $w^{s_i}_{\ell}(\bm{\mu})=w^{s_{i+1}}_{\ell}(\bm{\mu})$ or $w^{s_i}_{\ell}(\bm{\mu})=\frac{w^{s_{i+1}}_{\ell}(\bm{\mu})/d_{s_{i+1}}(\bm{\mu})}{\sum_{c\in \mathcal{C}(s_i)}1/d_c(\bm{\mu})}$ for $i=0,1,\dots,i_0-1$.
Hence, $w^{s_0}_{\ell}=0$, which contradicts the assumption. Therefore, $d_{s_i}(\bm{\mu})>0$ for all $i=0,1,\dots,k$.
In this case, $s_{i+1}=c^*(s_i)$ for $i=0,1,\dots,k-1$ with $L(s_i)=\text{`MAX'}$ because, if not, $w^{s_0}_{\ell}(\bm{\mu})=0$ by a similar reason, which also contradicts the assumption.\\
($\Leftarrow$) Assume that 
$d_{s_i}(\bm{\mu})>0 \text{ for } i=0,1,\dots,k \text{ and } s_{i+1}=c^*(s_i) \text{ for } i=0,1,\dots,k-1 \text{ with } L(s_i)=\text{`MAX'}$.
Then, $L(s_i)=\text{`MAX'}$ or ($L(s_i)=\text{`MIN'}$ and $d_c(\bm{\mu})>0$ for $c\in \mathcal{C}(s_i)$) for $i=0,1,\dots,k-1$. Therefore, $w^{s_i}_{\ell}(\bm{\mu})=w^{s_{i+1}}_{\ell}(\bm{\mu})$ or $w^{s_i}_{\ell}(\bm{\mu})=\frac{w^{s_{i+1}}_{\ell}(\bm{\mu})/d_{s_{i+1}}(\bm{\mu})}{\sum_{c\in \mathcal{C}(s_i)}1/d_c(\bm{\mu})}$ for $i=0,1,\dots,k-1$. Since $w^{s_k}_{\ell}(\bm{\mu})>0$, $w^{s_0}_{\ell}(\bm{\mu})>0$.
\end{proof}

The following lemma shows the part of the domain that guarantees $\bm{w}^{s_0}(\cdot)$ has the same non-zero components.
\begin{lemma}\label{lem:same-nonzero-range}
Let 
\begin{align*}
\mathcal{L}(\mathcal{T})_{\bm{\mu}}=&\{\ell\in\mathcal{L}(\mathcal{T})\mid w^{s_0}_{\ell}(\bm{\mu})>0\}\\
\Gamma_{\bm{\mu}}^{\mathrm{abs}}=& \min_{s\in S:\mathcal{D}(s)\cap\mathcal{L}(\mathcal{T})_{\bm{\mu}}\neq\emptyset} d_s(\bm{\mu})\\
\Gamma_{\bm{\mu}}^{*} = &\begin{cases}\displaystyle\min_{s \in S: \mathcal{D}(s)\cap\mathcal{L}(\mathcal{T})_{\bm{\mu}}\neq \emptyset,L(s)=\text{`MAX'}}\min_{c \in \mathcal{C}(s) \setminus \{c^*(s)\}} \left( d_{c^*(s)}(\bm{\mu}) - d_{c}(\bm{\mu}) \right) & (a_{s_0}(\bm{\mu})=\text{`win'})\\
\displaystyle\min_{s \in S: \mathcal{D}(s)\cap\mathcal{L}(\mathcal{T})_{\bm{\mu}}\neq \emptyset,L(s)=\text{`MIN'}} \min_{c \in \mathcal{C}(s) \setminus \{c^*(s)\}}\left( d_{c^*(s)}(\bm{\mu}) - d_{c}(\bm{\mu}) \right) & (a_{s_0}(\bm{\mu})=\text{`lose'}).
\end{cases}
\end{align*}
Then,
    \[
    |\mu'_{\ell}-\mu_{\ell}|<|V_{s_0}(\bm{\mu})-\theta|, |d_{\ell}(\bm{\mu}') - d_{\ell}(\bm{\mu})| < \min \{\Gamma_{\bm{\mu}}^{*}/2,\Gamma_{\bm{\mu}}^{\mathrm{abs}}\} \text{ for } \ell\in \mathcal{L}(\mathcal{T})
    \Rightarrow
    (w^{s_0}_\ell(\bm{\mu}') > 0 \Leftrightarrow w^{s_0}_\ell(\bm{\mu}) > 0) \text{ for } \ell\in \mathcal{L}(\mathcal{T}). \]
\end{lemma}
\begin{proof}
Assume that $ |\mu'_{\ell}-\mu_{\ell}|<|V_{s_0}(\bm{\mu})-\theta|, |d_{\ell}(\bm{\mu}') - d_{\ell}(\bm{\mu})| < \min \{\Gamma_{\bm{\mu}}^{*}/2,\Gamma_{\bm{\mu}}^{\mathrm{abs}}\} \text{ for } \ell\in \mathcal{L}(\mathcal{T})$.
Since $|\mu'_{\ell}-\mu_{\ell}|<|V_{s_0}(\bm{\mu})-\theta|$ for $\ell\in \mathcal{L}(\mathcal{T})$,  $a_{s_0}(\bm{\mu}')=a_{s_0}(\bm{\mu})$ holds by Corollary~\ref{cor:same-answer-range}.
In the following, we prove the corollary inequality in the case with $a_{s_0}(\bm{\mu})=\text{`win'}$.
We can prove the inequality in the case with $a_{s_0}(\bm{\mu})=\text{`lose'}$ similarly. 
    By proposition~\ref{prop:d-value-change}, $|d_{\ell}(\bm{\mu}') - d_{\ell}(\bm{\mu})| < \min \{\Gamma_{\bm{\mu}}^{*}/2,\Gamma_{\bm{\mu}}^{\mathrm{abs}}\}$ implies $|d_{s}(\bm{\mu}') - d_{s}(\bm{\mu})|<\min \{\Gamma_{\bm{\mu}}^{*}/2,\Gamma_{\bm{\mu}}^{\mathrm{abs}}\}$ for $s \in S$.
For $\ell\in \mathcal{L}(\mathcal{T})$ with $w^{s_0}_{\ell}(\bm{\mu})>0$, let $s_0,s_1,\cdots,s_k=\ell$ be the node sequence on the path from the root node $s_0$ to the leaf node $\ell$. Then, $d_{s_i}(\bm{\mu}')>d_{s_i}(\bm{\mu})-\Gamma_{\bm{\mu}}^{\mathrm{abs}}\ge 0$ for $i=0,1,\dots,k$ and $c^*(s_i)=c^*(s_i)$ for $i=0,1,\dots,k-1$  with $L(s_i)=\text{“MAX'}$ because $d_{c^*(s_i)}(\bm{\mu}')>d_{c^*(s_i)}(\bm{\mu})-\Gamma_{\bm{\mu}}^{*}/2\ge d_{c}(\bm{\mu})+\Gamma_{\bm{\mu}}^{*}/2>d_{c}(\bm{\mu}')$ for $c\in \mathcal{C}(s)\setminus \{c^*(s_i)\}$.
Thus, $w^{s_0}_{\ell}(\bm{\mu}')>0$ by Proposition~\ref{prop:WeightPositiveCond}.
For $\ell\in \mathcal{L}(\mathcal{T})$ with $w^{s_0}_{\ell}(\bm{\mu})=0$, let $i_0=\min\{i\in\{0,1,\dots,k\}\mid \mathcal{D}(s_i)\cap\mathcal{L}(\mathcal{T})_{\bm{\mu}}\neq \emptyset\}$. Then, $L(s_{i_0-1})=\text{`MAX'}$ and $s_{i_0}\neq c^*(s_{i_0-1})$. Thus, $d_{c^*(s_{i_0-1})}(\bm{\mu}')>d_{c^*(s_{i_0-1})}(\bm{\mu})-\Gamma_{\bm{\mu}}^{*}/2\ge d_{s_{i_0}}(\bm{\mu})+\Gamma_{\bm{\mu}}^{*}/2>d_{s_{i_0}}(\bm{\mu}')$ holds, which implies $w^{s_0}_{\ell}(\bm{\mu}')=0$. Therefore,  
    $(w^{s_0}_\ell(\bm{\mu}') > 0 \Leftrightarrow w^{s_0}_\ell(\bm{\mu}) > 0)$ holds.
\end{proof}

\begin{proof}[Proof of Lemma~\ref{lem:weight_continuity}]
We prove the lemma in the case with  $a_{s_0}(\bm{\mu})=\text{`win'}$. We can prove similarly in the case with $a_{s_0}(\bm{\mu})=\text{`lose'}$.

Since Function
\[
d_{\ell}(\bm{\mu})=\begin{cases}
    d(\mu_{\ell},\theta) & (\mu_{\ell}\ge\theta)|\\
    0 & (\mu_{\ell}<0)
\end{cases}
\]
is continuous, there exists $\xi_{\bm{\mu}}>0$ such that 
\[
\bm{\mu}'\in \prod_{\ell\in \mathcal{L}(\mathcal{T})}[\mu-\xi_{\bm{\mu}},\mu+\xi_{\bm{\mu}}] \Rightarrow |\mu'_{\ell}-\mu_{\ell}|<|V_{s_0}(\bm{\mu})-\theta|, |d_{\ell}(\bm{\mu}') - d_{\ell}(\bm{\mu})| < \min \{\Gamma_{\bm{\mu}}^{*}/2,\Gamma_{\bm{\mu}}^{\mathrm{abs}}\} \text{ for } \ell\in \mathcal{L}(\mathcal{T}).
\]
Then, $ a_{s_0}(\bm{\mu}') = a_{s_0}(\bm{\mu})$ by Corollary~\ref{cor:same-answer-range} and $(w^{s_0}_\ell(\bm{\mu}') > 0 \Leftrightarrow w^{s_0}_\ell(\bm{\mu}) > 0) $ for all $\ell \in \mathcal{L}(\mathcal{T})$ by Lemma~\ref{lem:same-nonzero-range}.
By Proposition~\ref{prop:weight_val_form},
\[
w_{\ell}^{s_0}(\bm{\mu})=\frac{d_{s_0}(\bm{\mu})}{d(\mu_{\ell},\theta)},\ \ \ \ w_{\ell}^{s_0}(\bm{\mu}')=\frac{d_{s_0}(\bm{\mu}')}{d(\mu'_{\ell},\theta)}  
\]
for $\ell\in \mathcal{L}(\mathcal{T})$ with $w^{s_0}_{\ell}(\bm{\mu})>0$.
Since $d(\mu_{\ell},\theta)=d_{\ell}(\cdot)$ is continuous as above, and $d_{s_0}(\cdot)$  is continuous 
by Proposition~\ref{prop:d-value-change}, $w^{s_0}_{\ell}(\cdot)$ is also continuous. Thus, for any $\epsilon>0$, there exists $0<\xi<\xi_{\bm{\mu}}$ such that 
\[
\bm{\mu}'\in \prod_{\ell\in \mathcal{L}(\mathcal{T})}[\mu-\xi,\mu+\xi] \Rightarrow 
|w^{s_0}_\ell(\bm{\mu}') - w^{s_0}_\ell(\bm{\mu})| < \epsilon
\text{ for } \ell\in \mathcal{L}(\mathcal{T}).
\]
\end{proof}

\subsection{Proof of Lemma~\ref{lem:sampling_proportion_convergence}}\label{appendix:RatioConvergence}

The following lemma is the RD-Tracking-TMCTS version of  Lemma~17 of \citep{pmlr-v49-garivier16a}, which is used to prove Lemma~\ref{lem:sampling_proportion_convergence}.

\begin{lemma}\label{lem:DrawProportionConvergence}
    For $\epsilon>0$, assume that there exists a positive integer $t_0$ such that 
\begin{align}
|w^{s_0}_{\ell}(\hat{\bm{\mu}}(t))-w^{s_0}_{\ell}(\bm{\mu})|\le \epsilon \text{ for all } t\ge t_0 \text{ and } \ell\in \mathcal{L}(\mathcal{T}).\label{ineq:LemAssumption}
\end{align}
Then, the sampling rule of RD-Tracking-TMCTS ensures that there exists a positive integer $t_1\ge t_0$ such that 
\begin{align}
\max_{\ell\in \mathcal{L}(\mathcal{T})}\left|\frac{N_{\ell}(t)}{t}- w^{s_0}_{\ell}(\bm{\mu})\right|\le 2(|\mathcal{L}(\mathcal{T})|-1)\epsilon \text{\ \ for\ \ } t\ge t_1.\label{ineq:UpperBoundDiffFromOpt}
\end{align}
\end{lemma}

\begin{proof}
For all $\ell\in \mathcal{L}(\mathcal{T})$, let $E_{\ell,t}=N_{\ell}(t)-tw^{s_0}_{\ell}(\bm{\mu})$. Then
\begin{align}
\max_{\ell\in \mathcal{L}(\mathcal{T})}|E_{\ell,t}|\le (|\mathcal{L}(\mathcal{T})|-1)\max_{\ell\in \mathcal{L}(\mathcal{T})} E_{\ell,t}\label{ineq:E_ltAbsVal}
\end{align}
holds because $\sum_{\ell\in \mathcal{L}(\mathcal{T})}E_{\ell,t}=0$.

Let $t'_0\ge \max\{t_0,1/\epsilon\}$ be a positive integer such that 
\[
\sqrt{t+1}-|\mathcal{L}(\mathcal{T})|/2\le t\epsilon \text{ for all } t\ge t'_0.
\]

We can prove that 
\begin{align}
(I(t+1)=\ell)\subseteq (E_{\ell,t}\le t\epsilon) \text{ for } t\ge t'_0  \label{event:E_ltInclusion} 
\end{align}
as follows.
By the sampling rule of RD-Tracking-TMCTS, we have
\begin{align*}
    (I(t+1)=\ell)\subseteq \left(N_{\ell}(t)< \sqrt{t+1}-|\mathcal{L}(\mathcal{T})|/2 \text{ or } \ell=\argmax_{\ell'\in \mathcal{L}(\mathcal{T})}\frac{w^{s_0}_{\ell'}(\hat{\bm{\mu}}(t))}{N_{\ell'}(t)}\right).
\end{align*}
In the case with $N_{\ell}(t)< \sqrt{t+1}-|\mathcal{L}(\mathcal{T})|/2$,
\[
E_{\ell,t}<\sqrt{t+1}-|\mathcal{L}(\mathcal{T})|/2-tw^{s_0}_{\ell}(\bm{\mu})\le \sqrt{t+1}-|\mathcal{L}(\mathcal{T})|/2 \le t\epsilon
\]
for $t\ge t'_0$.
In the case with $\ell=\argmax_{\ell'\in \mathcal{L}(\mathcal{T})}\frac{w^{s_0}_{\ell'}(\hat{\bm{\mu}}(t))}{N_{\ell'}(t)}$,
\[
\frac{w^{s_0}_{\ell}(\hat{\bm{\mu}}(t))}{\frac{N_{\ell}(t)}{t}}\ge 1
\]
holds because $\bm{w}^{s_0}(\hat{\bm{\mu}}(t)),\left(\frac{N_{\ell'}(t)}{t}\right)_{\ell'\in \mathcal{L}(\mathcal{T})}\in \Delta$.
Thus, we have
\begin{align*}
    N_{\ell}(t)-tw^{s_0}_{\ell}(\hat{\bm{\mu}}(t))\le& 0\\
    N_{\ell}(t)-tw^{s_0}_{\ell}(\bm{\mu})+t(w^{s_0}_{\ell}(\bm{\mu})-w^{s_0}_{\ell}(\hat{\bm{\mu}}(t)))\le& 0\\
    E_{\ell,t}\le& -t(w^{s_0}_{\ell}(\bm{\mu})-w^{s_0}_{\ell}(\hat{\bm{\mu}}(t)))\le t\epsilon
\end{align*}
for $t\ge t'_0\ge t_0$.
Therefore, 
\[
\left(N_{\ell}(t)< \sqrt{t+1}-|\mathcal{L}(\mathcal{T})|/2 \text{ or } \ell=\argmax_{\ell'\in \mathcal{L}(\mathcal{T})}\frac{w^{s_0}_{\ell'}(\hat{\bm{\mu}}(t))}{N_{\ell'}(t)}\right)\subseteq (E_{\ell,t}\le t\epsilon)
\]
holds for $t\ge t'_0$.
Thus, (\ref{event:E_ltInclusion}) holds.

Since $E_{\ell,t+1}=E_{\ell,t}+\mathds{1}\{I(t+1)=\ell\}-\mu_{\ell}$, (\ref{event:E_ltInclusion}) implies
\[
E_{\ell,t+1}\le E_{\ell,t}+\mathds{1}\{E_{\ell,t}\le t\epsilon\}-\mu_{\ell} \text{ for } t\ge t'_0.
\]
Then, we can prove that
\begin{align}
E_{\ell,t}\le \max\{E_{\ell,t'_0}, t\epsilon+1\}\label{ineq:E_lt}
\end{align}
for $ t\ge t'_0$
by induction on $t$ as follows.
For $t=t'_0$, Ineq.~(\ref{ineq:E_lt}) holds trivially.
Assume that Ineq.~(\ref{ineq:E_lt}) holds for some $t\ge t'_0$. 
Then, 
\begin{align*}  E_{\ell,t+1}\le\begin{cases}
    t\epsilon+1-\mu_{\ell}\le t\epsilon+1 & (E_{\ell,t}\le t\epsilon)\\
    \max\{E_{\ell,t'_0}, t\epsilon+1\}-\mu_{\ell}\le \max\{E_{\ell,t'_0}, t\epsilon+1\} & (E_{\ell,t}> t\epsilon).
    \end{cases}
\end{align*}
Thus, $E_{\ell,t+1}\le \max\{E_{\ell,t'_0}, (t+1)\epsilon+1\}$ holds, which means that (\ref{ineq:E_lt}) holds for $t+1$.
Therefore, (\ref{ineq:E_lt}) holds for all $t\ge t'_0$.

For $t\ge t'_0$, by Ineq.~(\ref{ineq:E_ltAbsVal}) and Ineq.~(\ref{ineq:E_lt}),
\[
\max_{\ell\in \mathcal{L}(\mathcal{T})}\left|\frac{E_{\ell,t}}{t}\right|\le (|\mathcal{L}(\mathcal{T})|-1)\max\left\{\frac{E_{\ell,t'_0}}{t}, \epsilon+\frac{1}{t}\right\}\le (|\mathcal{L}(\mathcal{T})|-1)\max\left\{\frac{t'_0}{t},2\epsilon\right\}
\]
holds, where the last inequality uses $E_{\ell,t'_0}\le t'_0$ and $1/t\le \epsilon$.
Therefore, there exists $t_1\ge t'_0$ such that
\[
\max_{\ell\in \mathcal{L}(\mathcal{T})}\left|\frac{E_{\ell,t}}{t}\right|\le 2(|\mathcal{L}(\mathcal{T})|-1)\epsilon \text{ for } t\ge t_1,
\]
which means this lemma's statement holds.

\end{proof}

\begin{remark}\label{remark:UpperBoundDiffFromOpt}
If Assumption (\ref{ineq:LemAssumption}) is restricted as 
    \begin{align}
|w^{s_0}_{\ell}(\hat{\bm{\mu}}(t))-w^{s_0}_{\ell}(\bm{\mu})|\le \epsilon \text{ for all } t_0\le t\le T \text{ and } \ell\in \mathcal{L}(\mathcal{T}),
\end{align}
the existence $t_1\ge t_0$ that satisfies 
\begin{align}
\max_{\ell\in \mathcal{L}(\mathcal{T})}\left|\frac{N_{\ell}(t)}{t}- w^{s_0}_{\ell}(\bm{\mu})\right|\le 2(|\mathcal{L}(\mathcal{T})|-1)\epsilon \text{\ \ for\ \ } t_1\le t \le T
\end{align}
is still guaranteed.
\end{remark}

\begin{proof}[Proof of Lemma~\ref{lem:sampling_proportion_convergence}]
Let $\epsilon>0$ be an arbitrary positive real number.
The forced exploration ($I(t)\gets\min_{\ell\in \mathcal{L}(\mathcal{T})}N_{\ell}(t-1)$ if $N_{\ell}(t-1) < \sqrt{t} - |\mathcal{L}(\mathcal{T})|/2$ for some $\ell\in \mathcal{L}(\mathcal{T})$) of RD-Tracking-TMCTS sampling rule ensures 
$N_{\ell}(t)\rightarrow \infty$ ($t\rightarrow \infty$), thus $\hat{\bm{\mu}}(t)\rightarrow \bm{\mu}$ ($t\rightarrow \infty$) by law of large numbers.
Therefore, $\bm{w}^{s_0}(\hat{\bm{\mu}}(t))\rightarrow \bm{w}^{s_0}(\bm{\mu})$ ($t\rightarrow\infty$) holds by the continuity of the function $\bm{w}^{s_0}(\cdot)$ in the neighborhood of $\bm{\mu}$. Then, there exists a positive integer $t_0$ such that 
\[
|w^{s_0}_{\ell}(\hat{\bm{\mu}}(t))-w^{s_0}_{\ell}(\bm{\mu})|\le \epsilon \text{ for all } t\ge t_0 \text{ and } \ell\in \mathcal{L}(\mathcal{T}).
\]
Thus, by Lemma~\ref{lem:DrawProportionConvergence},
\[
\frac{N_{\ell}(t)}{t}\rightarrow w^{s_0}_{\ell}(\bm{\mu}) \ \ (t\rightarrow \infty)\ \ \text{ for all } \ell\in \mathcal{L}(\mathcal{T}).
\]
\end{proof}

\section{Proof for Sample Complexity of Algorithm~\ref{alg:proposed} (Proof of Theorem~\ref{thm:sample_complexity})}
\label{append:sample_complexity}

The proof is conducted along the same lines as the proof of Theorem 14 in \citet{pmlr-v49-garivier16a}.

For $\epsilon>0$, define 
\[ \mathcal{I}_\epsilon = [\mu_\ell - \xi(\epsilon), \mu_\ell + \xi(\epsilon)]^{|\mathcal{L}(\mathcal{T})|} \]
using $\xi(\epsilon)>0$ that satisfies
\[ \forall \bm{\lambda} \in \mathcal{I}_\epsilon, \max_{\ell \in \mathcal{L}(\mathcal{T})} |w^{s_0}_\ell(\bm{\lambda}) - w^{s_0}_\ell(\bm{\mu})| < \epsilon. \]
Such $\xi(\epsilon)$ exists by Lemma~\ref{lem:weight_continuity}.

Define Event $\mathcal{E}_T(\epsilon)$ as
\[ \mathcal{E}_T(\epsilon) = \bigcap_{t = T^{1/4}}^T (\hat{\bm{\mu}}(t) \in \mathcal{I}_\epsilon) \]
for any $T \in \mathbb{N}$.

The following Corollary~\ref{lem:kaufmann19} and Lemma~\ref{lem:concentration_result} are used to prove the theorem.

Corollary~\ref{lem:kaufmann19} is the RD-Tracking-TMCTS version of Lemma~19 \citep{pmlr-v49-garivier16a}, which holds for the forced exploration of D-Tracking.

\begin{corollary}[Corollary of Lemma~19 \citep{pmlr-v49-garivier16a}]\label{lem:kaufmann19}
         Let $\mathcal{E}_T^c(\epsilon)$ denote the complementary event  of event $\mathcal{E}_T(\epsilon)$. Then, inequality
         \[ \mathbb{P}(\mathcal{E}_T^c(\epsilon)) < BT\exp(-CT^{1/8})\]
         holds for any sampling rule that ensures $N_{\ell}(t)\ge \sqrt{t}-|\mathcal{L}(\mathcal{T})|$ for $t\ge T^{1/4}$,
         where
         \begin{align*}
             B=&\sum_{\ell\in \mathcal{L}(\mathcal{T})}\left(\frac{\displaystyle e^{|\mathcal{L}(\mathcal{T})|d(\mu_{\ell}-\xi(\epsilon),\mu_{\ell})}}{\displaystyle 1-e^{|\mathcal{L}(\mathcal{T})|d(\mu_{\ell}-\xi(\epsilon),\mu_{\ell})}}+\frac{\displaystyle e^{|\mathcal{L}(\mathcal{T})|d(\mu_{\ell}+\xi(\epsilon),\mu_{\ell})}}{\displaystyle 1-e^{|\mathcal{L}(\mathcal{T})|d(\mu_{\ell}+\xi(\epsilon),\mu_{\ell})}}\right) \text{ and}\\
             C=&\min_{\ell\in\mathcal{L}(\mathcal{T})}\min\{d(\mu_{\ell}-\xi(\epsilon),\mu_{\ell}),d(\mu_{\ell}+\xi(\epsilon),\mu_{\ell})\}.
         \end{align*}
    \end{corollary}

The following lemma can be obtained by the definition of $\mathcal{I}_{\epsilon}$ and Lemma~\ref{lem:DrawProportionConvergence} (and Remark~\ref{remark:UpperBoundDiffFromOpt}).
\begin{lemma}[RD-Tracking-TMCTS version of Lemma~20 of \citep{pmlr-v49-garivier16a}]\label{lem:concentration_result}
    For RD-Tracking-TMCTS, there exists $T_\epsilon$ such that, 
    for any $T \geq T_\epsilon$, $\mathcal{E}_T(\epsilon)$ implies the following.
    \[ \forall t \in [\sqrt{T}, T] \cap \mathbb{N}, \forall \ell \in \mathcal{L}(\mathcal{T}), \left| \frac{N_\ell(t)}{t} - w^{s_0}_\ell(\bm{\mu}) \right| \leq 2(|\mathcal{L}(\mathcal{T})|-1)\epsilon. \]
\end{lemma}

\begin{proof}[Proof of Theorem~\ref{thm:sample_complexity}]
Let 
\begin{align*}
  g(\bm{\mu}', \bm{w})
  &= \inf_{\bm{\lambda} \in \mathrm{Alt}(\hat{\bm{\mu}}(t))}\sum_{\ell\in \mathcal{L}(\mathcal{T})}w_{\ell}d(\mu'_{\ell},\lambda_{\ell}).
\end{align*}
Then, using function $g$, the GLR statistic $Z_{s_0}(t)$ can be represented as
\begin{align*}
    Z_{s_0}(t) 
    =& \inf_{\bm{\lambda} \in \mathrm{Alt}(\hat{\bm{\mu}}(t))} \sum_{\ell\in \mathcal{L}(\mathcal{T})}N_{\ell}(t)d(\hat{\mu}_{\ell}(t),\lambda_{\ell})\\
    =& t \inf_{\bm{\lambda} \in \mathrm{Alt}(\hat{\bm{\mu}}(t))} \sum_{\ell\in \mathcal{L}(\mathcal{T})}\frac{N_{\ell}(t)}{t}d(\hat{\mu}_{\ell}(t),\lambda_{\ell})\\
    =& t g\left(\hat{\mu}(t),\left(\frac{N_{\ell}(t)}{t}\right)_{\ell\in \mathcal{L}(\mathcal{T})}\right).
\end{align*}
Define $d_{s_0}^{\epsilon}(\bm{\mu})$ as
\[ d_{s_0}^{\epsilon}(\bm{\mu}) := \inf_{\substack{\bm{\mu}'\in \mathcal{I}_{\epsilon}\\ \bm{w}': || \bm{w}' - \bm{w}^{s_0}(\bm{\mu})|| < 2(|\mathcal{L}(T)| - 1)\epsilon}} g(\bm{\mu}', \bm{w}') \]
Then, there exists $T_{\epsilon}$ such that, on the event $\mathcal{E}_T(\epsilon)$ and for $T\ge T_{\epsilon}$,
\[ Z_{s_0}(t) \geq t d_{s_0}^\epsilon(\bm{\mu}). \]
holds for every $\sqrt{T}\le t \le T $ by Lemma~\ref{lem:concentration_result}.

Thus, for $T\ge T_{\epsilon}$ and on $\mathcal{E}_T(\epsilon)$,
    \begin{align*}
        \min(\tau_\delta, T)
        &\leq 
        \sqrt{T} + \sum_{t = \sqrt{T}}^T \mathds{1}\{\tau_\delta > t\}\\
        &\leq
        \sqrt{T} + \sum_{t = \sqrt{T}}^T \mathds{1}\{Z_{s_0}(t) < \beta(t,\delta)\}\\
        &\leq
        \sqrt{T} + \sum_{t = \sqrt{T}}^T \mathds{1}\{td_{s_0}^\epsilon(\bm{\mu}) < \beta(T,\delta)\}\\
        &\leq
        \sqrt{T} + \frac{\beta(T,\delta)}{d_{s_0}^\epsilon(\bm{\mu})}
    \end{align*}
    holds.
    Let $T_0(\delta)$ be
    \[ T_0(\delta) := \inf \left\{ T \in \mathbb{N} : \sqrt{T} + \frac{\beta(T,\delta)}{d_{s_0}^\epsilon(\bm{\mu})} \leq T \right\}.\]
    Then, for every $T \geq \max (T_0(\delta), T_\epsilon)$, 
    we have $\mathcal{E}_T(\epsilon) \subseteq (\tau_\delta \leq T)$, which implies
    \[ \mathbb{P}_{\bm{\mu}}(\tau_\delta > T) \leq \mathbb{P}(\mathcal{E}_T^c) \leq BT\exp(-CT^{1/8}) \]
    using Corollary~\ref{lem:kaufmann19}. 
    Therefore,
    \begin{align*}
     \mathbb{E}[\tau_\delta] =\mbox{}& \sum_{T=0}^{\infty}\mathbb{P}_{\bm{\mu}}(\tau_\delta > T)\\
     \le\mbox{} & \max (T_0(\delta), T_\epsilon)+\sum_{T=\max (T_0(\delta), T_\epsilon)}^{\infty}\mathbb{P}_{\bm{\mu}}(\tau_\delta > T)\\
     \le\mbox{} & T_0(\delta) + T_\epsilon + \sum_{T= 1}^\infty BT\exp(-CT^{1/8}). 
     \end{align*}
     Note that the sum of the last two terms, $T_\epsilon + \sum_{T= 1}^\infty BT\exp(-CT^{1/8})$, is a constant depending on $\epsilon$.
    We now analyze an upper bound on $T_0(\delta)$.
    Let $\eta > 0$ and define constant $C(\eta)$ as
    \[ C(\eta) = \inf \left\{ T \in \mathbb{N} : T - \sqrt{T} \geq T / (1 + \eta) \right\}. \]
    Then,
    \begin{align*}
        T_0(\delta) 
        \leq&
        C(\eta) + \inf \left\{ T \in \mathbb{N} : \frac{\beta(T,\delta)}{d_{s_0}^\epsilon(\bm{\mu})} \leq \frac{T}{(1 + \eta)} \right\}\\
        =&  C(\eta) +\frac{(1+\eta)\ln (1 / \delta)}{d_{s_0}^{\epsilon}(\bm{\mu})}+o\left(\ln (1 / \delta)\right)\tag{by Lemma~\ref{lem:stopping-time-upper-bound}}
    \end{align*}
    holds.
    Thus, we have
    \begin{align*}
         \mathbb{E}[\tau_\delta] \le  \frac{(1+\eta)\ln (1 / \delta)}{d_{s_0}^{\epsilon}(\bm{\mu})}+o\left(\ln (1 / \delta)\right),
    \end{align*}
    which yields, for every $\eta > 0$ and $\epsilon > 0$,
    \[ \limsup_{\delta \to +0} \frac{\mathbb{E}_{\bm{\mu}}[\tau_\delta]}{\log(1 / \delta)} \leq \frac{1+\eta}{ d_{s_0}^{\epsilon}(\bm{\mu})}. \]
    Since 
    \[ \lim_{\epsilon \to +0} d_{s_0}^{\epsilon}(\bm{\mu}) = d_{s_0}(\bm{\mu}) \]
    holds by continuity of $g$ and definition of $\bm{w}^{s_0}$,
    we have 
    \[ \limsup_{\delta \to +0} \frac{\mathbb{E}_{\bm{\mu}}[\tau_\delta]}{\log(1 / \delta)} \leq \frac{1}{d_{s_0}(\bm{\mu})} \]
    by letting $\eta$ and $\epsilon$ go to zero.
\end{proof}

\section{Time Complexity Analysis}\label{append:time_complexity}

\subsection{Heap Data Structure}

Let $\{(x_i, i)\}_{i=1, \ldots, N}$ be a set of $N$ elements, with $x_i \in \mathbb{R}$.
We consider a heap data structure, $\mathrm{IHeap}$, that supports the following operations on this set.
\begin{description}
    \item[heapify($\{(x_i, i)\}_{i=1, \ldots, N}$):]
    Initialize the heap data structure with the given set $\{(i, x_i)\}_{i=1, \ldots, N}$ in $O(N \log N)$ time.
    \item[peek:]
    Returns $(x_{i^*}, i^*)$ with $i^* = \argmin_i x_i$ in $O(1)$ time.
    \item[rewrite($v, i$):]
    Replace the value associated with index $i$, namely $(x_i, i)$, with $(v, i)$ and maintain the heap in $O(\log N)$ time.
\end{description}
This data structure can be implemented with a heap ordered by the values $x_i$. 
By maintaining an auxiliary array $\texttt{pos}[i]$ storing the heap position of each element, a $\mathrm{rewrite}(v, i)$ operation can be performed in $O(\log N)$ time via a constant-time update followed by at most $\log N$ heap adjustments.

Similarly, we consider another heap data structure $\mathrm{DHeap}$.
It supports the same operations as $\mathrm{IHeap}$, except that the $\mathrm{peek}$ operation returns the $\argmax$ element instead of the $\argmin$.

\subsection{Auxiliary Statistics}

Although $d_s(\hat{\bm{\mu}}(t))$ and $Z_s(t)$ are defined as the solution to a subproblem associated with the subtree rooted at $s$, they depend on $a_{s_0}(\hat{\bm{\mu}}(t))$ and are therefore not completely independent of the subtree. As a result, its value may change due to observations made outside the subtree.
To avoid this, we introduce\footnote{The core idea here is to assign a positive (resp. negative) value for `win' (resp. `lose') node so that we can effectively retain information on each node and back-propagate it to its parent node. In this way, it does not require a major update in the data structure even if the estimated answer switches between `win' and `lose'. Note that the way we use positive/negative signs is different from the well-known ``negamax'' method in minimax games \cite{knuth1975alphabeta} that flips the sign every time it ascends the node. Our sign represents the sign of the overall conclusion of win or lose, rather than exploiting the minimax property. Therefore, the data structure here is able to deal with general class of trees where max and min can be arbitrarily interleaved.} $\tilde{d}_s(\hat{\bm{\mu}}(t))$, and $\tilde{Z}_s(t)$ as stated in Recursive Formula~\ref{rf:tilde_stat}.
These variables no longer depend on $a_{s_0}(\hat{\bm{\mu}}(t))$ and we have $\tilde{d}_s(\hat{\bm{\mu}}(t)) = \tilde{d}_s(\hat{\bm{\mu}}(t-1))$ if $I(t) \notin \mathcal{D}(s)$.
This enables efficient value updates. (See Appendix~\ref{append:time_complexity:implement}.)
Additionally, we encode the information of $a_s(\hat{\bm{\mu}}(t))$ as the sign as follows.
This eliminates the need to explicitly compute $V_s(\hat{\bm{\mu}}(t))$ and $a_s(\hat{\bm{\mu}}(t))$.
\begin{align*}
    \begin{cases}
        \tilde{d}_s(\hat{\bm{\mu}}(t)) \geq 0 \text{ and } \tilde{Z}_s(t) \geq 0
        & (a_s(\hat{\bm{\mu}}(t)) = \text{`win'})\\
        \tilde{d}_s(\hat{\bm{\mu}}(t)) < 0 \text{ and } \tilde{Z}_s(t) < 0
        & (a_s(\hat{\bm{\mu}}(t)) = \text{`lose'}).
    \end{cases}
\end{align*}
These statistics correspond to the values obtained by evaluating Recursive Formula~\ref{rf:optimal_weight_win}, \ref{rf:GLR_win}, \ref{rf:optimal_weight_lose}, and \ref{rf:GLR_lose} with $a_s(\hat{\bm{\mu}}(t))$ in place of $a_{s_0}(\hat{\bm{\mu}}(t))$.
\footnote{For recursive computations at a child node $c$, we also use $a_s(\hat{\bm{\mu}}(t))$ rather than $a_c(\hat{\bm{\mu}}(t))$.}
Therefore, in the case with $a_{s_0}(\hat{\bm{\mu}}(t)) = a_s(\hat{\bm{\mu}}(t))$, these statistics match the original value.
\begin{align*}
    \begin{cases}
        d_s(\hat{\bm{\mu}}(t)) = \tilde{d}_s(\hat{\bm{\mu}}(t)) \text{ and } Z_s(t) = \tilde{Z}_s(t)
        & (a_{s_0}(\hat{\bm{\mu}}(t)) = a_s(\hat{\bm{\mu}}(t)) = \text{`win'})\\
        d_s(\hat{\bm{\mu}}(t)) = -\tilde{d}_s(\hat{\bm{\mu}}(t)) \text{ and } Z_s(t) = -\tilde{Z}_s(t)
        & (a_{s_0}(\hat{\bm{\mu}}(t)) = a_s(\hat{\bm{\mu}}(t)) = \text{`lose'})
    \end{cases}
\end{align*}
Since $Z_{s_0}(t) = |\tilde{Z}_{s_0}(t)|$, we can rewrite the stopping rule and the recommendation rule as:
\begin{align*}
    \tau_\delta &= \inf\{ t \in \mathbb{N} : |\tilde{Z}_{s_0}(t)| \geq \beta(t,\delta) \}\\
    a(t) &= \begin{cases}
        \text{`win'} & (\tilde{Z}_{s_0}(t) \geq 0)\\
        \text{`lose'} & (\tilde{Z}_{s_0}(t) < 0).
    \end{cases}
\end{align*}

Similarly, we redefine the $\bm{w}^s(\hat{\bm{\mu}}(t))$ as $\tilde{\bm{w}}^s(\hat{\bm{\mu}}(t))$ to eliminate its dependence on $a_{s_0}(\hat{\bm{\mu}}(t))$, as shown in Recursive Formula~\ref{rf:tilde_stat}.
This also corresponds to the proportion obtained by evaluating Recursive Formula~\ref{rf:optimal_weight_win} and \ref{rf:optimal_weight_lose} with $a_s(\hat{\bm{\mu}}(t))$ in place of $a_{s_0}(\hat{\bm{\mu}}(t))$.
This satisfies 
\[ \bm{w}^s(\hat{\bm{\mu}}(t)) = \tilde{\bm{w}}^s(\hat{\bm{\mu}}(t)) \text{ when } a_{s_0}(\hat{\bm{\mu}}(t)) = a_s(\hat{\bm{\mu}}(t)). \]

\captionsetup[algorithm]{name=Recursive Formula}
\setcounter{algorithm}{4}
\begin{algorithm}[tb]
  \caption{$d_s(\hat{\bm{\mu}}(t))$, $Z_s(t)$ and $\bm{w}^s(\hat{\bm{\mu}}(t))$ redefined in a form independent of $a_{s_0}(\hat{\bm{\mu}}(t))$} \label{rf:tilde_stat}
\begin{minipage}[t]{0.5\textwidth}  
\begin{align}
    &\tilde{d}_s(\bm{\mu}) \nonumber \\
    =&\begin{cases}
      d(\mu_s,\theta) & (s\in \mathcal{L}(T), \mu_s\ge \theta)\\
      -d(\mu_s,\theta) & (s\in \mathcal{L}(T), \mu_s< \theta)\\     
      \displaystyle\max_{c\in \mathcal{C}(s)}\tilde{d}_c(\bm{\mu}) &
      \left( \begin{aligned}
            &L(s)=\text{`MAX'},\\& \tilde{d}_c(\bm{\mu})\ge 0 \text{ for some } c\in \mathcal{C}(s)
        \end{aligned} \right)\\
  \displaystyle\frac{1}{\displaystyle\sum_{c\in \mathcal{C}(s)}1/\tilde{d}_c(\bm{\mu})} & \left( \begin{aligned}
            &L(s)=\text{`MAX'},\\& \tilde{d}_c(\bm{\mu})<0 \text{ for all } c\in \mathcal{C}(s)
  \end{aligned} \right)\\
       \displaystyle\frac{1}{\displaystyle\sum_{c\in \mathcal{C}(s)}1/\tilde{d}_c(\bm{\mu})} & \left( \begin{aligned}
            &L(s)=\text{`MIN'},\\& \tilde{d}_c(\bm{\mu})>0 \text{ for all } c\in \mathcal{C}(s)
       \end{aligned} \right)\\
        \displaystyle\min_{c\in \mathcal{C}(s)}\tilde{d}_c(\bm{\mu}) & \left( \begin{aligned}
            &L(s)=\text{`MIN'},\\& \tilde{d}_c(\bm{\mu})\le 0 \text{ for some } c\in \mathcal{C}(s)
       \end{aligned} \right)
    \end{cases} \label{eq:tilde_d}
\end{align}

\begin{align*}
    &\tilde{Z}_s(t)\\
    =&\begin{cases}
      N_s(t)d(\hat{\mu}_s(t),\theta) & (s\in \mathcal{L}(T), \hat{\mu}_s(t)\ge \theta)\\
      -N_s(t)d(\hat{\mu}_s(t),\theta) & (s\in \mathcal{L}(T), \hat{\mu}_s(t)< \theta)\\     
      \displaystyle\sum_{c\in \mathcal{C}(s),\tilde{Z}_c(t)>0}\tilde{Z}_c(t) &
      \left( \begin{aligned}
            &L(s)=\text{`MAX'},\\& \tilde{Z}_c(t)\ge 0 \text{ for some } c\in \mathcal{C}(s)
        \end{aligned} \right)\\
  \displaystyle\max_{c\in \mathcal{C}(s)}\tilde{Z}_c(t) & \left( \begin{aligned}
            &L(s)=\text{`MAX'},\\& \tilde{Z}_c(t)<0 \text{ for all } c\in \mathcal{C}(s)
  \end{aligned} \right)\\
       \displaystyle\min_{c\in \mathcal{C}(s)}\tilde{Z}_c(t) & \left( \begin{aligned}
            &L(s)=\text{`MIN'},\\& \tilde{Z}_c(t)>0 \text{ for all } c\in \mathcal{C}(s)
       \end{aligned} \right)\\
        \displaystyle\sum_{c\in \mathcal{C}(s),\tilde{Z}_c(t)<0}\tilde{Z}_c(t) & \left( \begin{aligned}
            &L(s)=\text{`MIN'},\\& \tilde{Z}_c(t)\le 0 \text{ for some } c\in \mathcal{C}(s)
       \end{aligned} \right)
    \end{cases}
\end{align*}
\end{minipage}
\begin{minipage}[t]{0.5\textwidth}
  \begin{align*}
      &\tilde{w}^s_{\ell}(\bm{\mu}) \\
    =&\begin{cases}
         1 & (s\in \mathcal{L}(T))\\
          \tilde{w}^{c^+(s)}_\ell(\bm{\mu}) &
          \left( \begin{aligned}
            &L(s)=\text{`MAX'},\\& \tilde{d}_c(\bm{\mu})\ge 0 \text{ for some } c\in \mathcal{C}(s),\\& \ell \in \mathcal{D}(c^+(s))
        \end{aligned}\right)\\
        0 & \left( \begin{aligned}
            &L(s)=\text{`MAX'},\\ & \tilde{d}_c(\bm{\mu})\ge 0 \text{ for some } c\in \mathcal{C}(s),\\ 
            &\ell \in \mathcal{D}(c),\ c \in \mathcal{C}(s) \setminus \{ c^+(s) \}
        \end{aligned}\right)\\
        \displaystyle\frac{\tilde{w}^c_\ell(\bm{\mu}) / \tilde{d}_c(\bm{\mu})}{\displaystyle\sum_{c' \in \mathcal{C}(s)} 1 / \tilde{d}_{c'}(\bm{\mu})}  & \left( \begin{aligned}
            &L(s)=\text{`MAX'},\\& \tilde{d}_c(\bm{\mu})< 0 \text{ for all } c\in \mathcal{C}(s),\\ 
            &\ell \in \mathcal{D}(c),\ c \in \mathcal{C}(s)
        \end{aligned}\right)\\
\displaystyle\frac{\tilde{w}^c_\ell(\bm{\mu}) / \tilde{d}_c(\bm{\mu})}{\displaystyle\sum_{c' \in \mathcal{C}(s)} 1 / \tilde{d}_{c'}(\bm{\mu})}  &
        \left( \begin{aligned}
          &L(s)=\text{`MIN'},\\& \tilde{d}_c(\bm{\mu})> 0 \text{ for all } c\in \mathcal{C}(s),\\
          &\ell \in \mathcal{D}(c),\ c \in \mathcal{C}(s)
        \end{aligned}\right)\\
\tilde{w}^{c^-(s)}_\ell(\bm{\mu}) & \left( \begin{aligned}
            &L(s)=\text{`MIN'},\\& \tilde{d}_c(\bm{\mu})\le 0 \text{ for some } c\in \mathcal{C}(s),\\ 
            &\ell \in \mathcal{D}(c^-(s))
        \end{aligned}\right)\\
        0 & \left( \begin{aligned}
            &L(s)=\text{`MIN'},\\& \tilde{d}_c(\bm{\mu})\le 0 \text{ for some } c\in \mathcal{C}(s),\\
            &\ell \in \mathcal{D}(c),\ c \in \mathcal{C}(s) \setminus \{ c^-(s) \}
        \end{aligned}\right)
    \end{cases}\\
    & \text{ where } c^+(s) = \argmax_{c \in \mathcal{C}(s)} \tilde{d}_c(\bm{\mu}) \text{ and }\\
    & \phantom{\text{ where }} c^-(s) = \argmin_{c \in \mathcal{C}(s)} \tilde{d}_c(\bm{\mu})
    \end{align*}
\end{minipage}
\end{algorithm}

\subsection{Leaf Selection}

First, we consider the forced exploration step (the first case of Eq.~(\ref{eq:sampling_rule_ratio})).
Maintaining an $\mathrm{IHeap}$ with the set $\{(N_\ell(t), \ell)\}_{\ell \in \mathcal{L}(\mathcal{T})}$, we can obtain $\argmin_{\ell \in \mathcal{L}(\mathcal{T})} N_\ell(t)$.

Next, we consider the ratio-based tracking step (the second case of Eq.~(\ref{eq:sampling_rule_ratio})).
Instead of computing $I_s(t)$ and $I(t)$, we introduce an $a_{s_0}(\hat{\bm{\mu}}(t))$-independent subproblem
\[ \tilde{I}_s(t) = \argmax_{\ell \in \mathcal{D}(s) } \tilde{w}^s_\ell(\hat{\bm{\mu}}(t)) / N_\ell(t). \]
Since $\bm{w}^s(\hat{\bm{\mu}}(t)) = \tilde{\bm{w}}^s(\hat{\bm{\mu}}(t))$ when $a_{s_0}(\hat{\bm{\mu}}(t)) = a_s(\hat{\bm{\mu}}(t))$, this satisfies $\tilde{I}_{s_0}(t) = I_{s_0}(t) = I(t)$.

In the remainder of this section, we distinguish four cases based on the node label $L(s)$ and the value of $a_s(\hat{\bm{\mu}}(t))$.

\noindent\textbf{The case with $L(s) = \text{`MAX'}$ and $a_s(\hat{\bm{\mu}}(t)) = \text{`win'}$:}\\
Since $L(s) = \text{`MAX'}$ and $a_s(\hat{\bm{\mu}}(t)) = \text{`win'}$, there exists $c \in \mathcal{C}(s)$ that satisfies $a_c(\hat{\bm{\mu}}(t)) = \text{`win'}$ and $\tilde{d}_c(\hat{\bm{\mu}}(t)) \geq 0$.
From Recursive Formula~\ref{rf:tilde_stat} we have, 
\begin{align*}
    \tilde{w}^s_\ell(\hat{\bm{\mu}}(t)) =
    \begin{cases}
        \tilde{w}^{c^+(s)}_\ell(\hat{\bm{\mu}}(t)) & (\ell \in \mathcal{D}(c^+))\\
        0 & (\ell \notin \mathcal{D}(c^+)),
    \end{cases}
\end{align*}
where $c^+(s) = \argmax_{c \in \mathcal{C}(s)} \tilde{d}_c(\hat{\bm{\mu}}(t))$.
Therefore, we obtain $\tilde{I}_s(t) = \tilde{I}_{c^+(s)}(t)$.

The same discussion also holds in the case with $L(s) = \text{`MIN'}$ and $a_{s_0}(\hat{\bm{\mu}}(t)) = \text{`lose'}$, except that the sign of $\tilde{d}_s(\hat{\bm{\mu}}(t))$ is reversed.

\noindent\textbf{The case with $L(s) = \text{`MAX'}$ and $a_s(\hat{\bm{\mu}}(t)) = \text{`lose'}$:}\\
Since $L(s) = \text{`MAX'}$ and $a_s(\hat{\bm{\mu}}(t)) = \text{`lose'}$, all $c \in \mathcal{C}(s)$ satisfies $a_c(\hat{\bm{\mu}}(t)) = \text{`lose'}$ and $\tilde{d}_c(\hat{\bm{\mu}}(t)) < 0$.
From Recursive Formula~\ref{rf:tilde_stat}, we have
\begin{equation} \label{eq:tilde_weight_MAX_lose}
    \tilde{w}^s_\ell(\hat{\bm{\mu}}(t)) = 
    \displaystyle \tilde{w}^c_\ell(\hat{\bm{\mu}}(t)) \frac{1 / \tilde{d}_c(\bm{\mu})}{\displaystyle\sum_{c' \in \mathcal{C}(s)} 1 / \tilde{d}_{c'}(\hat{\bm{\mu}}(t))}. 
\end{equation}
To solve $\tilde{I}_s(t)$, we split the space into $\mathcal{D}(s) = \bigcup_{c \in \mathcal{C}(s)} \mathcal{D}(c)$ and introduce a new subproblem:
\[ \tilde{I}'_c(t) = \argmax_{\ell \in \mathcal{D}(c) } \tilde{w}^s_\ell(\hat{\bm{\mu}}(t)) / N_\ell(t).\]
By Eq.~(\ref{eq:tilde_weight_MAX_lose}) and ignoring the common coefficient, this subproblem satisfies
\begin{align*}
    \tilde{I}'_c(t) 
    &= \argmax_{\ell \in \mathcal{D}(s)} \displaystyle \tilde{w}^c_\ell(\hat{\bm{\mu}}(t)) \frac{1 / \tilde{d}_c(\bm{\mu})}{\sum_{c' \in \mathcal{C}(s)} 1 / \tilde{d}_{c'}(\hat{\bm{\mu}}(t))} / N_\ell(t)\\
    &= \argmax_{\ell \in \mathcal{D}(s)} \tilde{w}^c_\ell(\hat{\bm{\mu}}(t)) / N_\ell(t)\\
    &= \tilde{I}_c(t).
\end{align*}
This implies that
\[ \tilde{I}_s(t) = \argmax_{\ell \in \{\tilde{I}_c(t)\}_{c \in \mathcal{C}(s)}} \tilde{w}^s_\ell(\hat{\bm{\mu}}(t)) / N_\ell(t).\]

Since $\frac{1 / \tilde{d}_c(\hat{\bm{\mu}}(t))}{\sum_{c' \in \mathcal{C}(s)} 1 / \tilde{d}_{c'}(\hat{\bm{\mu}}(t))} \neq 0$ and $\tilde{I}_c(t) = \argmax_{\ell \in \mathcal{D}(c)} \tilde{w}^c_\ell(\hat{\bm{\mu}}(t)) / N_\ell(t)$, we obtain $\tilde{w}^s_{\tilde{I}_c(t)}(\hat{\bm{\mu}}(t)) > 0$.
By the same argument as Proposition~\ref{prop:weight_val_form}, we derive
\[ \tilde{w}^s_{\tilde{I}_c(t)}(\hat{\bm{\mu}}(t)) = \frac{\tilde{d}_s(\hat{\bm{\mu}}(t))}{\tilde{d}_{\tilde{I}_c(t)}(\hat{\bm{\mu}}(t))}.\]
At this point, $\tilde{d}_s(\hat{\bm{\mu}}(t)) < 0$ since $a_s(\hat{\bm{\mu}}(t)) = \text{`lose'}$.

Therefore, we conclude that
\begin{align*}
    \tilde{I}_s(t) 
    &= \argmax_{\ell \in \{\tilde{I}_c(t)\}_{c \in \mathcal{C}(s)}} \frac{\tilde{d}_s(\hat{\bm{\mu}}(t))}{\tilde{d}_{\tilde{I}_c(t)}(\hat{\bm{\mu}}(t))} / N_\ell(t)\\
    &= \argmin_{\ell \in \{\tilde{I}_c(t)\}_{c \in \mathcal{C}(s)}} \frac{1}{N_\ell(t) \tilde{d}_{\tilde{I}_c(t)}(\hat{\bm{\mu}}(t))}.
\end{align*}

The same discussion also holds in the case with $L(s) = \text{`MIN'}$ and $a_{s_0}(\hat{\bm{\mu}}(t)) = \text{`lose'}$, except that the sign of $\tilde{d}_s(\hat{\bm{\mu}}(t))$ is reversed.

Finally, we can calculate $\tilde{I}_s(t)$ by 
\begin{align*}
    \tilde{I}_s(t) =
    \begin{cases}
        \tilde{I}_{c^+(s)}(t)
        & (L(s) = \text{`MAX'} \text{ and } a_s(\hat{\bm{\mu}}(t)) = \text{`win'})\\
        \argmin_{\ell \in \{I_c(t)\}_{c \in \mathcal{C}(s)} } \frac{1}{N_\ell(t) \tilde{d}(\hat{\bm{\mu}}_\ell(t))}
        & (L(s) = \text{`MAX'} \text{ and } a_s(\hat{\bm{\mu}}(t)) = \text{`lose'})\\
        \argmax_{\ell \in \{I_c(t)\}_{c \in \mathcal{C}(s)} } \frac{1}{N_\ell(t) \tilde{d}(\hat{\bm{\mu}}_\ell(t))}
        & (L(s) = \text{`MIN'} \text{ and } a_s(\hat{\bm{\mu}}(t)) = \text{`win'})\\
        \tilde{I}_{c^-(s)}(t)
        & (L(s) = \text{`MIN'} \text{ and } a_s(\hat{\bm{\mu}}(t)) = \text{`lose'}).
    \end{cases}
\end{align*}

To enable an efficient recomputation of $\tilde{I}_s(t)$, we introduce $\mathrm{RD}_s(t) \in \mathbb{R}^2$ in Recursive Formula~\ref{rf:RD}. 
The second element of $\mathrm{RD}_s(t)$ corresponds to $\tilde{I}_s(t)$ (we denote $\mathrm{RD}_s(t)[1]$).

\captionsetup[algorithm]{name=Recursive Formula}
\begin{algorithm}[tb]
  \caption{Auxiliary Variable for Efficient Computation of $\tilde{I}_s(t)$}\label{rf:RD}
\begin{align*}
    \mathrm{RD}_s(t) 
    =\begin{cases}
      \displaystyle\left(\frac{1}{\tilde{d}_s(\hat{\bm{\mu}}(t))N_s(t)},s\right) & (s\in \mathcal{L}(T),\tilde{d}_s(\hat{\bm{\mu}}(t))\neq 0)\\
          (+\infty,s) & (s\in \mathcal{L}(T),\tilde{d}_s(\hat{\bm{\mu}}(t))= 0)\\
          \mathrm{RD}_{c^+(s)}(t) &
          \left( \begin{aligned}
            &L(s)=\text{`MAX'},\\& \tilde{d}_c(\bm{\mu})\ge 0 \text{ for some } c\in \mathcal{C}(s)
        \end{aligned}\right)\\
        \displaystyle\argmin_{A\in\{\mathrm{RD}_{c}(t) \mid c\in \mathcal{C}(s)\}}A[0]  & \left( \begin{aligned}
            &L(s)=\text{`MAX'},\\& \tilde{d}_c(\bm{\mu})< 0 \text{ for all } c\in \mathcal{C}(s)
        \end{aligned}\right)\\
        \displaystyle\argmax_{A\in\{\mathrm{RD}_{c}(t) \mid c\in \mathcal{C}(s)\}}A[0]  & \left( \begin{aligned}
          &L(s)=\text{`MIN'},\\& \tilde{d}_c(\bm{\mu})> 0 \text{ for all } c\in \mathcal{C}(s)
        \end{aligned}\right)\\
\mathrm{RD}_{c^-(s)}(t) & \left( \begin{aligned}
            &L(s)=\text{`MIN'},\\& \tilde{d}_c(\bm{\mu})\le 0 \text{ for some } c\in \mathcal{C}(s)
        \end{aligned}\right)\\
    \end{cases}
\end{align*}
\end{algorithm}

\subsection{Implementation}\label{append:time_complexity:implement}
Let $D$ be the depth of the tree $\mathcal{T}$. Let $K = \max_{s \in S} |\mathcal{C}(s)|$ be the maximum number of branches.

After eliminating the dependence on $a_{s_0}(\hat{\bm{\mu}}(t))$, all values defined in Recursive Formula~\ref{rf:tilde_stat} and \ref{rf:RD} remain unchanged whenever $I(t) \notin \mathcal{D}(s)$.
Therefore, after each observation, only the values along the path from $I(t)$ to the root $s_0$ need to be updated.
Consequently, the update can be performed with a computational cost proportional to (tree depth) $\times$ (sibling comparisons) per time step.
By appropriately maintaining sibling values in a heap, operations that take an $\argmax$ or $\argmin$ over siblings (e.g., the third case of Eq.~(\ref{eq:tilde_d})) can be implemented in $O(\log |\mathcal{C}(s)|)$ time.
Moreover, operations that require summing sibling values (e.g., the fourth case of Eq.~(\ref{eq:tilde_d})) can be performed in $O(1)$ time by maintaining and updating the aggregate sums (in particular, $\mathrm{SZ}_s$ (``sum of $Z$'') and $\mathrm{RSR}_s$ (``reciprocal sum of the reciprocal'')).
Therefore, we can provide an implementation of Algorithm~\ref{alg:proposed} that runs in $O(D \log K)$ time per step, which we present as Algorithm~\ref{alg:proposed_efficient}.

\captionsetup[algorithm]{name=Algorithm}
\setcounter{algorithm}{1}
\begin{algorithm}[tb]
\caption{An efficient implementation for RD-Tracking-TMCTS} \label{alg:proposed_efficient}
\begin{algorithmic}[1]
  \Require $\mathcal{T}$: tree, $\theta$: threshold, $\delta$: confidence parameter
  \Ensure Return $a_{s_0}(\bm{\mu}) \in \{\text{`win'}, \text{`lose'}\}$ correctly with prob. at least $1 - \delta$
\State $t\gets 1$
  \For{each $\ell\in \mathcal{L}(\mathcal{T})$}
  \State Draw $\ell$ and observe reward $r(t)$.
  \State $N_{\ell}\gets 1$, $\hat{\mu}_{\ell}\gets r(t)$, $t\gets t+1$
  \EndFor
  \State $(\tilde{d}_{s_0},\tilde{Z}_{s_0},\mathrm{RD}_{s_0})\gets$ initialization($s_0$) $\quad$ (Algorithm~\ref{alg:initialization})
  \State $\mathrm{IHeap}\text{-}N.\mathrm{heapify}(\{(N_{\ell},\ell)\})$
\Loop
\If{$\mathrm{IHeap}\text{-}N.\mathrm{peek}[0]<\sqrt{t}-|\mathcal{L}(\mathcal{T})|/2$}
\State $I(t)\gets \mathrm{IHeap}\text{-}N.\mathrm{peek}[1]$
\Else
\State $I(t)\gets \mathrm{RD}_{s_0}[1]$
\EndIf
    \State Observe reward $r(t)$ from leaf $I(t)$
    \State update\_estimation($I(t), r(t)$) $\quad$ (Algorithm~\ref{alg:update_estimation})
    \If{$|\tilde{Z}_{s_0}(t)|\ge\beta(t,\delta)$} 
      \State \textbf{Return} $\begin{cases}\text{`win'} & (\tilde{Z}_{s_0}(t)\ge 0)\\ \text{`lose'} & (\tilde{Z}_{s_0}(t)<0)\end{cases}$
      \EndIf
    \State $t\gets t+1$  
  \EndLoop
\end{algorithmic}
\end{algorithm}

\begin{algorithm}[tb]
\caption{initialization} \label{alg:initialization}
\begin{algorithmic}[1]
  \Require $s$: node
  \If{$s\in \mathcal{L}(\mathcal{T})$}
  \State $\tilde{d}_s\gets \begin{cases} d(\hat{\mu}_s(t),\theta) & (\hat{\mu}_s(t)\ge \theta)\\ -d(\hat{\mu}_s(t),\theta) & (\hat{\mu}_s(t)<\theta)\end{cases}$
  \State $\tilde{Z}_s\gets N_s(t)\tilde{d}_s$
  \State $\mathrm{RD}_s\gets \left(1/\tilde{d}_sN_s,s\right)$
  \State \textbf{Return} $(\tilde{d}_s,\tilde{Z}_s,\mathrm{RD}_s)$
  \Else
  \State $i\gets 0$, $\mathrm{A}_{\tilde{d}},\mathrm{A}_{\tilde{Z}},\mathrm{A}_{\mathrm{RD}}\gets [\ ]$
  \For{$c\in \mathcal{C}(s)$}
  \State $(\tilde{d}_c,\tilde{Z}_c,\mathrm{RD}_c)\gets$ initialization($c$)
  \State $\mathrm{A}_{\tilde{d}}[i]\gets (\tilde{d}_c,c)$
  \State $\mathrm{A}_{\tilde{Z}}[i]\gets (\tilde{Z}_c,c)$
  \State $\mathrm{A}_{\mathrm{RD}}[i]\gets \mathrm{RD}_c$
  \State $i\gets i+1$
  \EndFor
  \If{$L(s)=\text{`MAX'}$}
  \State $\mathrm{RSR}_s\gets 1/\sum_{c\in \mathcal{C}(s),\tilde{d}_c<0}(1/\tilde{d}_c)$
  \State $\mathrm{SZ}_s\gets \sum_{c\in \mathcal{C}(s),\tilde{Z}_c\ge 0}\tilde{Z}_c$
  \State $\mathrm{DHeap}\text{-}\tilde{d}_s.\mathrm{heapify}(\mathrm{A}_{\tilde{d}})$
  \State $\mathrm{DHeap}\text{-}\tilde{Z}_s.\mathrm{heapify}(\mathrm{A}_{\tilde{Z}})$
  \State $\mathrm{IHeap}\text{-}\mathrm{RD}_s.\mathrm{heapify}(\mathrm{A}_{\mathrm{RD}})$
  \If{$\mathrm{DHeap}\text{-}\tilde{d}_s.\mathrm{peek}[0]\ge 0$}
  \State $\tilde{d}_s\gets \mathrm{DHeap}\text{-}\tilde{d}_s.\mathrm{peek}[0]$
  \State $\tilde{Z}_s\gets \mathrm{SZ}_s$
    \State $\mathrm{RD}_s\gets \mathrm{RD}_{\mathrm{DHeap}\text{-}\tilde{d}_s.\mathrm{peek}[1]}$
  \Else
  \State $\tilde{d}_s\gets \mathrm{RSR}_s$
  \State $\tilde{Z}_s\gets \mathrm{DHeap}\text{-}\tilde{Z}_s.\mathrm{peek}[0]$
  \State $\mathrm{RD}_s\gets \mathrm{IHeap}\text{-}\mathrm{RD}_s.\mathrm{peek}$
  \EndIf
  \Else
  \State $\mathrm{RSR}_s\gets 1/\sum_{c\in \mathcal{C}(s),\tilde{d}_c>0}(1/\tilde{d}_c)$
  \State $\mathrm{SZ}_s\gets \sum_{c\in \mathcal{C}(s),\tilde{Z}_c< 0}\tilde{Z}_c$
  \State $\mathrm{IHeap}\text{-}\tilde{d}_s.\mathrm{heapify}(\mathrm{A}_{\tilde{d}})$
  \State $\mathrm{IHeap}\text{-}\tilde{Z}_s.\mathrm{heapify}(\mathrm{A}_{\tilde{Z}})$
  \State $\mathrm{DHeap}\text{-}\mathrm{RD}_s.\mathrm{heapify}(\mathrm{A}_{\mathrm{RD}})$
  \If{$\mathrm{IHeap}\text{-}\tilde{d}_s.\mathrm{peek}[0]< 0$}
  \State $\tilde{d}_s\gets \mathrm{IHeap}\text{-}\tilde{d}_s.\mathrm{peek}[0]$
  \State $\tilde{Z}_s\gets \mathrm{SZ}_s$
    \State $\mathrm{RD}_s\gets \mathrm{RD}_{\mathrm{IHeap}\text{-}\tilde{d}_s.\mathrm{peek}[1]}$
  \Else
  \State $\tilde{d}_s\gets \mathrm{RSR}_s$
  \State $\tilde{Z}_s\gets \mathrm{IHeap}\text{-}\tilde{Z}_s.\mathrm{peek}[0]$
  \State $\mathrm{RD}_s\gets \mathrm{DHeap}\text{-}\mathrm{RD}_s.\mathrm{peek}$
  \EndIf
  \EndIf
  \EndIf
    \State \textbf{Return} $(\tilde{d}_s,\tilde{Z}_s,\mathrm{RD}_s)$
\end{algorithmic}
\end{algorithm}

\begin{algorithm}[tb]
\caption{update\_estimation} \label{alg:update_estimation}
\begin{minipage}[t]{0.5\textwidth}
\begin{algorithmic}[1]
  \Require $\ell$: leaf, $r$: reward
  \Ensure update $\hat{\mu}_{\ell}$, $N_{\ell}$, heaps correctly
  \State $\mathrm{prev}\text{-}\tilde{d}_c\gets \tilde{d}_{\ell}$, $\mathrm{prev}\text{-}\tilde{Z}_c\gets \tilde{Z}_{\ell}$
  \State $\displaystyle\hat{\mu}_{\ell}\gets \frac{\hat{\mu}_{\ell}N_{\ell}+r}{N_{\ell}+1}$, $N_{\ell}\gets N_{\ell}+1$
  \State $\tilde{d}_{\ell}\gets \begin{cases}d(\hat{\mu}_{\ell},\theta) & (\hat{\mu}_{\ell}\ge \theta)\\-d(\hat{\mu}_{\ell},\theta) &
    (\hat{\mu}_{\ell}< \theta) \end{cases}$
  \State $\tilde{Z}_{\ell} \gets N_{\ell}\tilde{d}_{\ell}$, $\mathrm{RD}_{\ell}\gets \left(1/\tilde{Z}_{\ell},\ell\right)$
  \State $c\gets \ell$, $s\gets$ parent node of $c$
\end{algorithmic}
\end{minipage}
\begin{minipage}[t]{0.5\textwidth}
  \begin{algorithmic}[1]
\setcounter{ALG@line}{5}
\Repeat
  \If{$L(s)=\text{`MAX'}$}
  \If{$\mathrm{prev}\text{-}\tilde{d}_c<0$}
  \State $\mathrm{RSR}_s\gets 1/(1/\mathrm{RSR}_s-1/\mathrm{prev}\text{-}\tilde{d}_c)$
  \Else
\State  $\mathrm{SZ}_s\gets \mathrm{SZ}_s-\mathrm{prev}\text{-}\tilde{Z}_c$
  \EndIf
  \If{$\tilde{d}_c<0$}
  \State $\mathrm{RSR}_s\gets 1/(1/\mathrm{RSR}_s+1/\tilde{d}_c)$
  \Else
  \State $\mathrm{SZ}_s\gets \mathrm{SZ}_s+\tilde{Z}_c$
  \EndIf
  \State $\mathrm{prev}\text{-}\tilde{d}_c\gets \tilde{d}_s$, $\mathrm{prev}\text{-}\tilde{Z}_c\gets \tilde{Z}_s$
  \State $\mathrm{DHeap}\text{-}\tilde{d}_s.\mathrm{rewrite}(\tilde{d}_c,c)$
  \State $\mathrm{DHeap}\text{-}\tilde{Z}_s.\mathrm{rewrite}(\tilde{Z}_c,c)$
  \State $\mathrm{IHeap}\text{-}\mathrm{RD}_s.\mathrm{rewrite}(\mathrm{RD}_c,c)$
  \If{$\mathrm{DHeap}\text{-}\tilde{d}_s.\mathrm{peek}[0]\ge 0$}
  \State $\tilde{d}_s\gets \mathrm{DHeap}\text{-}\tilde{d}_s.\mathrm{peek}[0]$
  \State $\tilde{Z}_s\gets \mathrm{SZ}_s$
  \State $\mathrm{RD}_s\gets \mathrm{RD}_{\mathrm{DHeap}\text{-}\tilde{d}_s.\mathrm{peek}[1]}$
  \Else
  \State $\tilde{d}_s\gets \mathrm{RSR}_s$
  \State $\tilde{Z}_s\gets \mathrm{DHeap}\text{-}\tilde{Z}_s.\mathrm{peek}[0]$
  \State $\mathrm{RD}_s\gets \mathrm{IHeap}\text{-}\mathrm{RD}_s.\mathrm{peek}$
  \EndIf
  \Else
  \If{$\mathrm{prev}\text{-}\tilde{d}_c>0$}
  \State $\mathrm{RSR}_s\gets 1/(1/\mathrm{RSR}_s-1/\mathrm{prev}\text{-}\tilde{d}_c)$
  \Else
  \State $\mathrm{SZ}_s\gets \mathrm{SZ}_s-\mathrm{prev}\text{-}\tilde{Z}_c$
  \EndIf
  \If{$\tilde{d}_c>0$}
  \State $\mathrm{RSR}_s\gets 1/(1/\mathrm{RSR}_s+1/\tilde{d}_c)$
  \Else
\State  $\mathrm{SZ}_s\gets \mathrm{SZ}_s+\tilde{Z}_s$
  \EndIf
  \State $\mathrm{prev}\text{-}\tilde{d}_c\gets \tilde{d}_s$, $\mathrm{prev}\text{-}\tilde{Z}_c\gets \tilde{Z}_s$
  \State $\mathrm{IHeap}\text{-}\tilde{d}_s.\mathrm{rewrite}(\tilde{d}_c,c)$
  \State $\mathrm{IHeap}\text{-}\tilde{Z}_s.\mathrm{rewrite}(\tilde{Z}_c,c)$
    \State $\mathrm{DHeap}\text{-}\mathrm{RD}_s.\mathrm{rewrite}(\mathrm{RD}_c,c)$
  \If{$\mathrm{IHeap}\text{-}\tilde{d}_s.\mathrm{peek}[0]< 0$}
  \State $\tilde{d}_s\gets \mathrm{IHeap}\text{-}\tilde{d}_s.\mathrm{peek}[0]$
  \State $\tilde{Z}_s\gets \mathrm{SZ}_s$
  \State $\mathrm{RD}_s\gets \mathrm{RD}_{\mathrm{IHeap}\text{-}\tilde{d}_s.\mathrm{peek}[1]}$  
  \Else
  \State $\tilde{d}_s\gets \mathrm{RSR}_s$
  \State $\tilde{Z}_s\gets \mathrm{IHeap}\text{-}\tilde{Z}_s.\mathrm{peek}[0]$
    \State $\mathrm{RD}_s\gets \mathrm{DHeap}\text{-}\mathrm{RD}_s.\mathrm{peek}$
  \EndIf
  \EndIf
\State $c\gets s$, $s\gets$ parent node of $c$
\Until{$s=$NULL}
\end{algorithmic}
\end{minipage}
\end{algorithm}

\section{Proofs for Good Action Identification Problem}\label{append:GAI}

\subsection{$\delta$-correctness}

Here, we state the $\delta$-correctness of the algorithm for the good action identification problem.
\begin{theorem}
    The algorithm obtained by replacing the stopping rule of Algorithm~\ref{alg:proposed} (line~\ref{line:stopping_condition}) with $Z^\mathrm{GAI}_{s_0}(t) \geq \beta(t,\delta)$ satisfies 
    \[ \mathbb{P}(\tau_\delta < \infty, \hat{a}^\mathrm{GAI}(\tau_\delta) \text{ is not a correct answer}) \leq \delta\]
    with the sequence of thresholds $(\beta(t,\delta))_{t \in \mathbb{N}}$ defined by Eq.~(\ref{eq:betadef}).
\end{theorem}
\begin{proof}
    We define the set of mean vectors of reward distributions for which the answer $\hat{a}^\mathrm{GAI}(t)$ is incorrect as
    \begin{align*}
        \mathrm{Alt}^\mathrm{GAI}(t) = 
        \begin{cases}
            \{\bm{\lambda}=(\lambda_{\ell})_{\ell\in \mathcal{L}(\mathcal{T})} \in X^{|\mathcal{L}(\mathcal{T})|} \mid V_{\hat{a}^\mathrm{GAI}(t)}(\bm{\lambda}) < \theta \}
            & (\hat{a}^\mathrm{GAI}(t) \neq \text{``no good action"}),\\
            \{\bm{\lambda}=(\lambda_{\ell})_{\ell\in \mathcal{L}(\mathcal{T})} \in X^{|\mathcal{L}(\mathcal{T})|} \mid V_{s_0}(\bm{\lambda}) \geq \theta \}
            & (\hat{a}^\mathrm{GAI}(t) = \text{``no good action"}).
        \end{cases}
    \end{align*}
    Here, the GLR statistics become
    \begin{align*}  
      Z^\mathrm{GAI}_{s_0}(t)
      =&
      \log \frac{\sup_{\bm{\lambda} \in X^{|\mathcal{L}(\mathcal{T})|}} \Lambda(\bm{\lambda} \mid R_{s_0}(t))}{\sup_{\bm{\lambda} \in \mathrm{Alt}^\mathrm{GAI}(t)} \Lambda(\bm{\lambda} \mid R_{s_0}(t))}.
    \end{align*}
    As in Corollary~\ref{cor:kaufmann_prop15}, by following the argument of Proposition~15 in \citet{kaufmann2021}, we can verify the $\delta$-correctness of the stopping rule using $Z^\mathrm{GAI}_{s_0}(t)$ of this form.

    We next show that $Z^\mathrm{GAI}_{s_0}(t)$ can be expressed as Eq.~(\ref{eq:GLR_GAI}).
    First, we consider the case with $a_{s_0}(\hat{\bm{\mu}}(t)) = \text{`lose'}$.
    This implies $\hat{a}^\mathrm{GAI}(t) = \text{``no good action"}$, and we obtain $\mathrm{Alt}(\hat{\bm{\mu}}(t)) = \mathrm{Alt}^\mathrm{GAI}(t)$ and 
    \begin{align*}  
        Z^\mathrm{GAI}_{s_0}(t)
        &=
        \log \frac{\sup_{\bm{\lambda} \in X^{|\mathcal{L}(\mathcal{T})|}} \Lambda(\bm{\lambda} \mid R_{s_0}(t))}{\sup_{\bm{\lambda} \in \mathrm{Alt}^\mathrm{GAI}(t)} \Lambda(\bm{\lambda} \mid R_{s_0}(t))}\\
        &=
        \log \frac{\sup_{\bm{\lambda} \in X^{|\mathcal{L}(\mathcal{T})|}} \Lambda(\bm{\lambda} \mid R_{s_0}(t))}{\sup_{\bm{\lambda} \in \mathrm{Alt}(\hat{\bm{\mu}}(t))} \Lambda(\bm{\lambda} \mid R_{s_0}(t))}
        = Z_{s_0}(t).
    \end{align*}
    
    Next, we consider the case with $a_{s_0}(\hat{\bm{\mu}}(t)) = \text{`win'}$.
    This implies $\hat{a}^\mathrm{GAI}(t) \neq \text{``no good action"}$ and 
    \begin{align*}
        \mathrm{Alt}^\mathrm{GAI}(t)
        &= 
        \{\bm{\lambda}=(\lambda_{\ell})_{\ell\in \mathcal{L}(\mathcal{T})} \in X^{|\mathcal{L}(\mathcal{T})|} \mid V_{a^\mathrm{GAI}(t)}(\bm{\lambda}) < \theta \}\\
        &=
        \{\bm{\lambda}=(\lambda_{\ell})_{\ell\in \mathcal{L}(\mathcal{T})} \in X^{|\mathcal{L}(\mathcal{T})|} \mid (\lambda_{\ell'})_{\ell' \in \mathcal{D}(a^\mathrm{GAI}(t))} \in \mathrm{Alt}_{a^\mathrm{GAI}(t)}(\hat{\bm{\mu}}(t)) \}.
    \end{align*}
    From the definition of the likelihood, $\Lambda(\bm{\lambda} \mid R_{s_0}(t)) = \prod_{c \in \mathcal{C}(s_0)} \Lambda((\lambda_\ell)_{\ell \in \mathcal{C}(c)} \mid R_c(t))$ holds.
    Hence, we obtain
    \begin{align*}  
        Z^\mathrm{GAI}_{s_0}(t)
        &=
        \log \frac{\sup_{\bm{\lambda} \in X^{|\mathcal{L}(\mathcal{T})|}} \Lambda(\bm{\lambda} \mid R_{s_0}(t))}{\sup_{\bm{\lambda} \in \mathrm{Alt}^\mathrm{GAI}(t)} \Lambda(\bm{\lambda} \mid R_{s_0}(t))}\\
        &=
        \log \frac{\sup_{\bm{\lambda} \in X^{|\mathcal{L}(\mathcal{T})|}} \prod_{c \in \mathcal{C}(s_0)} \Lambda((\lambda_\ell)_{\ell \in \mathcal{C}(c)} \mid R_c(t))}{\sup_{\bm{\lambda} \in \mathrm{Alt}^\mathrm{GAI}(t)} \prod_{c \in \mathcal{C}(s_0)} \Lambda((\lambda_\ell)_{\ell \in \mathcal{C}(c)} \mid R_c(t))}\\
        &=
        \log \left(
        \begin{multlined}
            \frac{\prod_{c \in \mathcal{C}(s_0)} \sup_{(\lambda_\ell)_{\ell \in \mathcal{D}(c)} \in X^{|\mathcal{D}(c)|}} \Lambda((\lambda_\ell)_{\ell \in \mathcal{C}(c)} \mid R_c(t))}{\prod_{c \in \mathcal{C}(s_0) \setminus \{ \hat{a}^\mathrm{GAI}(t) \}} \sup_{(\lambda_\ell)_{\ell \in \mathcal{D}(c)} \in X^{|\mathcal{D}(c)|}} \Lambda((\lambda_\ell)_{\ell \in \mathcal{C}(c)} \mid R_c(t))}\\
            \times \frac{1}{\sup_{(\lambda_\ell)_{\ell \in \mathcal{D}(\hat{a}^\mathrm{GAI}(t))} \in \mathrm{Alt}_{\hat{a}^\mathrm{GAI}(t)}(\hat{\bm{\mu}}(t))} \Lambda((\lambda_\ell)_{\ell \in \mathcal{C}(\hat{a}^\mathrm{GAI}(t))} \mid R_{\hat{a}^\mathrm{GAI}(t)}(t))}
        \end{multlined}
        \right)\\
        &=
        \log 
        \frac{\sup_{(\lambda_\ell)_{\ell \in \mathcal{D}(\hat{a}^\mathrm{GAI}(t))} \in X^{|\mathcal{D}(\hat{a}^\mathrm{GAI}(t))|}} \Lambda((\lambda_\ell)_{\ell \in \mathcal{C}(\hat{a}^\mathrm{GAI}(t))} \mid R_{\hat{a}^\mathrm{GAI}(t)}(t))}{\sup_{(\lambda_\ell)_{\ell \in \mathcal{D}(\hat{a}^\mathrm{GAI}(t))} \in \mathrm{Alt}_{\hat{a}^\mathrm{GAI}(t)}(\hat{\bm{\mu}}(t))} \Lambda((\lambda_\ell)_{\ell \in \mathcal{C}(\hat{a}^\mathrm{GAI}(t))} \mid R_{\hat{a}^\mathrm{GAI}(t)}(t))}\\
        &= 
        Z_{\hat{a}^\mathrm{GAI}(t)}(t)\\
        &=
        \max_{c \in \mathcal{C}(s_0)} Z_c(t). \qedhere
    \end{align*}
\end{proof}

\subsection{Proof of Theorem~\ref{thm:sample_complexity_GAI}}

\begin{proof}
    First, we consider the case with $a_{s_0}(\hat{\bm{\mu}}(t)) = \text{`lose'}$.
    From Eq.~(\ref{eq:GLR_GAI}), we have $Z^\mathrm{GAI}_{s_0}(t) = Z_{s_0}(t)$.
    Therefore, by following the same steps as in Theorem~\ref{thm:sample_complexity}, we can verify that
    \[Z^\mathrm{GAI}_{s_0}(t) \geq t d_{s_0}^\epsilon(\bm{\mu}) \]
    and 
    \[\limsup_{\delta \to +0} \frac{ \mathbb{E}[\tau_\delta] }{\log(1 / \delta)} \leq
    \frac{1}{d_{s_0}(\bm{\mu})}\]
    are guaranteed to hold.

    Next, we consider the case with $a_{s_0}(\hat{\bm{\mu}}(t)) = \text{`win'}$.
    From Eq.~(\ref{eq:GLR_GAI}), we have $Z^\mathrm{GAI}_{s_0}(t) = \max_{c \in \mathcal{C}(s)} Z_c(t)$.
    This implies that
    \begin{align*}
        Z^\mathrm{GAI}_{s_0} 
        &= \max_{c \in \mathcal{C}(s)} Z_c(t)\\
        &\geq \max_{c \in \mathcal{C}(s)} t d^\epsilon_{c}(\bm{\mu}),
    \end{align*}
    where
    \begin{align*}
        d^\epsilon_{c}(\bm{\mu}) := \inf_{\substack{\bm{\mu}'\in \mathcal{I}_{\epsilon}\\ \bm{w}': || \bm{w}' - \bm{w}^{s_0}(\bm{\mu})|| < 2(|\mathcal{L}(T)| - 1)\epsilon}} \inf_{\bm{\lambda} \in \mathrm{Alt}(\hat{\bm{\mu}}(t))} \sum_{\ell\in \mathcal{D}(c)}\frac{N_{\ell}(t)}{t}d(\hat{\mu}_{\ell}(t),\lambda_{\ell}).
    \end{align*}

    By following the same steps as in Theorem~\ref{thm:sample_complexity}, we obtain
    \[ \limsup_{\delta \to +0} \frac{\mathbb{E}_{\bm{\mu}}[\tau_\delta]}{\log(1 / \delta)} \leq \frac{1+\eta}{ \max_{c \in \mathcal{C}(s)} d^\epsilon_{c}(\bm{\mu})} \]
    for every $\eta > 0$ and $\epsilon > 0$.

    Since $\lim_{\epsilon \to +0} d^\epsilon_{c}(\bm{\mu}) = d_c(\bm{\mu})$ and $d_{s_0}(\bm{\mu}) = \max_{c \in \mathcal{C}(s)} d_c(\bm{\mu})$, we conclude
    \[\limsup_{\delta \to +0} \frac{ \mathbb{E}[\tau_\delta] }{\log(1 / \delta)} \leq
    \frac{1}{d_{s_0}(\bm{\mu})}. \qedhere\]
\end{proof}

\end{document}